\documentclass[twoside,11pt,preprint]{article}

\usepackage{blindtext}

\usepackage{jmlr2e}

\usepackage{amsmath}
\usepackage{multirow}
\usepackage{float}
\usepackage{threeparttable}
\usepackage{booktabs}
\usepackage{xcolor}
\usepackage{dsfont}

\usepackage{subcaption}

\hypersetup{
  pdftitle={Underdamped Langevin MCMC with third order convergence},
  pdfauthor={Maximilian Scott, Dáire O'Kane, Andraž Jelinčič and James Foster}
}

\allowdisplaybreaks

\newcommand{\m}{\hspace{0.25mm}}
\newcommand{\mm}{\hspace{5mm}}
\newcommand{\mmm}{\hspace{10mm}}

\newcommand{\mmmmm}{\hspace{20mm}}
\newcommand{\mfill}{\hspace{50mm}}

\newcommand{\E}{\mathbb{E}}
\renewcommand{\L}{\mathbb{L}}

\usepackage{lastpage}
\jmlrheading{23}{2025}{1-\pageref{LastPage}}{1/21; Revised 5/22}{9/22}{21-0000}{Maximilian Scott,  D\'{a}ire O'Kane, Andra\v{z} Jelin\v{c}i\v{c} and James Foster}

\ShortHeadings{Underdamped Langevin MCMC with third order convergence}{Scott, O'Kane, Jelin\v{c}i\v{c} and Foster}
\firstpageno{1}

\begin{document}


\title{Underdamped Langevin MCMC \\ with third order convergence}

\author{\name Maximilian Scott \email ms4314@bath.ac.uk 
	\AND
	\name D\'{a}ire O'Kane \email dok29@bath.ac.uk 
	\AND
	\name Andra\v{z} Jelin\v{c}i\v{c} \email aj2382@bath.ac.uk 
       \AND
       \name James Foster \email jmf68@bath.ac.uk \m \\
       \addr Department of Mathematical Sciences\\
       University of Bath, UK}

\editor{My editor}

\maketitle

\begin{abstract}
In this paper, we propose a new numerical method for the underdamped Langevin diffusion (ULD) and present a non-asymptotic analysis of its sampling error in the 2-Wasserstein distance when the $d$-dimensional target distribution $p(x)\propto e^{-f(x)}$ is strongly log-concave and has varying degrees of smoothness. Precisely, under the assumptions that the gradient and Hessian of $f$ are Lipschitz continuous, our algorithm achieves a 2-Wasserstein error of $\varepsilon$ in $\mathcal{O}\big(\sqrt{d\m}/\varepsilon\m\big)$ and $\mathcal{O}\big(\sqrt{d\m}/\sqrt{\varepsilon}\m\big)$ steps respectively. Therefore, our algorithm has a similar complexity as other popular Langevin MCMC algorithms under matching assumptions. However, if we additionally assume that the third derivative of $f$ is Lipschitz continuous, then our algorithm achieves a 2-Wasserstein error of $\varepsilon$ in $\mathcal{O}\big(\sqrt{d\m}/\varepsilon^{\frac{1}{3}}\m\big)$ steps. To the best of our knowledge, this is the first gradient-only method for ULD with third order convergence. To support our theory, we perform Bayesian logistic regression across a range of real-world datasets, where our algorithm achieves competitive performance compared to an existing underdamped Langevin MCMC algorithm and the popular No U-Turn Sampler (NUTS). 
\end{abstract}\medbreak


\section{Introduction}
\label{sec:introduction}

In a variety of applications ranging from statistical inference \citep{ma2015complete, mandt2017stochastic, durmus2019high} to molecular dynamics \citep{pastor1988analysis, pastor1994techniques, allen2017computer}, we are often interested in sampling from a target distribution 
\begin{equation}
\label{eq:target_distribution}
	p(x) \propto e^{-f(x)},
\end{equation}
where $x \in \mathbb{R}^d$ and $f : \mathbb{R}^d \to \mathbb{R}$ denotes a scalar potential. In practice, we typically only know this target distribution up to a normalisation constant -- which becomes challenging to compute when the dimension $d$ is large. So instead, we can generate samples by constructing and simulating a Markov chain whose marginal distribution becomes close to the target \eqref{eq:target_distribution}. This methodology for sampling from $p$ is known as Markov Chain Monte Carlo (or MCMC). \medbreak\noindent
In this paper, we consider MCMC algorithms obtained by simulating the underdamped Langevin diffusion (ULD), given by the following stochastic differential equation (SDE):
\begin{align}
\label{eq:ULD}
	\begin{split}
	dx_t &= v_t \, dt, \\[2pt]
	dv_t &= -\gamma v_t \, dt - u \nabla f(x_t) \, dt + \sqrt{2\gamma u}\, dW_t\m,
	\end{split}
\end{align}
where $x,v \in \mathbb{R}^d$ represent the position and momentum of a particle, $\gamma > 0$ denotes a friction coefficient, $u > 0$ and $W$ is a $d$-dimensional Brownian motion. Under mild assumptions on the scalar potential $f$, it is known that the SDE \eqref{eq:ULD} admits a unique strong solution which is ergodic with respect to the Gibbs measure with Hamiltonian $H(x,v) =  f(x) + \frac{1}{2u} \|v\|^2$ \citep[Proposition 6.1]{pavliotis2014stochastic}. In other words, ULD has a stationary measure $\pi$ on $\mathbb{R}^{2d}$ with density proportional to $\exp(-H(x,v))$. Therefore, the distribution of $x_t$ will converge to the target distribution \eqref{eq:target_distribution} and thus simulating ULD will allows us to generate samples. In general, the SDE \eqref{eq:ULD} which governs ULD does not admit a closed-form solution, and so it is necessary to develop appropriate numerical methods for approximating these samples.\medbreak\noindent
A variety of approximations have previously been proposed for ULD. Some noteworthy examples include the unadjusted and Metropolis-adjusted OBABO schemes \citep{bussi2007accurate, monmarche2021high, song2021hamiltonian, rioudurand2022malt}, the UBU method \citep{alamo2019word, sanz2021wasserstein}, the ALUM methods \citep{hu2021optimal}, the randomised midpoint method \citep{shen2019randomized}, the BAOAB method \citep{leimkuhler2016computation} and a Strang splitting \citep{buckwar2019spectral}. In addition, there exist a variety of methods for improving the accuracy and robustness of underdamped Langevin MCMC algorithms. These include introducing a preconditioning matrix \citep{goodman2010ensemble, leimkuhler2018ensemble, alrachid2018remarks}, using adaptive step sizes \citep{foster2023convergence, leroy2024adaptive, leimkuhler2025adam} or using unbiased estimators based on Multilevel Monte Carlo techniques \citep{chada2024unbiased}.\medbreak\noindent
In this paper we introduce a new practical method for approximating underdamped Langevin dynamics called QUICSORT\footnote{\textbf{QU}adrature \textbf{I}nspired and \textbf{C}ontractive \textbf{S}hifted \textbf{O}DE with \textbf{R}unge-Kutta \textbf{T}hree}, which is constructed using a piecewise linear discretisation of the underlying Brownian motion. This approach builds upon the Shifted ODE introduced by \citet{foster2021shifted} with the behaviour that, given certain convexity and smoothness assumptions on $f$, it yields third order convergence guarantees that hold for all $n\geq 0$.
\begin{definition}[QUICSORT]
\label{def:QUICSORT}
Let $\{ t_n \}_{n \geq 0}$ be a sequence of times with $t_0 = 0$, $t_{n+1} > t_n$ and $h_n = t_{n+1} - t_n$. We construct a numerical solution $\{ (X_n, V_n) \}_{n \geq 0}$ for the SDE \eqref{eq:ULD} by choosing some initial $(X_0, V_0) \sim \pi_0$ and, for each $n \geq 0$, defining $(X_{n+1}, V_{n+1})$ as follows:
\begin{align}\label{eq:QUICSORT}
	V^{(1)}_n &:= V_n + \sigma( H_n + 6K_n),\nonumber \\[2pt]
	X^{(1)}_n &:= X_n + \frac{1 - e^{-\lambda_- \gamma h_n}}{\gamma} V^{(1)}_n + \frac{e^{-\lambda_- \gamma h_n} + \lambda_- \gamma h_n - 1}{\gamma^2 h_n}\m C_n\m,\nonumber \\
	X^{(2)}_n &:= X_n + \frac{1 - e^{-\lambda_+ \gamma h_n}}{\gamma} V^{(1)}_n - \frac{1 - e^{-\frac{1}{3}\gamma h_n}}{\gamma} u \nabla f(X^{(1)}_n) h_n + \frac{e^{-\lambda_+ \gamma h_n} + \lambda_+ \gamma h_n - 1}{\gamma^2 h_n}\m C_n\m,\nonumber \\[1pt]
	V^{(2)}_n &:= e^{-\gamma h_n} V^{(1)}_n - \frac{1}{2} e^{-\lambda_+ \gamma h_n} u \nabla f(X^{(1)}_n) h_n - \frac{1}{2} e^{-\lambda_- \gamma h_n} u \nabla f(X^{(2)}_n) h_n + \frac{1 - e^{-\gamma h_n}}{\gamma h_n}\m C_n\m,\nonumber\\[2pt]
	X_{n+1} &:= X_n - \frac{1 - e^{-\lambda_+ \gamma h_n}}{2\gamma} u \nabla f(X^{(1)}_n) h_n - \frac{1 - e^{-\lambda_- \gamma h_n}}{2\gamma} u \nabla f(X^{(2)}_n) h_n \\
					&\hspace{11.5mm} + \frac{1 - e^{-\gamma h_n}}{\gamma} V^{(1)}_n  + \frac{e^{- \gamma h_n} + \gamma h_n - 1}{\gamma^2 h_n}\m C_n\m,\nonumber\\[1pt]
	V_{n+1} &:= V^{(2)}_n - \sigma ( H_n - 6K_n)\nonumber.
	\end{align}
where $\sigma = \sqrt{2\gamma u}$, $\lambda_+ = \frac{3 + \sqrt{3}}{6}$, $\lambda_- = \frac{3 - \sqrt{3}}{6}$ and $C_n = \sigma(W_n - 12K_n)$. The vectors $\{W_n\}_{n\geq0}$, $\{H_n\}_{n\geq0}$ and $\{K_n\}_{n\geq0}$ are independent and Gaussian distributed with $W_n \sim \mathcal{N} \left( 0, h_n I_d \right)$, $H_n \sim \mathcal{N} \left( 0, \frac{1}{12} h_n I_d \right)$ and $K_n \sim \mathcal{N} \left( 0, \frac{1}{720} h_n I_d \right)$ for all $n\geq0$. These random vectors can be viewed as the first three coefficients in a polynomial expansion of the Brownian motion $W$ (see \citet{foster2020optimal}).
\end{definition}

\noindent
In the recent literature, it has become standard practice to understand a method's efficacy through a non-asymptotic error analysis. More precisely, by finding an upper bound for the 2-Wasserstein distance between the target distribution \eqref{eq:target_distribution} and the distribution of $X_n$, the position variable at the $n$-th iteration, we can establish a mixing rate for the different numerical methods. Quantitatively, we compare the number of steps $n$ required to reach an error of $W_2 (X_n\m, e^{-f}) \leq \varepsilon$. The main theoretical contribution of this paper is that we show that the QUICSORT method achieves a 2-Wasserstein error of $\varepsilon$ in $\mathcal{O}(\sqrt{d} \epsilon^{-\frac{1}{3}})$ steps under certain smoothness and strong convexity assumption on $f$. Our central result is as follows:
\begin{theorem}[QUICSORT non-asymptotic convergence]
\label{thm:W2_convergence}
Let $\{ (X_n, V_n) \}_{n \geq 0}$ denote the QUICSORT approximation \eqref{eq:QUICSORT} where $0 < h \leq 1$ is fixed. Suppose that the potential $f$ is three time differentiable, strongly convex and its gradient, Hessian and third derivative are all Lipschitz continuous. Then there exists constants $\beta, C > 0$ (not depending on $d$ and $h$) such that for $n \geq 0$, the 2-Wasserstein distance between the law of $X_n$ and $p(x)\propto e^{-f(x)}$ is
\begin{align}
	W_2 \big(\m X_n\m , e^{-f}\m\big) \leq c \m e^{-n \beta h} \m W_2 \big(\m X_0\m , e^{-f}\m\big) + C \m d^{1.5} \m h^3.
\end{align}
\end{theorem}
\vspace{3mm}
\noindent
In Table \ref{table:method_convergence}, we compare the mixing rate of QUICSORT with other approximations of ULD. For proofs of their convergence rates see \citet{monmarche2021high}, \citet[Sections 6.1 and 6.2]{sanz2021wasserstein}, \citet{paulin2024error}, \citet{dalayan2020sampling} and \citet{hu2021optimal} respectively.\bigbreak

\begin{threeparttable}
  \centering
  \begin{tabular}{llc}
    \toprule
    Numerical method \quad \quad \quad & Smoothness assumptions & Number of steps $n$ to achieve\\[3pt]
    & on the strongly convex $f$ & \hspace*{-2mm}an error of $W_2\big(\m X_n\m , e^{-f}\m\big) \leq \varepsilon$\\
    \midrule
    QUICSORT (ours) & Lipschitz Gradient  & $\mathcal{O}\big(\sqrt{d\m}/\varepsilon\m\big)$ \\
      & \, + Lipschitz Hessian & $\mathcal{O}\big(\sqrt{d\m}/\sqrt{\varepsilon}\m\big)$\\
     & \, + Lipschitz third derivative & $\mathcal{O}\big(\sqrt{d\m}/\varepsilon^{\frac{1}{3}}\m\big)$\\
     \midrule
     OBABO splitting  & Lipschitz Gradient  & $\mathcal{O}\big(\sqrt{d\m}/\varepsilon\m\big)$ \\
           & \, + Lipschitz Hessian & $\mathcal{O}\big(\sqrt{d\m}/\hspace{-0.25mm}\sqrt{\varepsilon}\m\big)$\\
     \midrule
     Exponential Euler & Lipschitz Gradient  & $\mathcal{O}\big(\m\sqrt{d}/\varepsilon \m\big)$ \\
     \midrule
     UBU splitting & Lipschitz Gradient  & $\mathcal{O}\big(\sqrt{d\m}/\varepsilon\m\big)$ \\
   	 & \, + Lipschitz Hessian &  $\mathcal{O}\big(\sqrt{d\m}/\hspace{-0.25mm}\sqrt{\varepsilon}\m\big)$\\
   	 & \, + Strongly Lipschitz Hessian & $\mathcal{O}\big(\sqrt[4]{d\m}/\sqrt{\varepsilon}\m\big)$\\
     \midrule
     ALUM & Lipschitz Gradient  & $\mathcal{O}\big(\sqrt[3]{d\m}/\varepsilon^{\frac{2}{3}}\m\big)$ \\
    \bottomrule
  \end{tabular}
  \caption{Summary of complexities for ULD methods (with respect to $2$\hspace{0.125mm}-Wasserstein error).}\label{table:method_convergence}
  \end{threeparttable}
  
\noindent
We can numerically verify the mixing rates for most of these methods by estimating their strong\footnote{Throughout the paper, the terminology ``strong'' error refers to the $\L_2$ (or root mean square) error.} error at a large time $T$. This involves comparing the numerical solution obtained by the method using both large and small step sizes. To ensure that these two numerical solutions are close, we will apply methods using the same underlying Brownian motion. The estimated strong convergence rates for some of the numerical methods in Table \ref{table:method_convergence} are presented in Figure \ref{fig:intro_strong_convergence_graph}. We give a more detailed description of our experiments in Section \ref{sect:experiments}.

\begin{figure}[ht]
\begin{center}
\includegraphics[width=0.9\textwidth]{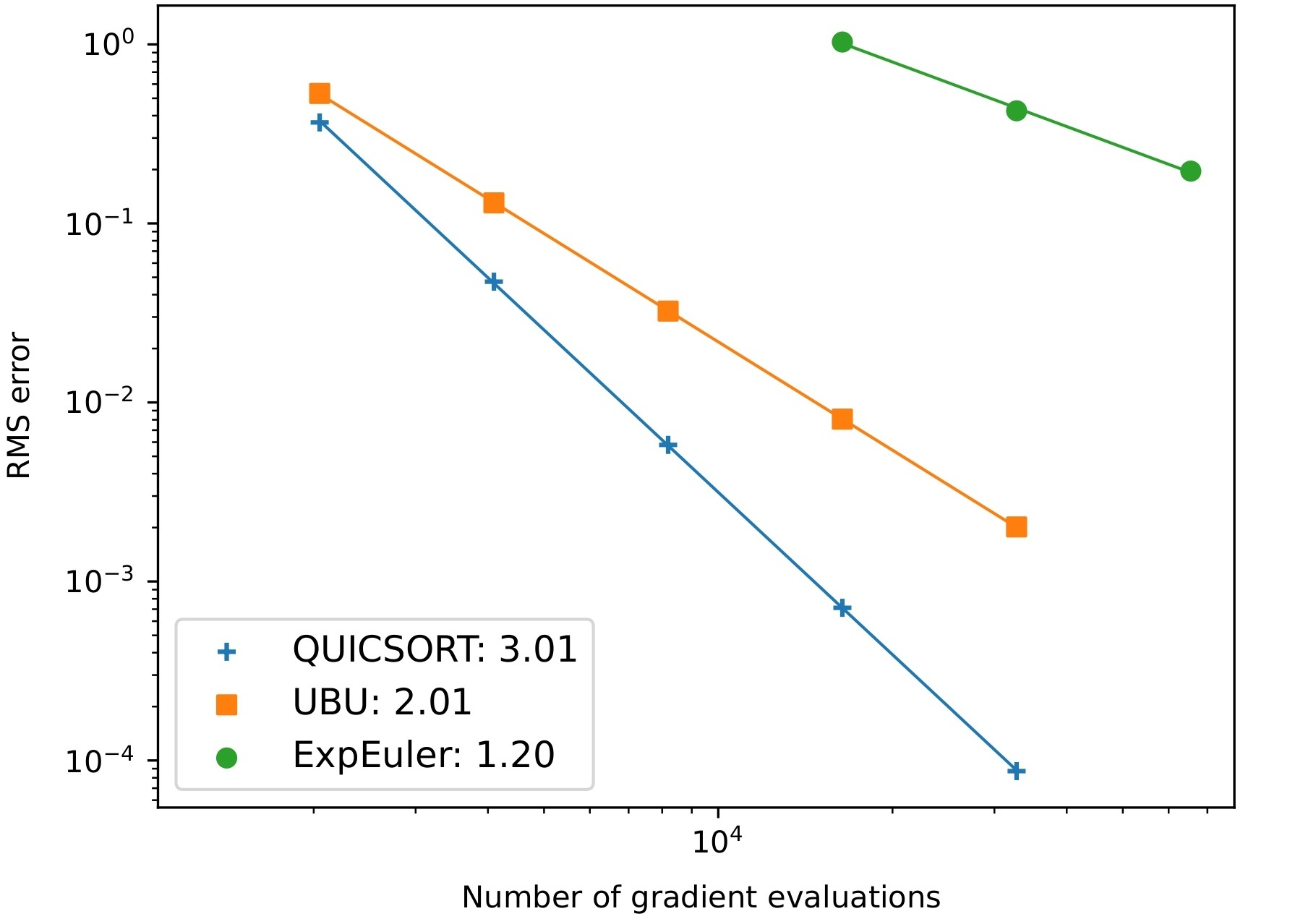}
\caption{Empirical validation of the complexities in Table \ref{table:method_convergence}. Here, the numerical methods are applied to simulate underdamped Langevin dynamics where the potential $f$ is obtained from a Bayesian logistic regression that is applied to the Titanic passenger survival dataset \citep{hendricks2015titanic}. We compute the Monte Carlo estimators for the strong error at $T = 100$.}
\label{fig:intro_strong_convergence_graph}
\end{center}
\end{figure}\vspace{-5mm}

\noindent
To help facilitate applications, our QUICSORT method has been implemented in Diffrax \citep{kidger2021ndes}, which is a high performance open-source software package for ODE and SDE simulation (see \textcolor{blue}{\href{https://docs.kidger.site/diffrax/examples/underdamped_langevin_example/}{docs.kidger.site/diffrax/examples/underdamped{\_}langevin{\_}example}}).\medbreak

\noindent
The paper is organised as follows. In Section \ref{sec:notation_definitions} we establish our key notation, definitions and assumptions. We also provide an outline for the derivation of the QUICSORT method. In Section \ref{sec:error_analysis}, we present our main results (first, second and third order convergence rates for the QUICSORT method) and describe our strategy of proof. In Section \ref{sect:experiments}, we provide experiments demonstrating the convergence and efficacy of the QUICSORT method as an MCMC algorithm for performing Bayesian logistic regression on real-world data. Finally, the technical details of our non-asymptotic error analysis can be found in the appendix.

\section{Notation, definitions and derivation of the QUICSORT method}
\label{sec:notation_definitions}

In this section, we present some of the notation, definitions and assumptions for the paper. We then provide an informal derivation of the QUICSORT method \eqref{eq:QUICSORT}.\medbreak\noindent
Throughout the paper, we use $\|\cdot\|_2$ to denote the standard Euclidean norm on $\mathbb{R}^d$ and $\mathbb{R}^{2d}$. Along with the Euclidean norm, we will use the usual inner product $\langle \m \cdot\m, \m\cdot\m\rangle$ with $\langle v, v\rangle = \|v\|_2^2\m$.\medbreak\noindent
We shall now briefly recall the notions of tensor and operator norms. For $k \geq 1$, a $k$-tensor on $\mathbb{R}^d$ is an element of the $d^k$-dimensional space $\mathbb{R}^{d \times \cdots \times d}$. We can interpret a $k$-tensor on $\mathbb{R}^d$ as a multilinear map into $\mathbb{R}$ with $k$ arguments from $\mathbb{R}^d$. Hence, for $k$-tensor $T$ on $\mathbb{R}^d$ (with $k \geq 2$) and a vector $v \in \mathbb{R}^d$, we define $Tv := T(v, \cdots)$ as a $(k-1)$-tensor on $\mathbb{R}^d$. Using this, we can iteratively define the operator norm of $T$ as
\begin{align*}
	\| T \|_{\text{op}} :=  \begin{cases}\,\,\,\m\|T\|_2\m, & \text{if}\,\,T\in\mathbb{R}^d,\\[5pt] \,\underset{\substack{v\m\in\m\mathbb{R}^d,\\[2pt] \|v\|_2\m\leq\m 1}}{\sup}\|Tv\|_{\text{op}}\m, & \text{if}\,\,T\,\,\text{is a }k\text{-tensor on }\mathbb{R}^d\text{ with }k\geq 2\m.\end{cases}
\end{align*}
With a slight abuse of notation, we will write $\| \cdot \|_2$ instead of $\| \cdot \|_{\text{op}}$ for all $k$-tensors. 

\subsection{Assumptions on $f$} 

The potential function $f: \mathbb{R}^d \to \mathbb{R}$ is assumed to be differentiable throughout this paper. In addition, we will always assume that $f$ to be $m$-strongly convex
\begin{align}
\label{eq:assumption_a1}
	f(y) \geq f(x) + \langle \nabla f(x), y - x \rangle + \frac{1}{2} m \| x - y \|^2_2\,, \tag{A.1}
\end{align}
and that the gradient $\nabla f : \mathbb{R}^d\rightarrow\mathbb{R}^d$ is $M_1$-Lipschitz continuous
\begin{align}
\label{eq:assumption_a2}
	\| \nabla f(x) - \nabla f(y) \|_2 \leq M_1 \| x - y \|_2\,, \tag{A.2}
\end{align}
for all $x, y \in \mathbb{R}^d$. It is straightforward to show that the two conditions \eqref{eq:assumption_a1} and \eqref{eq:assumption_a2} are equivalent to the Hessian of $f$ being positive definite and satisfying $m I_d \preccurlyeq \nabla^2 f(x) \preccurlyeq M_1 I_d$ (where $I_d$ is the $d \times d$ identity matrix and $A \preccurlyeq B$ means that $x^T A x \leq x^T B x$ for $x \in \mathbb{R}^d$). \medbreak\noindent
To establish a second order rate of convergence for the QUICSORT method, we shall assume further that $f$ is twice differentiable and that its Hessian $\nabla^2 f$ is $M_2$-Lipschitz continuous 
\begin{align}
\label{eq:assumption_a3}
	\| \nabla^2 f(x)v - \nabla^2 f(y)v \|_2 \leq M_2 \| x - y \|_2 \, \| v \|_2\,, \tag{A.3}
\end{align}
and, to establish the third order convergence in Theorem \ref{thm:W2_convergence}, we additional assume that $f$ is three times differentiable and and that its third derivative $\nabla^3 f$ is $M_3$-Lipschitz continuous
\begin{align}
 \label{eq:assumption_a4}
	\| \nabla^3 f(x)\m (v,u) - \nabla^3 f(y) \, (v,u) \|_2 \leq M_3 \| x - y \|_2 \, \| v \|_2  \, \| u \|_2\,, \tag{A.4}
\end{align}
for all $x, y, u, v \in \mathbb{R}^d$. 

\subsection{Probability notation}

Suppose $\left( \Omega, \mathcal{F}, \mathbb{P} ; \{ \mathcal{F}_t \}_{t \geq 0} \right)$ is a filtered probability space carrying a standard $d$-dimensional Brownian motion. The only SDE that we study in this paper is \eqref{eq:ULD} which, under our smoothness and convexity assumptions on $f$, admits a unique strong solution that is ergodic and whose stationary measure $\pi$ has a density $\pi (x, v) \propto e^{-f(x) - \frac{1}{2u} \| v \|^2_2}$. The numerical methods for \eqref{eq:ULD} are computed at times $\{ t_n \}_{n\geq 0}$ with $t_0 = 0$, $t_{n+1} > t_n$ and $h_n = t_{n+1} - t_n\m$. However, for deriving our global error estimates, we assume the step sizes $\{h_n\}$ are constant.
For a random variable $X$, taking values in either $\mathbb{R}^d$ or $\mathbb{R}^{2d}$, we define the $\L_p$ norm of $X$ as
\begin{align*}
\| X \|_{\L_p} := \E\big[\m\|X\|^p_2 \, \big]^\frac{1}{p},
\end{align*}
for all $p \geq 1$. 

\subsection{Coupling and Wasserstein distance}

Let $\mu$, $\nu$ be probability measures on $\mathbb{R}^{d}$.
A coupling between $\mu$ and $\nu$ is a random variable $Z = (X,Y)$ for which $X\sim\mu$ and $Y\sim\nu$.
For $p\geq 1$, we define the $p$\hspace{0.25mm}-Wasserstein distance between $\mu$ and $\nu$ as
\begin{align*}
W_p\big(\mu, \nu\big) :=\inf_{(X, Y)\m\sim\m (\mu\times \nu)}\|X-Y\|_{\L_p}\m,
\end{align*}
where the above infimum is taken over all couplings of the random variables $X$ and $Y$ with distributions $X\sim\mu$ and $Y\sim\nu$.\medbreak

In this paper, we are interested in obtaining non-asymptotic bounds for the $2$-Wasserstein distance between the target distribution $p(x)\propto e^{-f(x)}$ and the QUICSORT method \eqref{eq:QUICSORT}. For this, we will use the standard bound between the distribution of $X_n$ and $p(x)\propto e^{-f(x)}$,
\begin{align}
\label{eq:W2_L2_bound}
	W_2\big(\m X_n\m , e^{-f}\m\big) \leq \| X_n - x_{t_n}\|_{\L_2}\m,
\end{align}
where $(x_t\m, v_t)$ denotes the solution of the SDE \eqref{eq:ULD} started from the stationary distribution $(x_0\m, v_0)\sim \pi$. This then implies that $x_t\sim p$ for all $t\geq 0$, which is required for \eqref{eq:W2_L2_bound} to hold.
Both the SDE solution $(x_t\m, v_t)$ and $(X_n\m, V_n)$ are obtained using the same Brownian motion.

\subsection{Definition of the variables $W_n$, $H_n$ and $K_n$}

When Taylor expanding the SDE \eqref{eq:ULD}, we will obtain iterated integrals of Brownian motion. So, in order to derive a high order scheme, we require a method of simulating these integrals. In \citet{foster2021shifted}, the authors use random variables $(W_n\m, H_n\m, K_n)$, which we now define.
\begin{definition}
\label{def:W_K_K}
	Let $\{ t_n \}_{n \geq 0}$ be a sequence of times with $t_0 = 0$, $t_{n+1} > t_n$ and $h_n = t_{n+1} - t_n$. From the same Brownian motion $W = \{W_t\}$, we define the following three random variables,
	\begin{align}
	\label{eq:W_H_K_definition}
		\begin{split}
			W_n &:= W_{t_n, t_{n+1}} = W_{t_{n+1}} - W_{t_n}\m, \\[5pt]
			H_n  &:= \frac{1}{h_n} \int^{t_{n+1}}_{t_n} \bigg( W_{t_n, t} - \frac{t - t_n}{h_n} W_n \bigg)\, dt, \\[5pt]
			K_n  &:= \frac{1}{h^2_n} \int^{t_{n+1}}_{t_n} \bigg(\frac{1}{2} h_n - (t-t_n) \bigg) \bigg( W_{t_n, t} - \frac{t - t_n}{h_n} W_n \bigg) \, dt.
		\end{split}
	\end{align}
\end{definition}
\begin{theorem}
	Let $I_d$ denote the $d\times d$ identity matrix. Then the three random variables, $W_n \sim \mathcal{N} \left( 0, h_n I_d \right)$, $H_n \sim \mathcal{N} \left( 0, \frac{1}{12} h_n I_d \right)$ and $K_n \sim \mathcal{N} \left( 0, \frac{1}{720} h_n I_d \right)$ are independent and\label{thm:levy_space_time_1_2}
	\begin{align}
	\label{eq:levy_space_time_1_2}
		\begin{split}
			\int^{t_{n+1}}_{t_n} W_{t_n,\m r_1} \, dr_1 &= \frac{1}{2} h_n W_n + h_n H_n\m, \\[3pt]
			\int^{t_{n+1}}_{t_n} \int^{r_1}_{t_n} W_{t_n,\m r_2} \, dr_2 \, dr_1 &= \frac{1}{6} h_n^2 W_n + \frac{1}{2} h_n^2 H_n + h_n^2 K_n\m.
		\end{split}
	\end{align}	
\end{theorem}
\begin{proof}
The result was shown in \citet[Theorem 2 and Remark 6]{foster2024approximating}.
\end{proof}
For a detailed derivation of these variables, we refer the reader to \citet{foster2020optimal} and \citet[Section 2]{foster2024approximating}. The key idea underlying our QUICSORT method and its third order convergence is the same idea underlying the SORT method in \citet{foster2021shifted}. Namely, we will replace the Brownian motion driving the SDE with a piecewise linear path.
By constructing the piecewise linear path $\mathcal{W}$ on $[t_n\m, t_{n+1}]$ to match the increment $W_n$ and the integrals \eqref{eq:levy_space_time_1_2}, the resulting ODE approximation matches the desired terms in its Taylor expansion and achieves high order convergence. This was shown in \citet{foster2021shifted}.
More generally, this approach can be used to design high order splitting methods for SDEs (see \citet{foster2024splitting}). In our case, we will define the following piecewise linear path $\mathcal{W}$.
\begin{definition}
\label{def:W_path}
	Let $\{ t_n \}_{n \geq 0}$ be a sequence of times with $t_0 = 0$, $t_{n+1} > t_n$ and $h_n = t_{n+1} - t_n$. We define an $\mathbb{R}^d$-valued piecewise linear path $\mathcal{W} = \{\mathcal{W}_t\}_{t\m\geq\m 0}\m,$ whose $n$-th piece is given by
	\begin{align*}
		\mathcal{W}_t := \begin{cases} \big(H_n + 6K_n\big) + \big(W_n - 12K_n\big)\m \frac{t-t_n}{h_n}\m, & \text{for}\,\,\,t\in(t_n,t_{n+1})\\
						\,\,W_t\m, & \text{for}\,\,t\in\{t_n,t_{n+1}\}\end{cases}.
	\end{align*}
\end{definition}\vspace{-4mm}
\begin{figure}[H]
\begin{center}
\includegraphics[width=\textwidth]{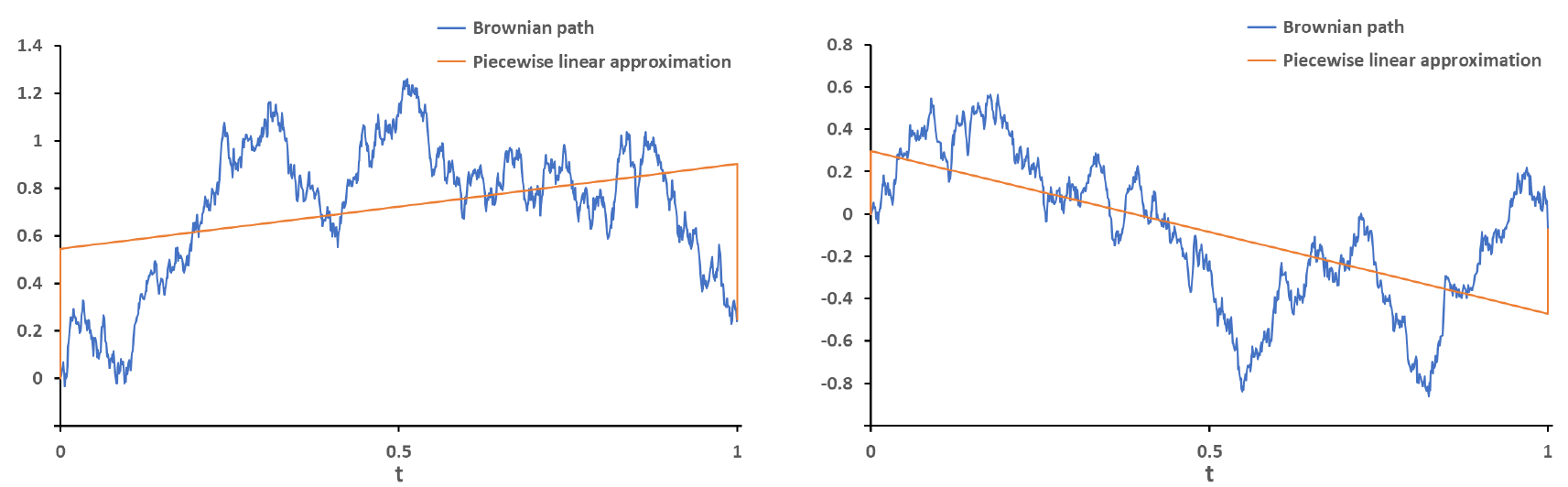}
\caption{Brownian motion alongside the piecewise linear path $\mathcal{W}$ given by Definition \ref{def:W_path} (diagram taken from \cite{foster2021shifted}).\vspace{-2mm}}
\label{fig:piecewise_lin_paths}
\end{center}
\end{figure}

\subsection{Derivation of the QUICSORT method}

In \citet{foster2021shifted}, the authors construct a splitting method for the underdamped Langevin dynamics \eqref{eq:ULD} using the piecewise linear path $\mathcal{W}$ defined in the previous subsection.
We will recall this splitting method and then use it to derive the QUICSORT method \eqref{eq:QUICSORT}.

\begin{definition}[The shifted ODE approximation of ULD]
Let $\{ t_n \}_{n\geq 0}$ be a sequence of times with $t_0 = 0$, $t_{n+1} > t_n$ and $h_n = t_{n+1} - t_n\m$. We construct a numerical solution $\{ (X_n, V_n) \}_{n\geq0}$ by first setting some initial $(X_0, V_0) \sim \pi_0$ and for each $n\geq0$ defining $(X_{n+1}, V_{n+1})$ as
\begin{align*}
	\begin{pmatrix} X_{n+1} \\[1pt] V_{n+1} \end{pmatrix} := \begin{pmatrix} \widetilde{x}^{\m n}_{t} \\[1pt] \widetilde{v}^{\m n}_{t} \end{pmatrix} - (H_n - 6K_n) \begin{pmatrix} 0 \\ \sigma \end{pmatrix},
\end{align*}
where $\sigma := \sqrt{2\gamma u}$ and $\{ (\widetilde{x}^{\m n}_t, \widetilde{v}^{\m n}_t) \}_{t \in [t_n, t_{n+1}]}$ solves the following Langevin-type ODE
\begin{align}\label{eq:langevin_ode}
	\frac{d}{dt}  \begin{pmatrix} \widetilde{x}^{\m n}_t \\[1pt] \widetilde{v}^{\m n}_t \end{pmatrix} = \begin{pmatrix} \widetilde{v}^{\m n}_t \\[1pt] - \gamma \widetilde{v}^{\m n}_t - u \nabla f(\widetilde{x}^{\m n}_t) \end{pmatrix} + (W_n - 12 K_n)\begin{pmatrix} 0 \\ \sigma \end{pmatrix},
\end{align}
with initial condition
\begin{align*}
	\begin{pmatrix} \widetilde{x}^{\m n}_{t_n} \\[1pt] \widetilde{v}^{\m n}_{t_n} \end{pmatrix} := \begin{pmatrix} X_{n} \\[1pt] V_{n} \end{pmatrix} + (H_n + 6K_n) \begin{pmatrix} 0 \\ \sigma \end{pmatrix}.
\end{align*}
\end{definition}

We now wish to find a suitable numerical discretisation of the Langevin-type ODE \eqref{eq:langevin_ode}.
We first let $C_n := \sigma(W_n - 12 K_n)$ and define a change of variables  $\widetilde{w}^{\m n}_t := e^{\gamma t} \widetilde{v}^{\m n}_t\m$. This gives,
\begin{align*}
	\frac{d}{dt}  \begin{pmatrix} \widetilde{x}^{\m n}_t \\ \widetilde{w}^{\m n}_t \end{pmatrix} = \begin{pmatrix} e^{-\gamma t} \widetilde{w}^{\m n}_t \\ - u e^{\gamma t} \nabla f(\widetilde{x}^{\m n}_t) + \sigma e^{\gamma t} C_n \end{pmatrix}.
\end{align*}
Taking the integral of both sides and expressing everything in terms of $\widetilde{x}^{\m n}$ and $\widetilde{v}^{\m n}$ produces
\begin{align*}
		 \widetilde{x}^{\m n}_{t_{n+1}} &=  \widetilde{x}^{\m n}_{t_n} + \frac{1 - e^{-\gamma h_n}}{\gamma}  \widetilde{v}^{\m n}_{t_n} - u \int^{t_{n+1}}_{t_n} \int^{r_1}_{t_n} e^{- \gamma (r_1 - r_2)} \nabla f ( \widetilde{x}^{\m n}_{r_2}) \, dr_2 \, dr_1 + \sigma \frac{e^{-\gamma h_n} + \gamma h_n - 1}{\gamma^2 h_n}  C_n\m, \\[3pt]
		  \widetilde{v}^{\m n}_{t_{n+1}} &= e^{-\gamma h_n}  \widetilde{v}^{\m n}_{t_n} - u \int^{t_{n+1}}_{t_n} e^{-\gamma (t_{n+1} - r_1)} \nabla f ( \widetilde{x}^{\m n}_{r_1}) \, dr_1 + \sigma \frac{1 - e^{-\gamma h_n}}{\gamma h_n} C_n\m.
\end{align*}
We then approximate the integral terms using a two-point Gauss–Legendre quadrature. That is, we use the quadrature points $\lambda_{\pm} = \frac{3 \pm \sqrt{3}}{6}$ and obtain
\begin{align*}
	\int^{t_{n+1}}_{t_n} e^{-\gamma (t_{n+1} - r_1)} \nabla f ( \widetilde{x}^{\m n}_{r_1}) \, dr_1 \approx \frac{1}{2} h_n \bigg(e^{-\lambda_+ \gamma h_n} \nabla f \big( \widetilde{x}^{\m n}_{t_{n} + \lambda_-} \big) + e^{-\lambda_- \gamma h_n} \nabla f \big( \widetilde{x}^{\m n}_{t_{n} + \lambda_+} \big) \bigg),
\end{align*}
and by additionally applying Fubini's theorem
\begin{align*}
	&\int^{t_{n+1}}_{t_n} \int^{r_1}_{t_n} e^{- \gamma (r_1 - r_2)} \nabla f ( \widetilde{x}^{\m n}_{r_2}) \, dr_2 \, dr_1 \\
	&\hspace{10mm} = \int^{t_{n+1}}_{t_n} \int^{t_{n+1}}_{r_2} e^{- \gamma (r_1 - r_2)}\, dr_1 \, \nabla f ( \widetilde{x}^{\m n}_{r_2}) \, dr_2 \\[3pt]
	&\hspace{10mm} \approx \frac{1}{2}h_n\bigg(\frac{1 - e^{-\lambda_+ \gamma h_n}}{\gamma} \nabla f \big( \widetilde{x}^{\m n}_{t_{n} + \lambda_-} \big) + \frac{1 - e^{-\lambda_- \gamma h_n}}{\gamma} \nabla f \big( \widetilde{x}^{\m n}_{t_{n} + \lambda_+} \big)\bigg)\m.	
\end{align*}
The terms $X^{(1)}_n$ and $X^{(2)}_n$ in the QUICSORT method are then designed to approximate $\widetilde{x}^{\m n}_{t_{n} + \lambda_-}$ and $ \widetilde{x}^{\m n}_{t_{n} + \lambda_+}$. To reduce the number of gradient evaluations, the formula for $X^{(1)}$ does not involve $\nabla f$. This is inspired by the UBU scheme \citep{sanz2021wasserstein} which approximates the SDE solution at the midpoint of the interval before evaluating $\nabla f$.
But, as $X^{(1)}$ has no $\nabla f$ term, the resulting scheme does not match the desired expansion. Thus, to correct this, we use a coefficient of $\frac{1}{3}$ instead of $\lambda_+$ for the gradient term in $X^{(2)}$. \medbreak

\noindent
As mentioned above, the QUICSORT method is a natural extension to the UBU scheme, with the difference being that we use a two-point Gaussian quadrature to approximate the gradient terms instead of a midpoint approximation. Similarly, just as the UBU scheme has a contractivity property for strongly convex $f$, we can show QUICSORT is also contractive. This is precisely the reason why we can establish third order convergence guarantees for the QUICSORT method and not the SORT method originally proposed in \citet{foster2021shifted}.

\section{Error analysis of the QUICSORT method}
\label{sec:error_analysis}

In this section, we will present our main results regarding the 2-Wasserstein convergence of our new method \eqref{eq:QUICSORT} to the stationary distribution of \eqref{eq:ULD}. Our main theorem is given below.

\begin{theorem}[Global error bound]
Let $\{(X_n, V_n) \}_{n \geq 0}$ be the QUICSORT method \eqref{eq:QUICSORT} and let $\{(x_t, v_t ) \}_{t \geq 0}$ denote the underdamped Langevin diffusion, given by \eqref{eq:ULD}, with the same underlying Brownian motion. Suppose that the assumptions \eqref{eq:assumption_a1} and \eqref{eq:assumption_a2} hold. Suppose further that $(x_0, v_0) \sim \pi$, the unique stationary measure of the diffusion process and that both processes have the same initial velocity $V_0 = v_0 \sim \mathcal{N} (0, u I_d)$. Let $0 < h \leq 1$ be fixed and $\gamma \geq 2 \sqrt{uM_1}$. Then for all $n \geq 0$, the global error at time $t=t_n$ has the bound\label{thm:global_error_bound}
	\begin{align}
		\| X_n - x_{t_n} \|_{\L_2} \leq  c\m e^{- n \alpha h} \| X_0 - x_0 \|_{\L_2} +  C_1 \m \sqrt{d} \m h,
	\end{align}
	where the contraction rate $\alpha$ is defined as
	\begin{align}
		\alpha = \frac{1}{\gamma} \min \bigg(\gamma^2 - uM_1, \, um, \, \lambda_- \gamma^2 - \frac{1}{2} u( M_1 - m) \bigg) - K h,
	\end{align}
	and the constants $c, C_1, K > 0$ depend only on $\gamma$, $u$, $m$ and $M_1\m$. Therefore, if $h$ is sufficiently small, we have $\alpha > 0$. Under the additional assumption \eqref{eq:assumption_a3}, the error at time $t_n$ satisfies
	\begin{align}
		\| X_n - x_{t_n} \|_{\L_2} \leq  c e^{-n \alpha h} \| X_0 - x_0 \|_{\L_2} +  C_2 \m d \m h^2,
	\end{align}
	where $C_2 > 0$ is a constant depending on $\gamma$, $u$, $m$, $M_1$ and $M_2\m$.\smallbreak\noindent
	Finally, under the additional assumption \eqref{eq:assumption_a4}, the global error at time $t=t_n$ has the bound
	\begin{align}
		\| X_n - x_{t_n} \|_{\L_2} \leq  c e^{-n \beta h} \| X_0 - x_0 \|_{\L_2} +  C_3 \m d^{1.5} \m h^3,
	\end{align}
	where $\beta = \frac{1}{2}\alpha$ and $C_3 > 0$ is a constant depending on $\gamma$, $u$, $m$, $M_1$, $M_2$ and $M_3\m$. 
\end{theorem}
A proof for this theorem can be found in Appendix \ref{sec:appendix_D}. The various constants are given explicitly throughout the appendices. The general recipe for the error analysis is as follows.
\begin{itemize}
\item At each time step $t_{n+1}$, we wish to bound the $\L_2$ error between the underdamped Langevin diffusion \eqref{eq:ULD} and the QUICSORT method \eqref{eq:QUICSORT} where we assume the two processes are synchronously coupled (i.e.~driven by the same Brownian motion $W$).

\item We introduce a third process, obtained by applying the QUICSORT approximation with step size $h_n$ to the diffusion process at each time $t_n$.

\item By showing that the QUICSORT method is contractive, we obtain an exponentially decaying bound for the $\L_2$ distance between our approximation and this third process.

\item We establish local error bounds between the true diffusion \eqref{eq:ULD} and the third process. This allows us to estimate the $\L_2$ error at time $t_{n+1}$ between the true diffusion \eqref{eq:ULD} and the QUICSORT approximation \eqref{eq:QUICSORT} in terms of the same error at time $t_n$.

\item By iteratively unfolding this local error bound, we obtain a global $\L_2$ error estimate over the entire length of the Markov chain. In particular, due to the contractivity of the QUICSORT method when $f$ is strongly convex, this bound holds for all $n\geq 0$.
\end{itemize}\medbreak\noindent
This approach is illustrated in Figure \ref{fig:error_analysis_diagram}:
\begin{figure}[H]
\begin{center}
\includegraphics[width=0.9\textwidth]{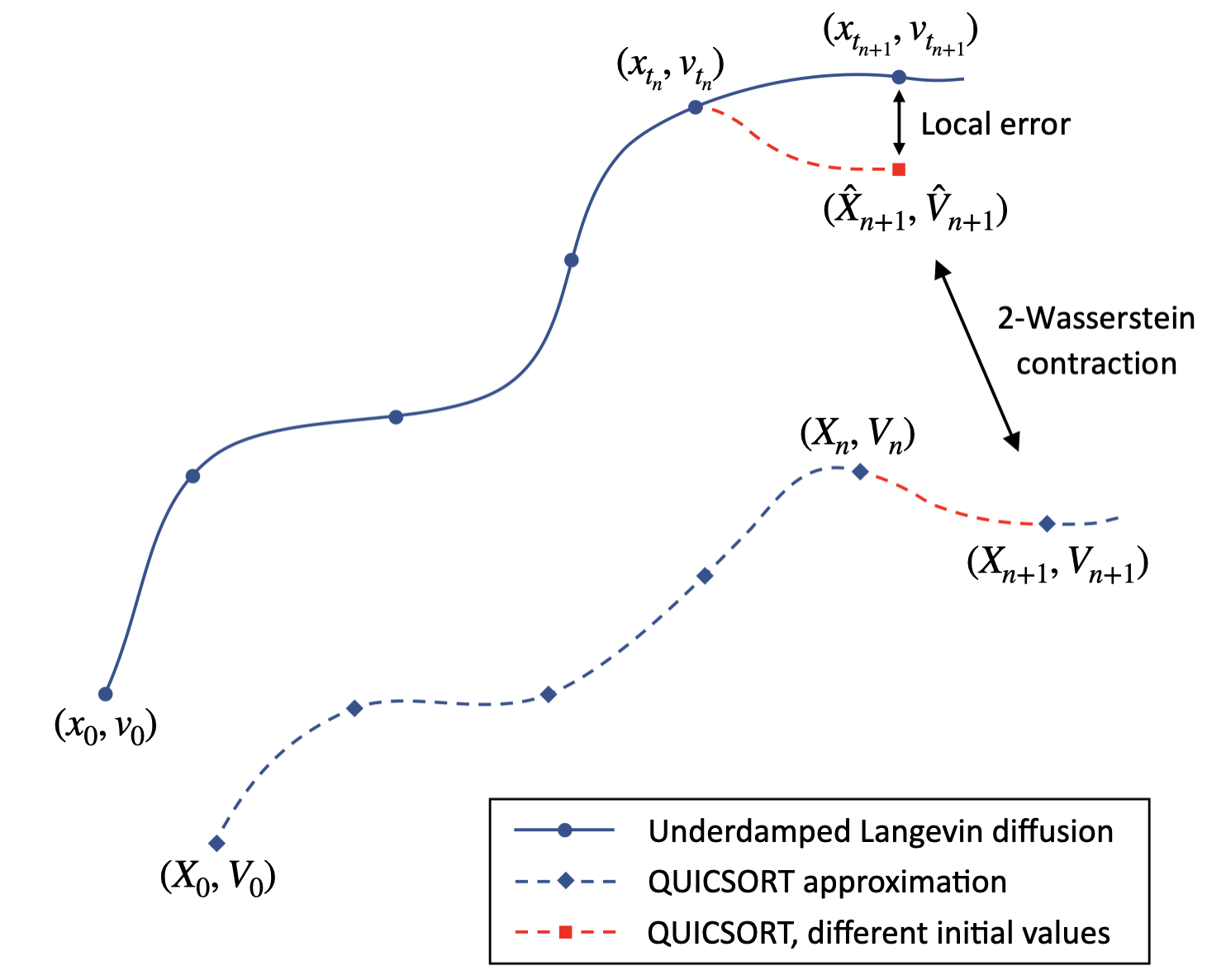}
\caption{Diagram outlining the general strategy for our error analysis.}
\label{fig:error_analysis_diagram}
\end{center}
\end{figure}

\noindent
For the rest of the section, we outline the key results used in each step of the error analysis.

\subsection{Exponential contractivity}

Typically the local errors between the approximation and the diffusion process propagate like a random walk and we would not be able to bound the global $\L_2$ error for all $n \geq 0$. Fortunately, it is possible to remedy this by applying a suitable $\L_2$ contractivity argument.
For ULD, contractivity does not hold directly for the solution process $\{ (x_t, v_t) \}_{t\geq 0}$, but instead holds after introducing a coordinate transformation \citep{eberle2019couplings}. In this paper we use the approach introduced by \citet{dalayan2020sampling} which considers,
\begin{align*}
	\left\{ \begin{pmatrix} w x_t + v_t \\ z x_t + v_t\end{pmatrix} \right\}_{t \geq 0},
\end{align*}
where $w \in [0, \frac{1}{2} \gamma )$ and $z := \gamma - w$ to obtain $\L_2$ contractivity. However, to establish our higher order convergence guarantees, we prove a similar contraction result between two synchronously coupled steps of our QUICSORT method \eqref{eq:QUICSORT}. This allows us to bound the local error at each time step in terms of global $\L_p$ bounds of the underlying diffusion process.\medbreak

\noindent
Our main result for this step is as follows.

\begin{theorem}[Exponential contractivity of the QUICSORT method]
Consider the application of the QUICSORT method from different initial conditions  $\left(X_n, V_n \right)$ and $\left( Y_n, U_n \right)$ at time $t_n$. Let $\left( X_{n+1}, V_{n+1} \right)$ and $\left( Y_{n+1}, U_{n+1} \right)$ be the corresponding approximations obtained at time $t_{n+1}$ using the same Brownian motion. We define the coordinate transformations:\label{thm:main_contractivity_theorem}
	\begin{align}
		\begin{split}
			\begin{pmatrix} W_{n} \\[3pt] Z_{n} \end{pmatrix} &= \begin{pmatrix} \left( w X_{n} + V_{n} \right) - \left( w Y_{n} + U_{n} \right) \\[3pt]  \left( z X_{n} + V_{n} \right) - \left( z Y_{n} + U_{n} \right) \end{pmatrix}, \\[5pt]
			\begin{pmatrix} W_{n+1} \\[3pt] Z_{n+1} \end{pmatrix} &= \begin{pmatrix} \left( w X_{n+1} + V_{n+1} \right) - \left( w Y_{n+1} + U_{n+1} \right) \\[3pt]  \left( z X_{n+1} + V_{n+1} \right) - \left( z Y_{n+1} + U_{n+1} \right) \end{pmatrix},
		 \end{split}
	\end{align}
	where $w \in \left[0, \frac{1}{2} \gamma \right)$ and $z = \gamma - w$. Then under the assumptions \eqref{eq:assumption_a1} and \eqref{eq:assumption_a2}, we have
	\begin{align}
		\left( \| W_{n+1} \|_{\L_2} +  \| Z_{n+1} \|_{\L_2} \right) \leq e^{-\alpha h_n} \big( \| W_n \|_{\L_2} +  \| Z_n \|_{\L_2}\big),
	\end{align}
	for all $n \geq 0$ and $h_n\in (0,1]$, where the contraction rate $\alpha$ is given by
	\begin{align}
		\alpha = \frac{1}{\gamma - 2 w} \min\bigg((\gamma - w)^2 - uM_1, \, um - w^2, \, \big(\lambda_- (\gamma - w)^2 - \lambda_+ w^2 \big) - \frac{1}{2} u( M_1 - m)\bigg) - K h_n\m,
	\end{align}
	for some constant $K > 0$ not depending on $h_n\m$. 
\end{theorem}
A proof of the above can be found in Appendix \ref{sec:appendix_B}.

\subsection{Global $\L_p$ bounds for underdamped Langevin dynamics}

Over each time interval $[t_n, t_{n+1}]$, we wish to iteratively bound the error at $t_{n+1}$ in terms of the error at time $t_n$. In order to obtain estimates independently of the current position and momentum of the SDE, we use the global $\L_p$ bounds of underdamped Langevin dynamics.
In particular, by using the moments of the momentum's stationary distribution ($\propto e^{-f}$) and extending \citet[Lemma 2]{dalalyan2017further}, we will establish the following global $\L_p$ bounds.

\begin{theorem}[Global $\L_p$ bounds for ULD]
	Let $\{ (x_t, v_t) \}_{t \geq 0}$ be the underdamped Langevin diffusion \eqref{eq:ULD} with initial condition $(x_0, v_0) \sim \pi$, the stationary distribution of the process. Then, for all $p \geq 1$, \label{thm:global_Lp_bounds}
	\begin{align}
		\| v_t \|_{\L_{2p}} \leq C(p) \sqrt{u \m d}\m,
	\end{align}
	where $C(p) = \left[ (2p - 1)!! \right]^{\frac{1}{2p}}$. Here, $!!$ denotes the standard double factorial for odd numbers.\smallbreak
	
	\noindent
	In addition, under the assumption that the gradient of $f$ is $M_1$-Lipschitz continuous, we have
	\begin{align}
		\| \nabla f(x_t) \|_{\L_{2p}} \leq C(p) \sqrt{M_1 \m d}\m,
	\end{align}
	for all $p \geq 1$.
\end{theorem}
A proof of the above can be found in Appendix \ref{sec:appendix_A}.

\subsection{Local error estimates}

We wish to estimate the absolute average difference (local weak error) and the mean squared difference (local strong error) between ULD and its QUICSORT approximation over the interval $[t_n, t_{n+1}]$ when the processes are synchronously coupled and coincide at time $t_n$.
By applying a contraction argument to the approximation and not to the SDE itself, we can use the global $\L_p$ bounds for the diffusion process when bounding each local error. Furthermore, under the smoothness assumptions on the scalar potential $f$, we are also able to estimate all of the error terms involving derivatives of $f$. 

\begin{theorem}[Local error bounds]
Let $Q_n =  (X_n, V_n)$ be the QUICSORT scheme \eqref{eq:QUICSORT} and $q_t =  (x_t, v_t )$ denote the underdamped Langevin diffusion \eqref{eq:ULD} with the same underlying Brownian motion.  Suppose the assumptions \eqref{eq:assumption_a1} and \eqref{eq:assumption_a2} hold and assume further that $(x_0, v_0) \sim \pi$, the unique stationary distribution of the diffusion process. Let $0 < h_n \leq 1$ be fixed. Then, if $Q_n = q_{t_n}$ at time $t_n$, the local strong and weak errors at time $t_{n+1}$ are bounded by\label{thm:local_error_bounds}
\begin{equation}
	\big\| Q_{n+1} - q_{t_{n+1}} \big\|_{\L_2} \leq C^{\m (1)}_{s} \m \sqrt{d} \m h_n^2\m, \quad \quad \big\|\m \E \big[ Q_{n+1} - q_{t_{n+1}} \big] \big\|_2 \leq C^{\m (1)}_{w} \m \sqrt{d} \m h_n^2\m,
\end{equation}
where $C^{\m (1)}_s, C^{\m (1)}_w > 0$ are finite constants depending on $\gamma$, $u$, $m$ and $M_1\m$.\smallbreak
\noindent
Under the additional assumption \eqref{eq:assumption_a3}, the local strong and weak errors are bounded by
\begin{equation}
	\big\| Q_{n+1} - q_{t_{n+1}} \big\|_{\L_2} \leq C^{\m (2)}_{s} \m d \m h_n^3\m, \quad \quad \big\|\m \E \big[ Q_{n+1} - q_{t_{n+1}} \big] \big\|_2 \leq C^{\m (2)}_{w} \m d\m h_n^3\m,
\end{equation}
where $C^{\m (2)}_s, C^{\m (2)}_w > 0$ are finite constants depending on $\gamma$, $u$, $m$, $M_1$ and $M_2\m$.\smallbreak
\noindent
Finally, under the additional assumption \eqref{eq:assumption_a4}, the local errors are bounded by
\begin{equation}
	\big\| Q_{n+1} - q_{t_{n+1}} \big\|_{\L_2} \leq C^{\m (3)}_{s} \m d^{1.5} \m h_n^{3.5}, \quad \quad \big\|\m \E \big[ Q_{n+1} - q_{t_{n+1}} \big] \big\|_2 \leq C^{\m (3)}_{w} \m d^{1.5} \m h_n^4\m,
\end{equation}
where $C^{\m (3)}_s, C^{\m (3)}_w > 0$ are finite constants depending on $\gamma$, $u$, $m$, $M_1$, $M_2$ and $M_3\m$. 
\end{theorem}

\noindent 
A proof for the above theorem can be found in Appendix \ref{sec:appendix_C}. 

\section{Numerical Experiments}\label{sect:experiments}

In this section, we will present experiments to demonstrate the convergence properties of the proposed QUICSORT method \eqref{eq:QUICSORT}. Firstly, we shall compare the QUICSORT method to other approximations of ULD and empirically verify its third order convergence rate. Secondly, we will show the efficacy of QUICSORT as an MCMC algorithm for performing Bayesian logistic regression on real-world data. In particular, in our highest dimensional datasets, we observe that QUICSORT (and the UBU splitting) have faster mixing times than the popular ``No U-Turn Sampler'' (NUTS) introduced in \citet{hoffman2014NUTS}.\medbreak

\noindent
The code for our experiments can be found at \textcolor{blue}{\href{https://github.com/andyElking/ThirdOrderLMC}{github.com/andyElking/ThirdOrderLMC}}. The ULD-based MCMC sampling was done by simulating ULD using Diffrax \citep{kidger2021ndes} and we used the package NumPyro \citep{phan2019numpryo} to generate samples using NUTS.\medbreak

\noindent
To begin, we will define the logistic regression model that we aim to generate samples from.
Given a labelled dataset $\mathcal{D} = \{(x_i\m, y_i)\}_{1\m\leq\m i\m\leq\m m}$ with features $x_i\in\mathbb{R}^d$ and labels $y_i\in\{-1, 1\}$, the logistic regression aims to find weights $\theta \in \mathbb{R}^d$ and a bias $b\in\mathbb{R}$ such that the relation
\begin{align*}
\mathbb{P}_{\theta,\m b}\big(y_i = 1\m |\m x_i \big) = \frac{1}{1 + \exp(-(\langle \theta, x_i\rangle + b))},
\end{align*}
is a good model for the data. Specifically, we would like to maximise the likelihood 
\begin{align}\label{eq:log_reg_likelihood}
p(\mathcal{D}\m|\m\theta, b) = \prod_{i=1}^m \frac{1}{1 + \exp(-y_i(\langle \theta, x_i\rangle + b))}\m.
\end{align}
For Bayesian logistic regression, we impose a prior distribution on the parameters $\tilde{\theta} := (\theta, b)$ and aim to generate samples from the posterior distribution $p(\tilde{\theta} | \mathcal{D})$ given by Bayes' rule:
\begin{align}\label{eq:bayes_rule}
p(\tilde{\theta} | \mathcal{D}) = \frac{p(\mathcal{D}|\tilde{\theta}) p(\tilde{\theta})}{p(\mathcal{D})}\m.
\end{align}
In this section, we will use the following Gaussian prior distribution for the parameters $\tilde{\theta}$,
\begin{align}\label{eq:gaussian_prior}
\tilde{\theta} \sim p(\tilde{\theta})\quad \Longleftrightarrow\quad\theta\sim\mathcal{N}\bigg(0, \frac{1}{2\m\mathrm{Var}(\{x_i\})}I_d\bigg),\quad b\sim\mathcal{N}(0,1),
\end{align}
where $\mathrm{Var}(\{x_i\})$ is the empirical variance of the dataset $\{x_1\m,\cdots, x_m\}$. We rescale by this variance simply as a way to normalise the inner products in \eqref{eq:log_reg_likelihood}. Whilst computing $p(\mathcal{D})$ involves an intractable integral, it disappears when we compute $\nabla f(\tilde{\theta}) = \nabla(-\log(p(\tilde{\theta} | \mathcal{D})))$.\medbreak\noindent
Therefore, using Bayes' rule \eqref{eq:bayes_rule}, the scalar potential coming from our logistic regression is
\begin{align}\label{eq:log_potential}
f(\tilde{\theta}\m) := \sum_{i=1}^m \log\hspace{-0.5mm}\big(1+\exp\big(\hspace{-0.5mm}-y_i \big(\langle \theta, x_i\rangle\big) + b\big)\big) + \frac{1}{4\m\mathrm{Var}(\{x_i\})}\|\theta\|_2^2 + \frac{1}{2}b^2\m.
\end{align}
\begin{remark}
The logistic regression loss function is typically convex but not strongly convex. However, due to the quadratic regularisation terms appearing due to the prior distribution, the potential (\ref{eq:log_potential}) is strongly convex. In addition, just as for the standard logistic regression loss, the above potential is smooth and its derivatives are all globally Lipschitz continuous. Thus, it satisfies the assumptions needed in our error analysis of QUICSORT (Theorem \ref{thm:global_error_bound}).
\end{remark}

\subsection{Verifying the third order convergence of QUICSORT}

In this subsection, we will present our experiment for empirically verifying the convergence rates for some of the Langevin methods given in Table \ref{table:method_convergence} (including the QUICSORT method).\medbreak

\noindent
We will be simulating the following underdamped Langevin dynamics with $\gamma = 1$ and $u = 1$,
\begin{align}
\label{eq:ULD2}
\begin{split}
	d\tilde{\theta}_t &= v_t \, dt, \\[2pt]
	dv_t &= - v_t \, dt - \nabla f(\tilde{\theta}_t) \, dt + \sqrt{2}\, dW_t\m,
\end{split}
\end{align}
where the initial condition $\tilde{\theta}_0$ is sampled from the prior distribution \eqref{eq:gaussian_prior} and $v_0\sim\mathcal{N}(0, I_d)$.
The scalar potential $f$ comes from the previously described Bayesian logistic regression applied to the Titanic passenger survival dataset \citep{hendricks2015titanic} and is of the form \eqref{eq:log_potential}. The dimensionality of the dataset is 4, so the corresponding Langevin dynamics has $\tilde{\theta}_t\in\mathbb{R}^5$.\medbreak

\noindent
To estimate the strong (or $\L_2$) convergence rate of each numerical method, we first compute reference solutions by simulating the SDE \eqref{eq:ULD2} using a very fine step size (in our experiment, we use $h_{\text{fine}} = 2^{-16}\m T$). More precisely, since we are estimating the $\L_2$ error by Monte Carlo simulation, we will compute a reference solution for each sample path of Brownian motion.
We used QUICSORT to compute these reference solutions due to its fast convergence rate.\medbreak

\noindent
For each Brownian sample path, we then simulate \eqref{eq:ULD2} using each numerical method with a larger step size and compute the mean squared error against the corresponding reference solution at time $T$. By averaging over sample paths, we obtain the following estimator for
$\E\big[\|X_N - \tilde{\theta}_T\|_2^2\big]$ where $X_N$ is the numerical approximation for the SDE solution at time $T$.
\begin{definition}[Strong error estimator]
Let $T > 0$ be fixed. For $N\geq 1$, we compute a numerical solution $\{(X_n\m, V_n)\}_{0\m\leq\m n\m\leq\m N}$ of \eqref{eq:ULD2} at times $t_n := nh$ using a step size of $h = \frac{T}{N}$. Similarly, let  $\{(X_n^{\mathrm{fine}}\m, V_n^{\mathrm{fine}})\}_{n\m\geq\m 0}$ be an approximation obtained with a smaller step size $h_{\text{fine}}\m$.
We generate $J$ samples of $\{(X\m, V\m, X^{\mathrm{fine}}\m, V^{\mathrm{fine}})\}$ and compute the Monte Carlo estimator:
\begin{align}\label{eq:strong_error_estimator}
S_{N,\m J} := \sqrt{\frac{1}{J}\sum_{j=1}^{J} \big\|X_{N,j} - X_{N, j}^{\mathrm{fine}}\big\|_2^2}\,\,,
\end{align}
where each pair $(X_{N,j}, X_{N, j}^{\mathrm{fine}})$ is computed using the same sample path of Brownian motion and the initial condition $X_{0,j} =  X_{0, j}^{\mathrm{fine}}$ is sampled from the Gaussian prior distribution \eqref{eq:gaussian_prior}.
\end{definition}

\begin{remark}
To simulate the SDE (\ref{eq:ULD2}), we use the Diffrax software package \citep{kidger2021ndes}. Given a single PRNG seed, the ``Virtual Brownian Tree'' in Diffrax can determine an entire sample path of Brownian motion (including the integrals (\ref{def:W_K_K})) up to a user-specified tolerance.
This then makes it straightforward to generate multiple different numerical solutions using the same Brownian sample path. The various algorithmic details and Brownian bridge constructions used by Diffrax's Virtual Brownian Tree can be found in \cite{jelincic2024vbt}.
\end{remark}

\begin{figure}[H]
\begin{center}
\includegraphics[width=0.9\textwidth]{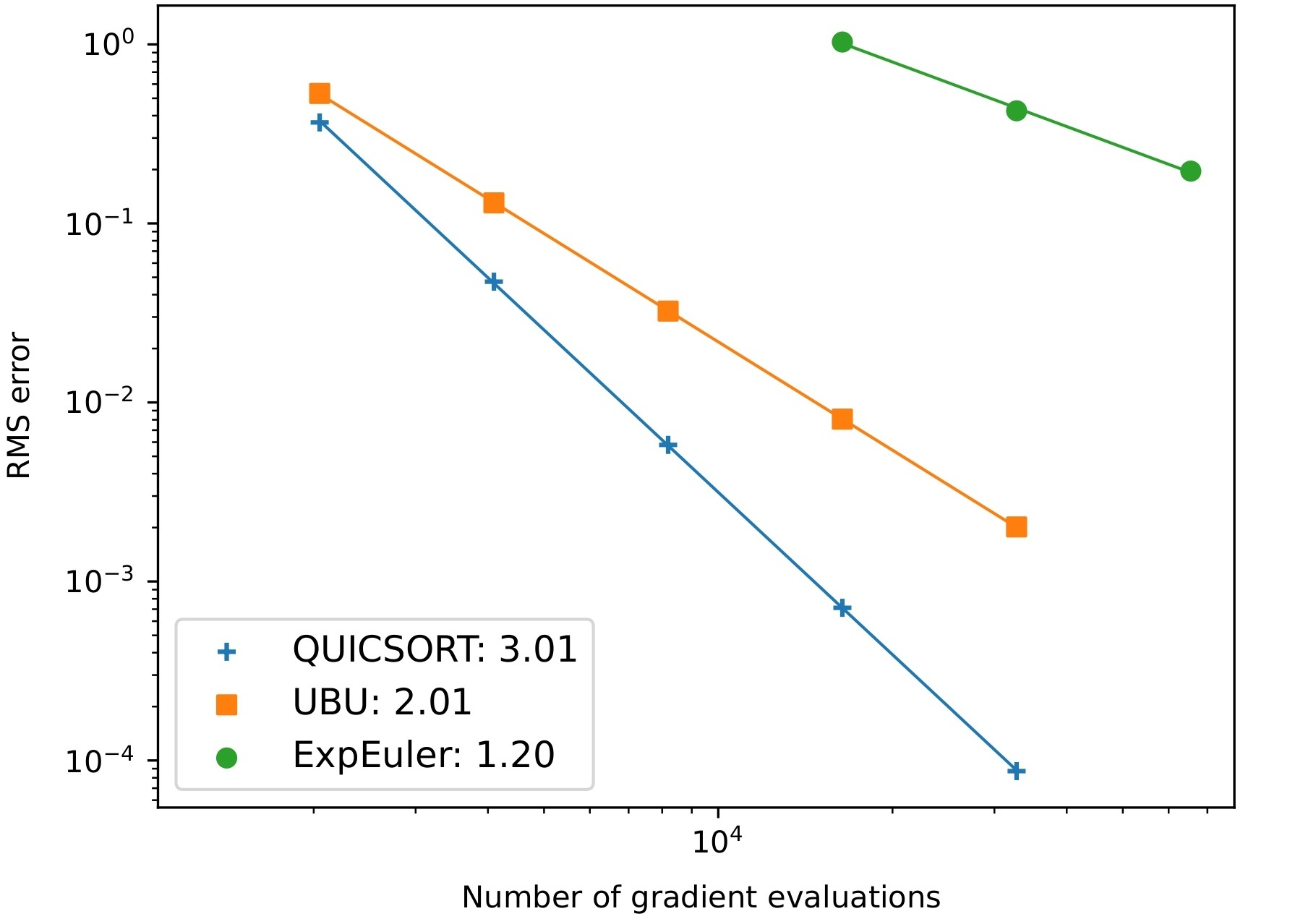}
\caption{Empirical validation of the complexities in Table \ref{table:method_convergence}. Here, each numerical method is used to simulate the underdamped Langevin dynamics \eqref{eq:ULD2} where the potential $f$ is obtained from a Bayesian logistic regression applied to the Titanic passenger survival dataset \citep{hendricks2015titanic}. The Monte Carlo estimator \eqref{eq:strong_error_estimator} for the root mean square (RMS) error is estimated at time $T = 100$ using $J = 1000$ samples. Instead of the number of steps, we report the number of evaluations of $\nabla f$ since that is often the most computationally expensive part of gradient-based MCMC.
}
\label{fig:strong_convergence_graph}
\end{center}
\end{figure}\vspace*{-5mm}

\noindent
From Figure \ref{fig:strong_convergence_graph}, we can see that the numerical methods for ULD  \eqref{eq:ULD2} achieve their theoretical rates of convergence with the QUICSORT method exhibiting a third order convergence rate.\medbreak

\noindent
When the step size is large, the QUICSORT method performs similarly to the UBU splitting method \citep{sanz2021wasserstein} which has only one gradient evaluation per step.
However, as the step size becomes smaller, the third order convergence of QUICSORT makes it significantly more accurate than the UBU scheme -- which is second order convergent.
As one would expect, the exponential Euler scheme proposed in \citep{cheng2018ee} is the least accurate among the numerical methods and only shows a first order convergence rate.\medbreak

\noindent
However, estimating the strong convergence only gives an small indication of each methods' efficacy as an MCMC algorithm. In the next subsection, we shall investigate this further.

\subsection{Mixing time experiments against the No U-Turn Sampler (NUTS)}

In this subsection, we compare the QUICSORT and UBU methods against the No U-Turn Sampler (NUTS) for performing Bayesian logistic regression with target distribution \eqref{eq:bayes_rule} and Gaussian prior \eqref{eq:gaussian_prior}. We test using 14 datasets with dimension $d$ ranging from $3$ to $617$.
The first 13 datasets are those used in the Bayesian logistic experiments in \citet{li2023sampling}.
For the final dataset, we use the ISOLET dataset \citep{cole1991isolet} with $d=617$ and 26 classes (letters of the alphabet), but we only use ``A'' and ``B'' for our logistic regression.\medbreak

\noindent
We would like to observe how quickly the MCMC chains approach the target distribution.
Inspired by \citet{li2023sampling}, we will compute the following metrics to assess this convergence.
\begin{itemize}
\item Energy distance
\begin{align*}
d^{\m 2}\big(\mu, \nu\big) = 2\m\E_{X\sim\mu,\m Y\sim\m\nu}\big[\|X - Y\|_2\big] - \E_{X\sim\mu,\m X^\prime\sim\mu}\big[\|X - X^\prime\|_2\big] - \E_{Y\sim\m\nu,\m Y^\prime\sim\m\nu}\big[\|Y - Y^\prime\|_2\big],
\end{align*}
where the random variables $(X, Y, X^\prime, Y^\prime)$ within the expectations are independent.
This is the Maximum Mean Discrepancy (MMD) obtained from the negative distance kernel $k(x,y) := -\|x-y\|_2$ and is explicitly computable when  $\mu$ and $\nu$ are discrete.

\item 2-Wasserstein distance
\begin{align*}
W_2\big(\mu, \nu\big) := \inf_{(X, Y)\m\sim\m (\mu\times \nu)} \big\| X - Y \big\|_{\L_2} \m,
\end{align*}
where the above infimum is taken over all couplings of the random variables $X$ and $Y$ with distributions $X\sim\mu$ and $Y\sim\nu$. If $\mu$ and $\nu$ are empirical distributions with $J$ points each, it becomes very computationally expensive to find an optimal coupling between $\mu$ and $\nu$ (as $J$ increases). For estimating the 2-Wasserstein distance between distributions, we use the Python Optimal Transport package \citep{flamary2021pot}.
\end{itemize}\bigbreak

\noindent
We present the results for 6 out of the 14 datasets in Figure \ref{fig:mcmc_convergence}. Here, we generate a collection of `ground truth'' samples using NUTS with a sufficiently large number of iterations and, at regular time intervals, compare against the samples from different MCMC algorithms (QUICSORT and UBU applied to ULD \eqref{eq:ULD2} as well as NUTS applied to the target \eqref{eq:bayes_rule}). Results for the other 8 datasets can be found at \textcolor{blue}{\href{https://github.com/andyElking/ThirdOrderLMC}{github.com/andyElking/ThirdOrderLMC}}.\medbreak

\noindent
In half of the datasets, we observed that NUTS converged significantly faster to the target distribution \eqref{eq:bayes_rule} than the numerical methods for ULD. This is perhaps unsuprising since NUTS is widely considered the ``go-to'' algorithm for performing gradient-based MCMC.
However, in the two highest dimensional datasets, splice ($d = 61$) and isolet ($d = 617$), the Langevin numerical methods converge faster than NUTS (after an initial burn-in period). Hence, exploring these algorithms in higher dimensional settings is a topic of future work.\medbreak

\noindent
In these experiments, we were able to achieve good performance from the two Langevin numerical methods with a relatively large step size. However, as we observed in Figure \ref{fig:strong_convergence_graph}, the strong accuracy of QUICSORT and UBU become very similar as the step size increases. This appears to carry over when we investigate the weak convergence of the algorithms as the QUICSORT and UBU methods give near identical performances across all 14 datasets.
In these experiments, we set $\m h_{\mathrm{QUICSORT}} = 2\m h_{\mathrm{UBU}}\m$ so that the QUICSORT and UBU chains require the same number of gradient evaluations (i.e.~to have essentially the same cost). Since both the energy and 2-Wasserstein distance become computationally expensive for large collections of particles, it is difficult to accurately compare the stationary distributions of the QUICSORT and UBU methods. Given the smoothness of the potential, we expect QUICSORT would have a better stationary distribution due to its faster convergence rate.

\begin{figure}
    \centering
    \begin{subfigure}{.33\textwidth}
      \centering
      \caption{banana ($d=3$)}
      \includegraphics[width=\linewidth]{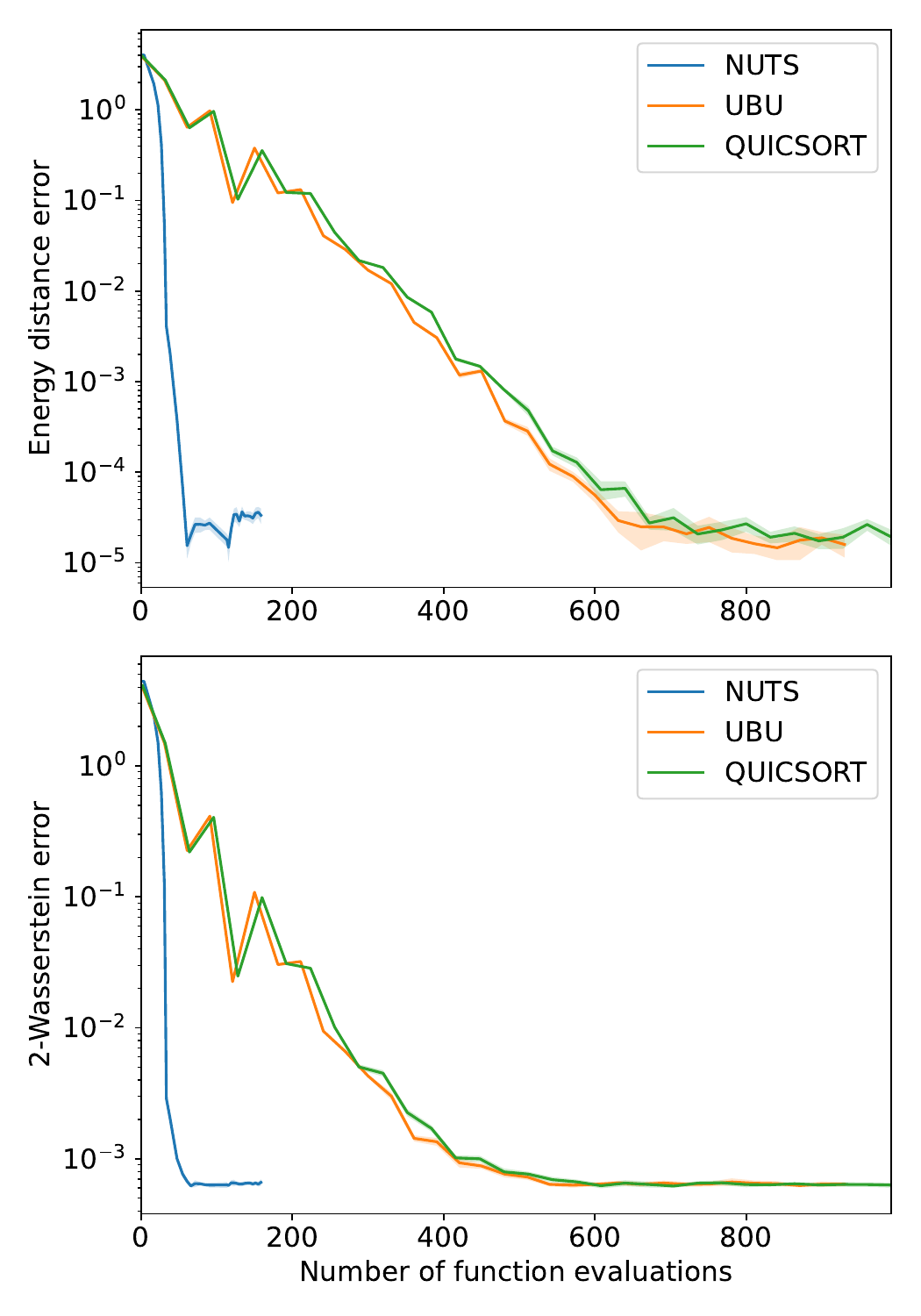}
    \end{subfigure}%
    \begin{subfigure}{.33\textwidth}
      \centering
      \caption{titanic ($d=4$)}
      \includegraphics[width=\linewidth]{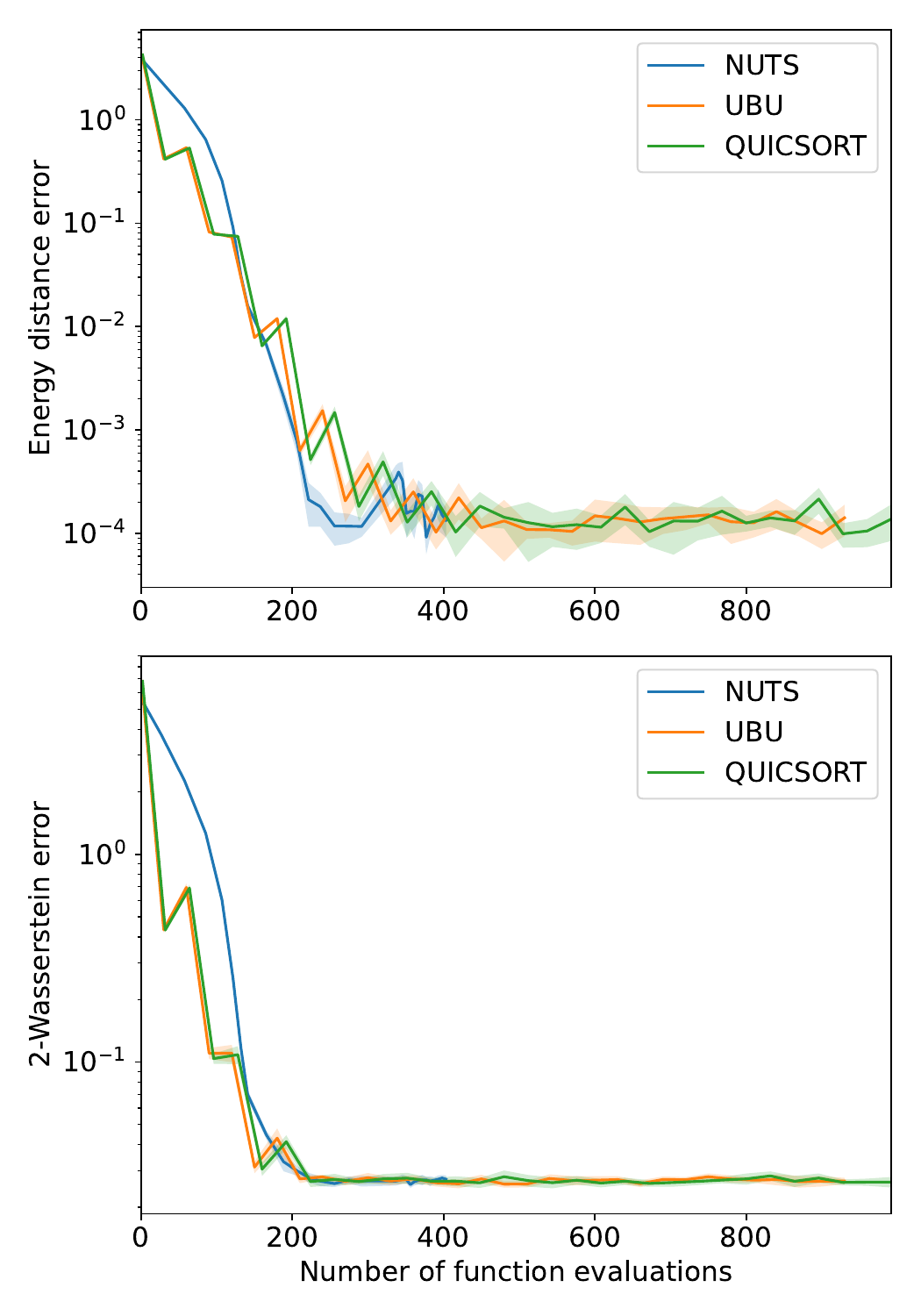}
    \end{subfigure}%
    \begin{subfigure}{.33\textwidth}
      \centering
      \caption{flare solar ($d=10$)}
      \includegraphics[width=\linewidth]{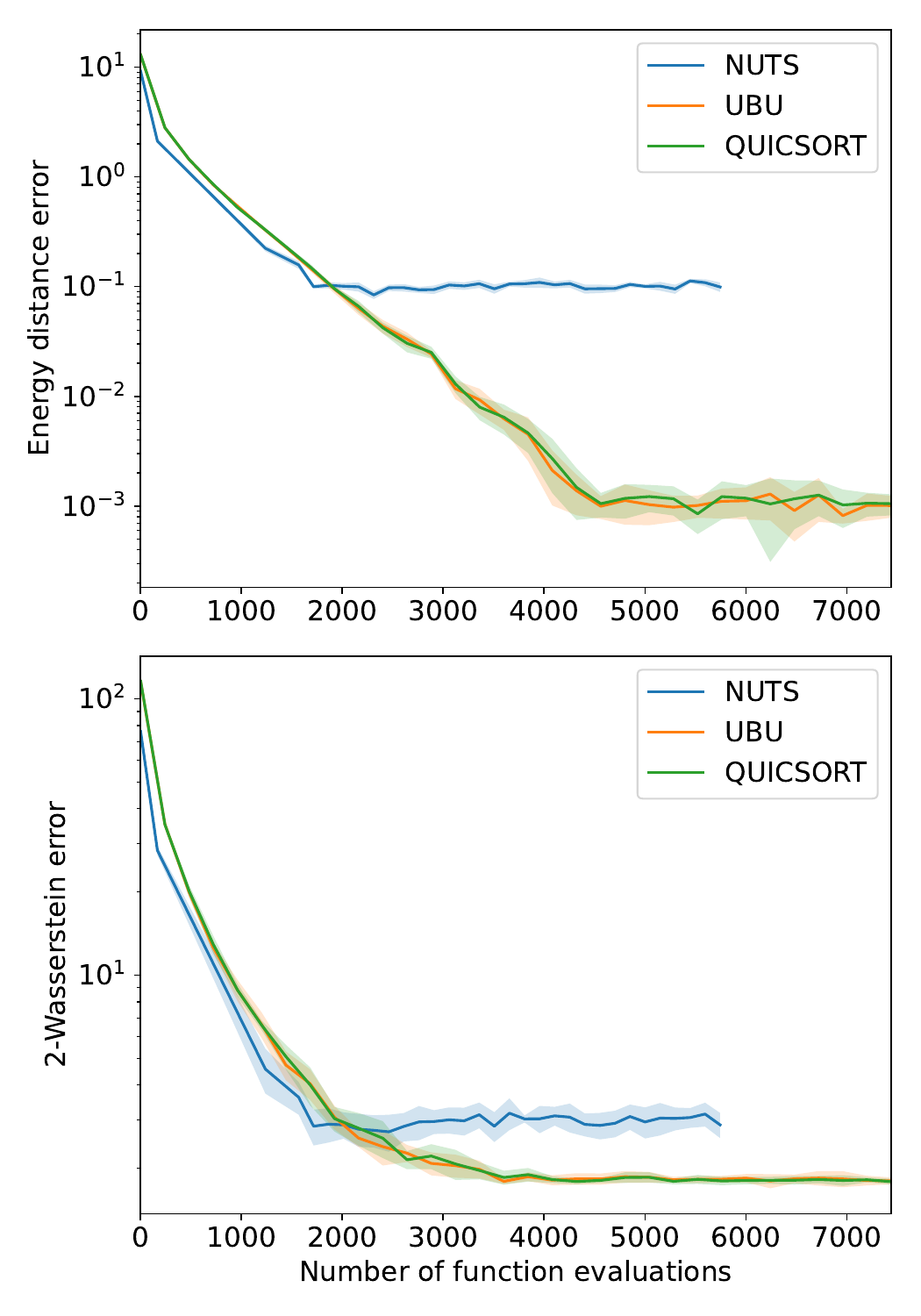}
    \end{subfigure}
    
    \begin{subfigure}{.33\textwidth}
      \centering
      \caption{image ($d=19$)}
      \includegraphics[width=\linewidth]{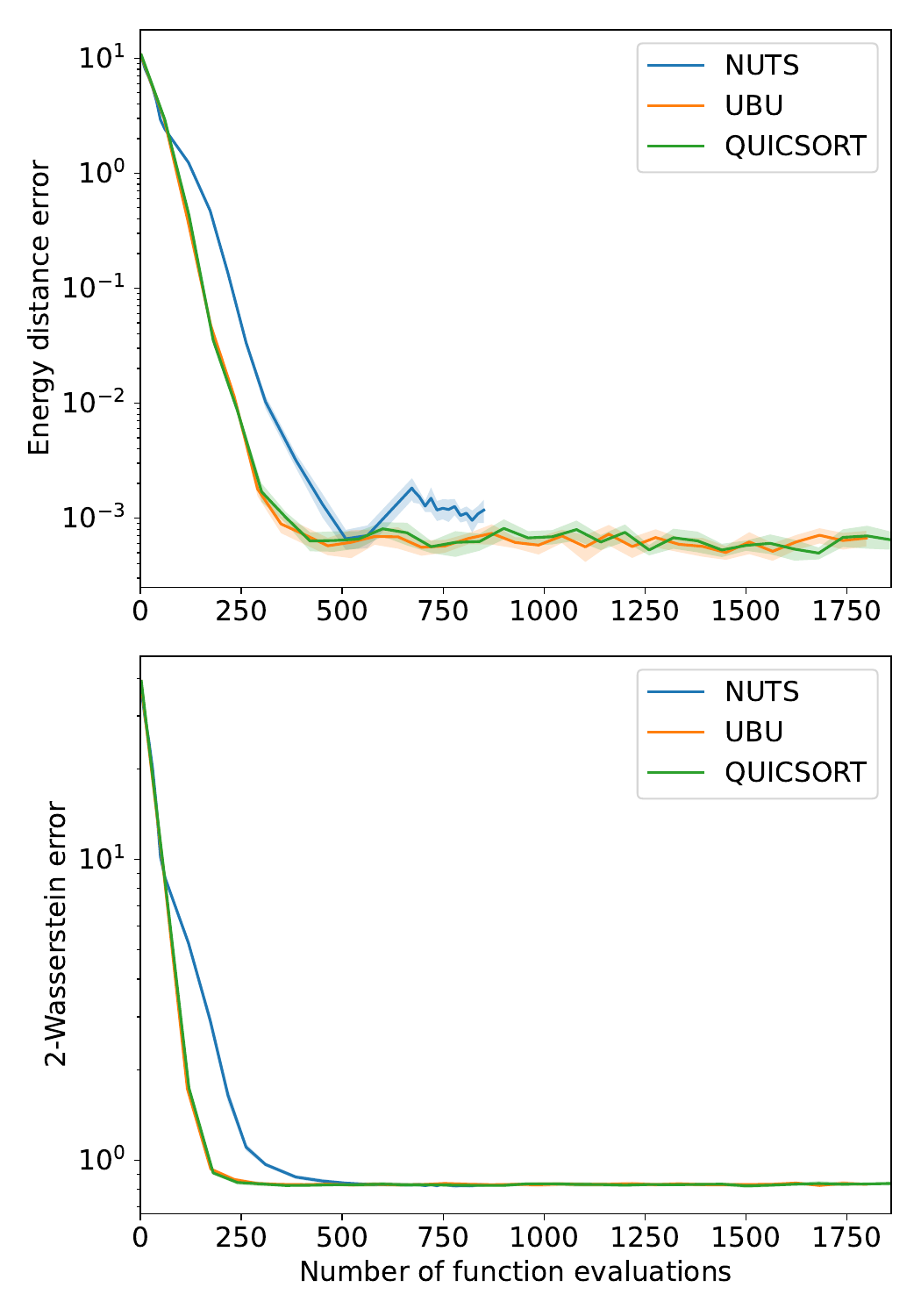}
    \end{subfigure}%
    \begin{subfigure}{.33\textwidth}
      \centering
      \caption{splice ($d=61$)}
      \includegraphics[width=\linewidth]{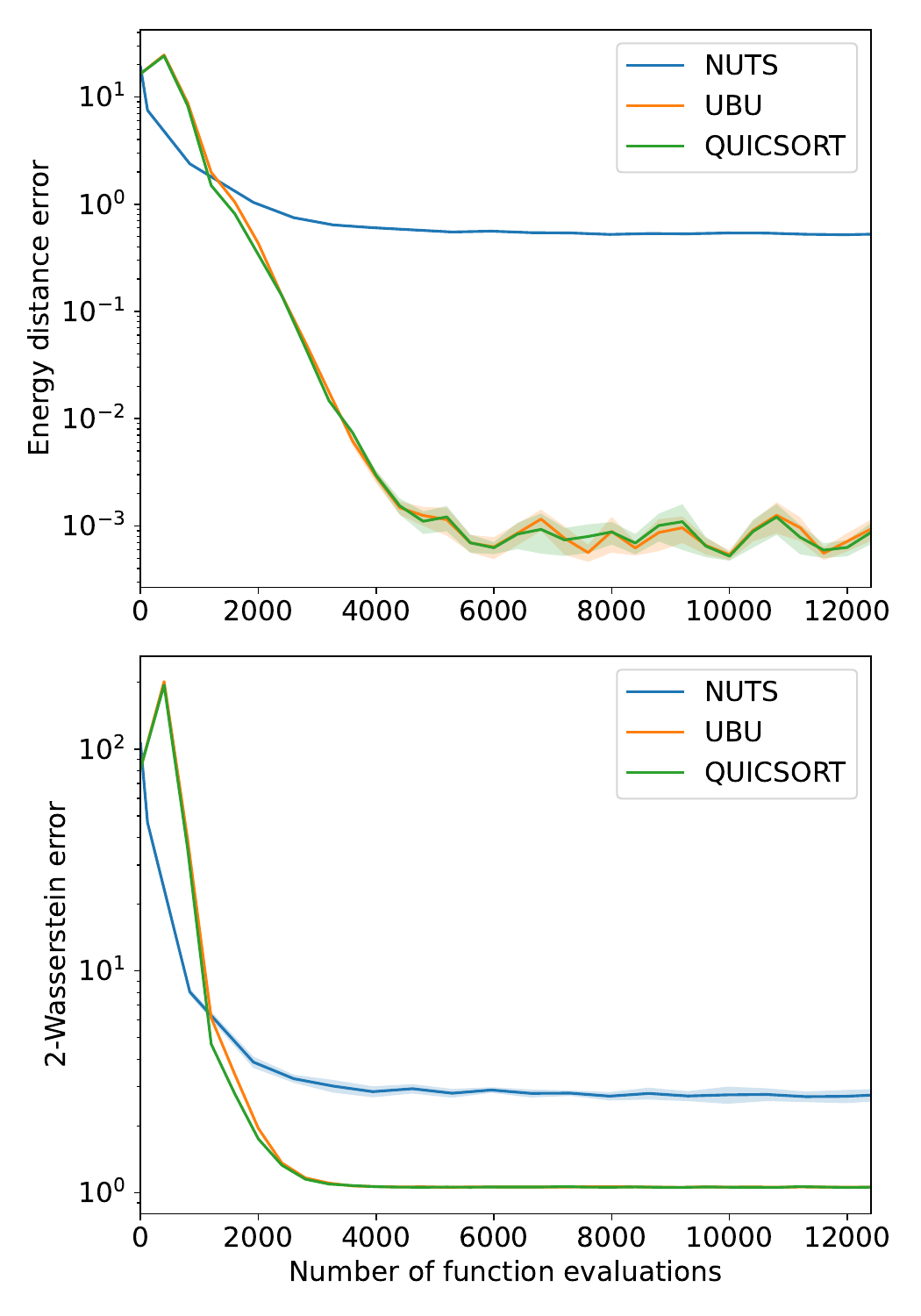}
    \end{subfigure}%
    \begin{subfigure}{.33\textwidth}
      \centering
      \caption{isolet ($d=617$)}
      \includegraphics[width=\linewidth]{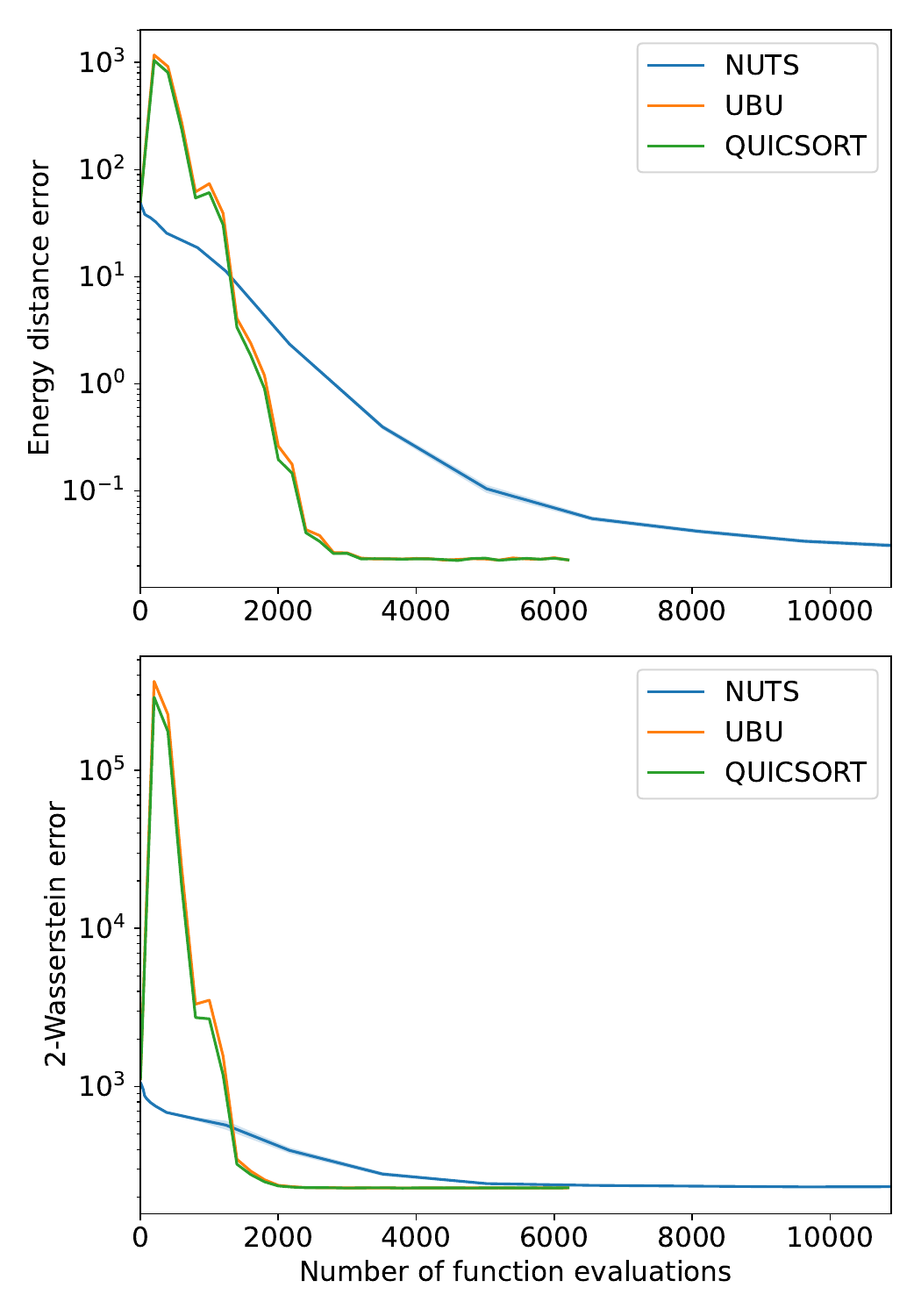}
    \end{subfigure}
    \caption{For each of the 14 datasets we ran $2^{15}$ independent chains of NUTS and the same number of Langevin trajectories simulated using the UBU and QUICSORT methods. At regular intervals along the run we compared these $2^{15}$ particles against the ``ground truth'' samples, which were obtained using NUTS with a sufficiently long burn-in period. We computed both the energy distance and the 2-Wasserstein distance (the latter on a subsample of size $2^{11}$). These distances are plotted against the number of times the gradient $\nabla f$ was evaluated up to that point in the run. We present 6 of the 14 plots here and the rest can be found at \textcolor{blue}{\href{https://github.com/andyElking/ThirdOrderLMC}{github.com/andyElking/ThirdOrderLMC}}. We ran the experiment five times (with different PRNG seeds) and present the standard deviation as the shaded regions. Although different step sizes are used for different datasets, we choose the step size to be as large as possible and set $\m h_{\mathrm{QUICSORT}} = 2\m h_{\mathrm{UBU}}\m$ so that the ULD methods have similar cost.}
    \label{fig:mcmc_convergence}
\end{figure}

\section{Conclusion and future work}

In this paper, we presented a numerical method for underdamped Langevin dynamcis (ULD) called QUICSORT\footnote{\textbf{QU}adrature \textbf{I}nspired and \textbf{C}ontractive \textbf{S}hifted \textbf{O}DE with \textbf{R}unge-Kutta \textbf{T}hree}, which can be seen as a particular discretisation of the ``Shifted ODE'' introduced in \citet{foster2021shifted}. Under strong convexity and smoothness assumptions on the potential $f$, we obtain first, second and third order $\L_2$ error  bounds for QUICSORT. In particular, by establishing a contraction property for QUICSORT, these non-asymptotic bounds hold over the entire trajectory of the numerical solutions. Therefore, when used as an MCMC algorithm to sample from a log-concave target distribution $p(x)\propto e^{-f(x)}$, we can bound the 2-Wasserstein error between $p$ and the distribution of the QUICSORT samples.  \medbreak

\noindent
Whilst similar bounds have been established for second order methods (such as for the UBU scheme by \citet{sanz2021wasserstein}), we believe that QUICSORT is the only ULD numerical method with third order convergence guarantees. \citet{foster2021shifted} obtained similar third order non-asymptotic bounds for the shifted ODE, but these bounds require the ODE to be solved exactly. Although other third order numerical methods for ULD have been proposed, they involve computing the Hessian $\nabla^2 f$ which is often challenging for high-dimensional applications. We refer the reader to \citet{milstein2003third, shkerin12025third} for examples of such derivative-based third order methods.
\medbreak

\noindent
Finally, to support our theory, we presented numerical experiments where we demonstrated the third order convergence of QUICSORT and showed that it can outperform the widely-used No U-Turn Sampler (NUTS) as an MCMC algorithm for Bayesian logistic regression. However, in our MCMC experiments, we found that QUICSORT and UBU performed near identically across all 14 datasets. We expect this is simply due to the large step sizes used.\medbreak

\noindent
As well as considering more sophisticated Bayesian models, we would also like to explore the applications of Langevin-based MCMC algorithms with adaptive step sizes. In particular, \cite{jelincic2024vbt} conduct additional Bayesian logistic regression experiments using the QUICSORT method with an adaptive step size controller. This was possible thanks to the ``Virtual Brownian Tree'' proposed by the authors and available in Diffrax \citep{kidger2021ndes}. \medbreak

\noindent
These adaptive step sizes are determined using Proportional Integral (PI) controllers, which are commonly used for numerical ODE simulation but have seen relatively little use for SDEs \citep{burrage2004variable, ilie2015adaptive}. Whilst QUICSORT with adaptive step sizes usually performed comparably to QUICSORT with constant step sizes, there were several datasets (including the two with the highest dimensions) where there was a significant improvement.
These results are shown in Figure \ref{fig:adaptive_mcmc_convergence} and clearly demonstrates that PI step size controllers (which are very well supported by Diffrax) show great promise for Langevin-based MCMC.\medbreak

\noindent
In \citet{chada2024unbiased}, it is shown that the UBU method can be applied to perform unbiased MCMC estimation (even with inexact gradients) using Multilevel Monte Carlo techniques. Thus, a natural topic of future work would be to apply their methodology with QUICSORT.
As a further topic, it may also be possible to apply their approach with adaptive step sizes.
\medbreak

\noindent
Another direction of future research would be to apply the QUICSORT method to the Kinetic Interacting Particle Langevin Diffusion proposed by \citet{valsecchi2024kiplmc} for the joint estimation of latent variables and parameters in latent variable models.

\begin{figure}[H]
    \centering
    \begin{subfigure}{.33\textwidth}
      \centering
      \caption{flare solar ($d=10$)}
      \includegraphics[width=\linewidth]{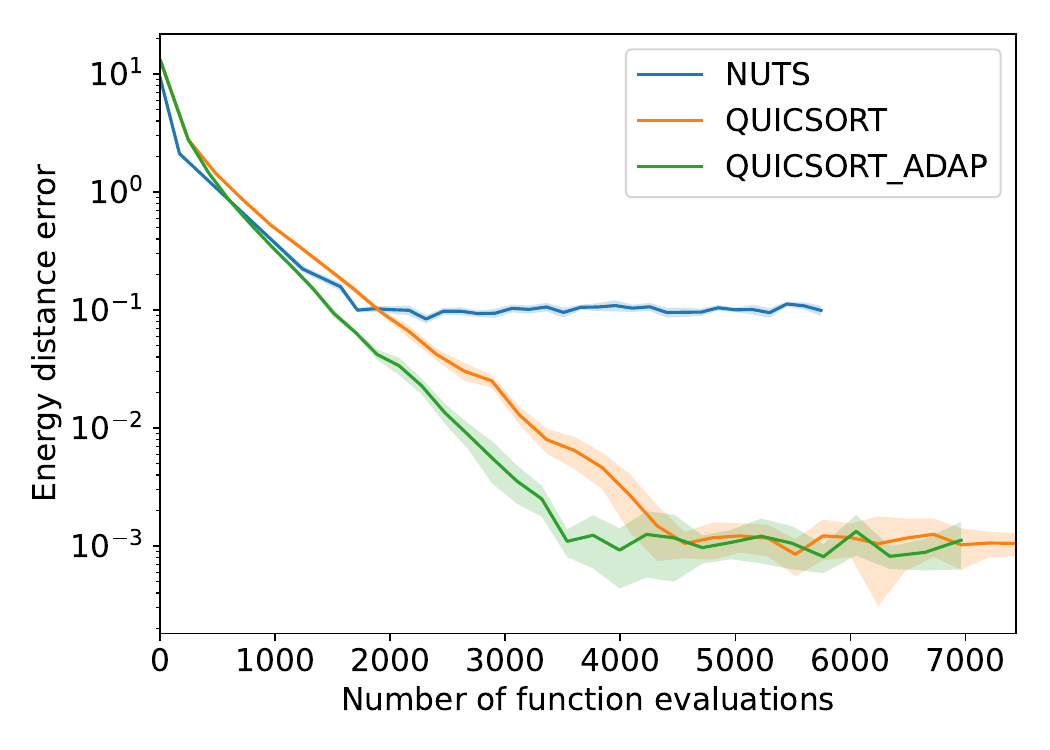}
    \end{subfigure}%
    \begin{subfigure}{.33\textwidth}
      \centering
      \caption{splice ($d=61$)}
      \includegraphics[width=\linewidth]{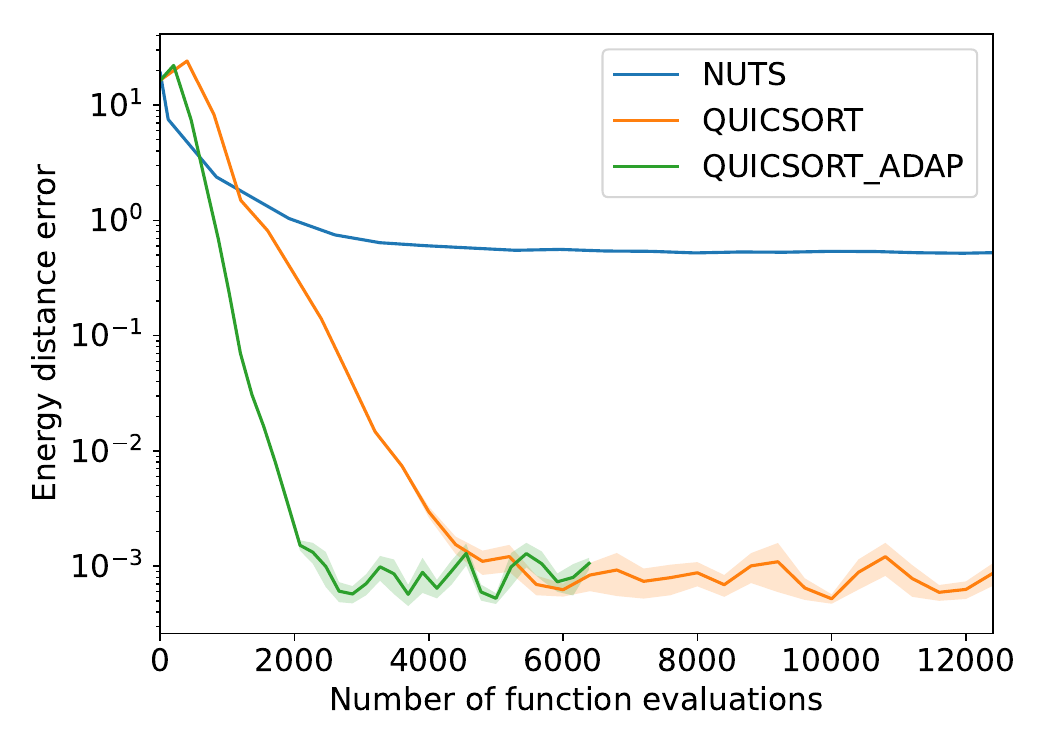}
    \end{subfigure}%
    \begin{subfigure}{.33\textwidth}
      \centering
      \caption{isolet ($d=617$)}
      \includegraphics[width=\linewidth]{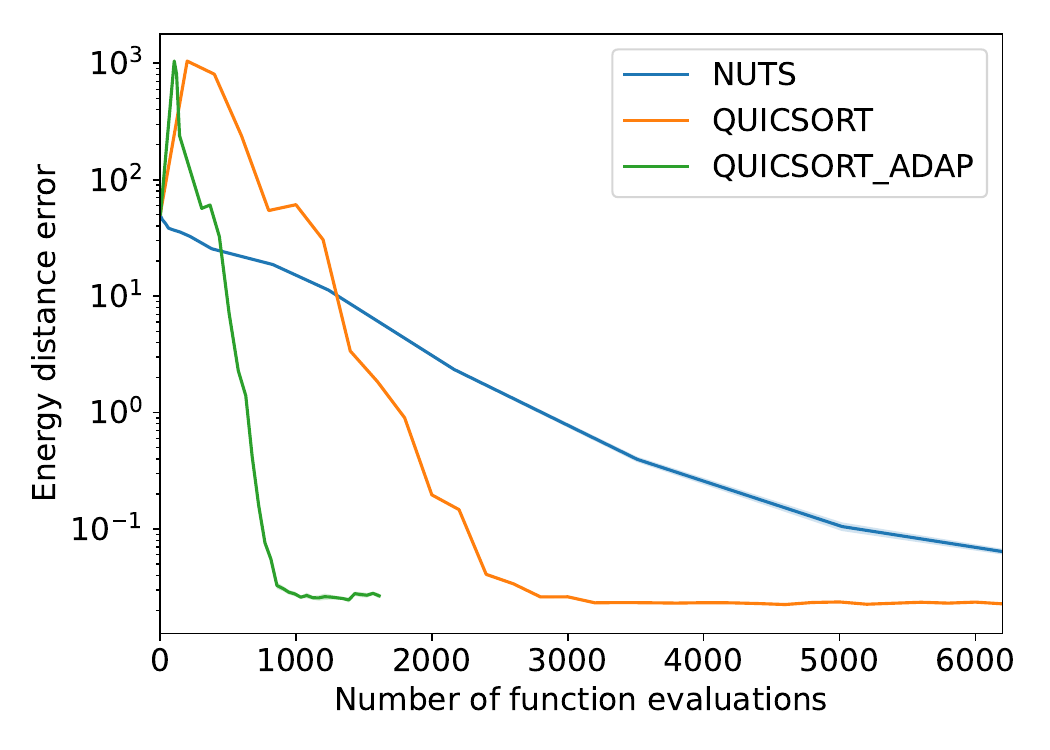}
    \end{subfigure}
    \caption{Using the same the target distribution and experimental setup as before, the QUICSORT method with constant and adaptive step sizes is now compared against NUTS. 
These adaptive step sizes are determined using a PI controller with parameters $K_P = 0.1$ and $K_I = 0.4$ (along with a user-specified tolerance and initial step size). We note that these are close to the parameters $(K_P\m, K_I) = (0.1, 0.3)$ recommended by \citet{ilie2015adaptive}. Similar to \citet{jelincic2024vbt}, we report the convergence results in the energy distance for three of the datasets (including the two highest dimensional ones), however the code and results for all 14 datasets can be found at \textcolor{blue}{\href{https://github.com/andyElking/Single-seed_BrownianMotion}{github.com/andyElking/Single-seed{\_}BrownianMotion}}.}
    \label{fig:adaptive_mcmc_convergence}
\end{figure}

Finally, another avenue of future research would be to improve the QUICSORT method by using ``exponential'' iterated integrals of the Brownian motion instead of $(W_n\m, H_n\m, K_n)$. This could then allow us to better match the following stochastic Taylor expansion of ULD:
\begin{align*}
x_{t_{n+1}}  & = x_{t_n} + \bigg(\frac{1 - e^{-\gamma (t_{n+1}-t_n)}}{\gamma}\bigg)v_{t_n} + \int_{t_n}^{t_{n+1}}\int_s^{r_1} e^{-\gamma(r_1-r_2)}\nabla f(x_{r_2})\, dr_2\,dr_1 \\
&\hspace{10.3mm}  + \sqrt{2\gamma u}\int_{t_n}^{t_{n+1}}\int_s^{r_1} e^{-\gamma (r_1-r_2)}\,dW_{r_2}\,dr_1\m,\\
v_{t_{n+1}}  & = e^{-\gamma (t_{n+1}-t_n)}v_{t_n} + \int_{t_n}^{t_{n+1}} e^{-\gamma(t_{n+1}-r_1)}\nabla f(x_{r_1})\, dr_1 + \sqrt{2\gamma u}\int_{t_n}^{t_{n+1}} e^{-\gamma (t_{n+1}-r_1)}\,dW_{r_1}\m,
\end{align*}
In particular, we could construct a piecewise linear path $\widetilde{W}$ on the interval $[s,t]$ such that
\begin{align*}
	\int_s^t e^{-\gamma(t - r_1)} \, d\widetilde{W}_{r_1} & =\int_s^t e^{-\gamma(t - r_1)} \, dW_{r_1}\m, \\[3pt]
	\int_s^t  \int_s^{r_1} e^{-\gamma(r_1 - r_2)} \, d\widetilde{W}_{r_2} \,  dr_1 & = \int_s^t  \int_s^{r_1} e^{-\gamma(r_1 - r_2)} \, dW_{r_2} \, dr_1\m, \\[3pt]
    \int_s^t e^{-\gamma(t-r_1)} \hspace{-0.5mm}\int_s^{r_1}\hspace{-0.5mm} \int_s^{r_2} e^{-\gamma(r_2 - r_3)} \, d\widetilde{W}_{r_3}\,  dr_2\,  dr_1 & = \int_s^t e^{-\gamma(t-r_1)}\hspace{-0.5mm}\int_s^{r_1}\hspace{-0.5mm} \int_s^{r_2} e^{-\gamma(r_2 - r_3)} \, dW_{r_3} \, dr_2 \, dr_1\m.
\end{align*}
If the path $\{(r, \widetilde{W}_r)\}$ had increments $(0,A),(t-s,B),(0,C)$, then the first integral becomes
\begin{align*}
\int_s^t e^{-\gamma(t - r)} \, d\widetilde{W}_r = \underbrace{e^{-\gamma(t - s)} A}_{\text{vertical jump at } r\m =\m s} + \underbrace{ \frac{1- e^{-\gamma (t-s)}}{\gamma (t-s)} B \,}_{\text{linear piece}} + \underbrace{C}_{\text{jump at } r\m =\m t}.
\end{align*}

\begin{figure}[h]
    \centering
    \includegraphics[width=0.7\linewidth]{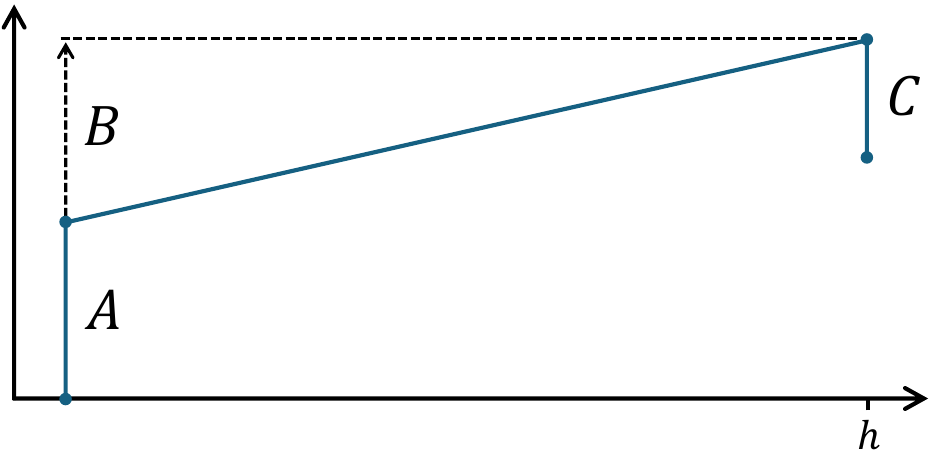}
    \caption{Illustration of the piecewise linear path $\widetilde{W}$ with increments $A$, $B$ and $C$ to match ``exponential'' iterated integrals of Brownian motion. Here, $h = t-s$ denotes the step size.}
    \label{fig:better_quicsort}
\end{figure}

Repeating for the other iterated integrals and substituting $h=t-s$, the problem of finding the increments $A, B$ and $C$ would reduce to solving the linear system $M\mathbf{x}=\mathbf{b}$ where 
\begin{align*}
M\hspace{-0.5mm}=\hspace{-0.75mm}
\begin{pmatrix}
e^{-\gamma h} & \dfrac{1 - e^{-\gamma h} }{\gamma h} & 1 \\[5pt]
\frac{1-e^{-\gamma h}}{\gamma} & \frac{e^{-\gamma h} + \gamma h - 1}{\gamma^2 h} & 0 \\[5pt]
\frac{1-e^{-\gamma h}(1+\gamma h)}{\gamma^2} & \frac{e^{-\gamma h}(2+\gamma h) + \gamma h - 2}{\gamma^3 h} & 0
\end{pmatrix},\hspace{2mm} \mathbf{b}\hspace{-0.5mm} = \hspace{-0.75mm}\begin{pmatrix}
    \int_s^t e^{-\gamma(t - r_1)} \, dW_{r_1} \\[3pt]
	 \int_s^t  \int_s^{r_1} e^{-\gamma(r_1 - r_2)} \, dW_{r_2} \, dr_1\\[3pt]
     \int_s^t e^{-\gamma(t-r_1)}\hspace{-0.5mm}\int_s^{r_1}\hspace{-0.5mm} \int_s^{r_2} e^{-\gamma(r_2 - r_3)} \, dW_{r_3} \, dr_2 \, dr_1
\end{pmatrix}\hspace{-0.5mm}
\end{align*}
and $\mathbf{x} = (A, B, C)^{\top}$. Since they are all linear functions of $W$, the iterated integrals in $\mathbf{b}$ are jointly Gaussian, but their covariance matrix would need to be computed to generate them.
Using these increments it is then straightforward to modify the QUICSORT method with $(H_n + 6K_n)$ replaced by $A$, $(W_n - 12K_n)$ replaced by $B$ and $(H_n - 6K_n)$ replaced by $C$.


\acks{We greatly appreciate the support from the Department of Mathematical Sciences at the University of Bath, the Maths4DL programme under EPSRC grant EP/V026259/1 and the Alan Turing Institute. AJ thanks the UK National Cyber Security Centre for funding his research and Extropic for funding his work implementing high order SDE solvers in Diffrax.}


\bibliography{references}

\newpage


\appendix
\section{Preliminary theorems}
\label{sec:appendix_A}
In this section, we detail several useful theorems that are crucial to proving our main results.
\begin{theorem}[Global $\L_p$ bounds of $\nabla f (x_t)$]
\label{thm:global_Lp_bounds_nabla_f}
	Let $\{ (x_t, v_t) \}_{t\m \geq\m 0}$ denote the underdamped Langevin diffusion \eqref{eq:ULD} with initial condition $(x_0, v_0) \sim \pi$, the unique stationary distribution of the diffusion process. Under the assumption that the gradient of $f$ is $M_1$-Lipschitz continuous then
	\begin{align}
		\label{eq:nabla_f_bounds_theorem}
		\big\| \nabla f(x_t) \big\|_{\L_{2p}} \leq \big[ (2p - 1)!! \big]^{\frac{1}{2p}} \sqrt{M_1 \m d},
	\end{align}
	for all $p \geq 1$. Here, we use $\m !!$ to denote the standard double factorial for odd numbers.
\end{theorem}
\begin{proof}
	Our strategy is to prove \eqref{eq:nabla_f_bounds_theorem} by induction for $d=1$, and then extend it for $d\geq 1$. Recall that ULD has the stationary distribution
	\begin{align*}
		\pi(x,v) \propto \exp\bigg(\hspace{-1mm} - f(x) - \frac{1}{2u}\| v \|^2_2 \bigg).
	\end{align*}
	Therefore, if $(x_0, v_0) \sim \pi$ then $(x_t, v_t) \sim \pi$ for all $t\geq 0$.\medbreak
	
	\noindent
	For $d=1$, we first consider the case where $p = 1$. Define
	\begin{align*}
		\widehat{\pi} (x) = -  f^\prime(x) \pi(x),
	\end{align*}
	and notice that $\widehat{\pi} (x) = \frac{d}{dx}\pi(x)$. Then
	\begin{align*}
			\E\big[ ( f^\prime(x))^2 \big] &= \int_{-\infty}^{\infty} ( f^\prime(x) )^2 \pi(x)\m dx = \int_{-\infty}^\infty f^\prime(x) \bigg[ f^\prime(0) + \int^x_0 f^{\prime\prime}(y) dy \bigg] \pi(x) \, dx \\[3pt]
							&= -f^\prime(0) \int_{-\infty}^\infty \widehat{\pi} (x) \, dx - \int_{-\infty}^\infty \int^x_0 f^{\prime\prime}(y) \widehat{\pi}(x) \, dy \, dx,
	\end{align*}
	by the fundamental theorem of calculus. For $\pi$ to be well defined as a probability measure we require $\pi(x) \to 0$ as $x \to \pm \infty$. Hence $\int_{-\infty}^\infty \widehat{\pi} (x) dx = \lim_{x\rightarrow\infty}\big(\pi(x) - \pi(-x)\big) = 0$ and
	\begin{align*}
			\E\big[ ( f^\prime(x))^2 \big] &= - \m\int_{-\infty}^\infty \int^x_0 f^{\prime\prime}(y) \widehat{\pi}(x) \, dy \, dx \\[3pt]
					&= - \int_0^\infty \int^x_0 f^{\prime\prime}(y) \widehat{\pi}(x) \, dy \, dx + \int_{-\infty}^0 \int_x^0 f^{\prime\prime}(y) \widehat{\pi}(x) \, dy \, dx \\[3pt]
					&= - \int_0^\infty \int_y^\infty f^{\prime\prime}(y) \widehat{\pi}(x) \, dx \, dy + \int_{-\infty}^0 \int_{-\infty}^y f^{\prime\prime}(y) \widehat{\pi}(x) \, dx \, dy \\[3pt]
					&= \int^{\infty}_0 f^{\prime\prime}(y) \pi(y) \, dy + \int_{-\infty}^0 f^{\prime\prime}(y) \pi(y) \, dy = \int_{-\infty}^\infty f^{\prime\prime}(x) \pi(x) \, dx,
	\end{align*}
	by Fubini's theorem and the properties of $\pi$, which in turn implies
	\begin{align*}
		\E \left[ ( f'(x))^2 \right] \leq M_1\m,
	\end{align*}
	as the derivative $f^\prime$ being $M_1$-Lipschitz implies that $f^{\prime\prime}$ is bounded by $M_1$ almost everywhere.\medbreak
	
	\noindent	
	For the inductive step, we will assume that for $p=n$
	\begin{align*}
		\E\Big[ \big( f^\prime(x) \big)^{2n} \Big] = \int_{-\infty}^\infty \big( f^\prime(x) \big)^{2n} \pi (x) \, dx \leq (2n - 1)!! \, M_1^n,
	\end{align*}
	and for $p=n+1$, we define
	\begin{align*}
		\widetilde{\pi} (x) = \big( f^\prime(x) \big)^{2n} \pi(x).
	\end{align*}
	Notice that
	\begin{align*}
		\frac{d}{dx} \widetilde{\pi} (x) = 2n f^{\prime\prime}(x) \big( f^\prime(x) \big)^{2n - 1} \pi(x) - \big( f^\prime(x) \big)^{2n+1} \pi(x),
	\end{align*}
	so rearranging gives
	\begin{align*}
		f^\prime(x) \widetilde{\pi} (x) =  2n f^{\prime\prime}(x) \big( f^\prime(x) \big)^{2n - 1} \pi(x) - \widetilde{\pi}^{\,\prime}(x).
	\end{align*}
	Therefore, from the definition of $\widetilde{\pi}$ and the above identity, we have
	\begin{align*}
		&\E \Big[ \big( f^\prime(x) \big)^{2(n+1)} \Big]\\
		& \quad = \int_{-\infty}^\infty f^\prime(x) \big[ f^\prime(x) \widetilde{\pi} (x) \big] \, dx \mfill \mm \\[3pt]
		& \mm =  \m \int_{-\infty}^\infty f^\prime(x) \Big[ 2n f^{\prime\prime}(x) \left( f^{\prime}(x) \right)^{2n - 1} \pi(x) - \widetilde{\pi}^{\,\prime}(x) \Big] \, dx \\[3pt]
		& \mm = 2n  \int_{-\infty}^\infty f^{\prime\prime}(x) \big(f^\prime(x)\big)^{2n} \pi(x) \, dx -  \int_{-\infty}^\infty f^\prime(0) \widetilde{\pi} \, ' (x) \, dx -  \m\int_{-\infty}^\infty \int^x_0 f^{\prime\prime}(y) \widetilde{\pi}^{\,\prime}(x) \, dy \, dx ,
	\end{align*}
	by the fundamental theorem of calculus.\medbreak
	
	\noindent
	Since $\int_{-\infty}^\infty \widetilde{\pi} (x) \, dx$ is bounded (by the induction hypothesis) and $\widetilde{\pi}$ is non-negative, it follows that $\widetilde{\pi} (x) \to 0$ as $x \to \pm \infty$. Therefore $\int_{-\infty}^\infty \widetilde{\pi}^{\,\prime} (x) \, dx = \lim_{x\rightarrow\infty}\big(\widetilde{\pi}(x) - \widetilde{\pi}(-x)\big) = 0$ and
	\begin{align*}
		&\E \Big[ \big( f^\prime(x) \big)^{2(n+1)} \Big]\\
		&\mm =  2n \int_{-\infty}^\infty f^{\prime\prime}(x) \big(f^\prime(x)\big)^{2n} \pi(x) \, dx - \int^\infty_0 \int^x_0 f^{\prime\prime}(y) \widetilde{\pi}^{\,\prime}(x) \, dy \, dx  + \int^0_{-\infty} \int_{x}^0 f^{\prime\prime}(y) \widetilde{\pi}^{\,\prime}(x) \, dy \, dx  \\[3pt]
		&\mm =  2n \int_{-\infty}^\infty f^{\prime\prime}(x) \big(f^\prime(x)\big)^{2n} \pi(x) \, dx - \int^\infty_0 \int_y^\infty f^{\prime\prime}(y) \widetilde{\pi}^{\,\prime}(x) \, dx \, dy  + \int^0_{-\infty} \int_{-\infty}^y f^{\prime\prime}(y) \widetilde{\pi}^{\,\prime}(x) \, dx \, dy  \\[3pt]
		&\mm = 2n \int_{-\infty}^\infty f^{\prime\prime}(x) \big(f^\prime(x)\big)^{2n} \pi(x) \, dx + \int_{-\infty}^\infty f^{\prime\prime}(x) \widetilde{\pi} (x) \, dx ,
	\end{align*}
	by Fubini's theorem and the property that $\widetilde{\pi} (x) \to 0$ as $x \to \pm \infty$.\medbreak

	\noindent	
	 Hence, by the definition of $\widetilde{\pi}$ and the fact that $|f^{\prime\prime}(x)|\leq M_1$ almost everywhere, we have
	\begin{align*}
		\int_{-\infty}^\infty \big( f^{\prime}(x) \big)^{2(n+1)} \pi (x) dx = (2n+1) \int_{-\infty}^\infty f^{\prime\prime}(x) \big( f^{\prime}(x) \big)^{2n} \pi (x) dx \leq (2n+1)!! \, M_1^{n+1}.
	\end{align*}
	Since this bound is true for $p=1$, it is also true for all $p \geq 1$ by induction.\medbreak

	\noindent	
	For a general dimension $d \geq 1$, we can apply this bound to each partial derivative of $f$, as each marginal distribution of $\pi$ is of the form $\exp(-g(x_i))$ where $g$ is $M_1$-Lipschitz. Then
	\begin{align*}
		\E \Big[ \big\| \nabla f(x) \big\|^{2p}_2 \Big] = \E \Bigg[ \bigg( \sum^d_{i=1} \bigg\vert \frac{\partial f}{\partial x_i} \bigg\vert^2\m \bigg)^p\m \Bigg] \leq d^{\m p-1} \E \Bigg[ \sum^d_{i=1} \bigg\vert \frac{\partial f}{\partial x_i} \bigg\vert^{2p} \Bigg],
	\end{align*}
	by the triangle inequality followed by Jensen's inequality \citep[Theorem A.1]{foster2021shifted}. Therefore
	\begin{align*}
		\E \Big[ \big\| \nabla f(x) \big\|^{2p}_2 \Big] \leq d^{p-1} \sum^d_{i=1} \E\Bigg[\bigg\vert \frac{\partial f}{\partial x_i}\bigg\vert^{2p} \Bigg] \leq (2p - 1)!! \, (M_1 \m d)^p,
	\end{align*}
	as required.
\end{proof}

\begin{remark}
	A direct implication of Theorem \ref{thm:global_Lp_bounds_nabla_f} are the following bounds:
	\begin{align}
	\label{eq:nabla_f_bounds}
		\big\| \nabla f(x_t) \big\|_{\L_{2}} &\leq \sqrt{M_1 \m d}\m, \nonumber\\[3pt]
		\big\| \nabla f(x_t) \big\|_{\L_{4}} &\leq \sqrt[4]{3} \sqrt{M_1 \m d}\m, \\[3pt]
		\big\| \nabla f(x_t) \big\|_{\L_{6}} &\leq \sqrt[6]{15} \sqrt{M_1 \m d}\m.\nonumber
	\end{align}
\end{remark}

\begin{theorem}
	Let $Z = ( Z^{(1)}, \dots, Z^{(d)}) \in \mathbb{R}^d$ denote a centered $d$-dimensional Gaussian multivariate random variable such that $\m\mathrm{Var}(Z) = \sigma^2\m$ for some $\sigma \in \mathbb{R}$. Then\label{thm:general_multivariate_bound}
	\begin{align}
	\label{eq:general_multivariate_bound}
		\| Z \|_{\L_{2p}} \leq [ (2p - 1)!!]^{\frac{1}{2p}} \sqrt{\sigma^2 d}\m,
	\end{align}
    for all $p\geq 1$.
\end{theorem}
\begin{proof}
	First consider the case when $d=1$. Define $X \sim \mathcal{N} (0, \sigma^2)$. Then, it is a standard result that the Moment Generating Function (MGF) of $X$ is given by $M_X(t) = e^{\frac{1}{2} t^2 \sigma^2}$ \citep[Chapter 7]{bulmer1979principles}. We can find the $(2p)$-th moment of $X$ using
	\begin{align*}
			\E\big[ X^{2p} \big] & = \frac{d^{2p}}{dt^{2p}} M_X(0) = \frac{d^{2p}}{dt^{2p}} \bigg[ 1 + \frac{1}{2} t^2 \sigma^2 + \cdots + \frac{1}{n!}\Big(\frac{1}{2}t^2\sigma^2\Big)^n + \cdots \bigg] = \frac{(2p)!}{2^p p!} \sigma^{2p} \\[3pt]
							& = (2p - 1)!! \, \sigma^{2p}.
	\end{align*}
	Then by Jensen's inequality \citep[Theorem A.1]{foster2021shifted}, we have
	\begin{align*}
	\E\Bigg[\bigg(\sum_{i=1}^d \big(Z^{(i)}\big)^2\bigg)^p\,\Bigg] \leq d^{p-1}\sum_{i=1}^d \E\Big[\big(Z^{(i)}\big)^{2p}\Big] \leq d^{p-1}\sum_{i=1}^d(2p - 1)!! \, \sigma^{2p} = (2p - 1)!! \, \sigma^{2p} d^{\m p}.
	\end{align*}
	The result now directly follows since the left hand side is precisely equal to $\|Z\|_{\L_{2p}}^{2p}$.
\end{proof}

\begin{corollary}[Global $\L_p$ bounds of $v_t$]
	Let $\{ (x_t, v_t) \}_{t \geq 0}$ be the underdamped Langevin diffusion with initial condition $(x_0, v_0) \sim \pi$, the stationary distribution of the process. Then\label{thm:global_Lp_bounds_v}
	\begin{align}
		\label{eq:momentum_bounds_theorem}
		\| v_t\|_{\L_{2p}} \leq [(2p - 1)!!]^{\frac{1}{2p}} \sqrt{u d}\m,
	\end{align}
	for all $p \geq 1$.
\end{corollary}
\begin{proof}
	The stationary distribution of ULD \eqref{eq:ULD} can be written explicitly as 
	\begin{align*}
		\pi(x,v) \propto \exp\bigg(\hspace{-1mm} - f(x) - \frac{1}{2u}\| v \|^2_2 \bigg).
	\end{align*}
	Thus, $v_t$ has a centred $d$-dimensional multivariate Gaussian distribution with $\mathrm{Var} ( v^{(i)}_t ) = u$. The result now immediately follows by Theorem \ref{thm:general_multivariate_bound}. 
\end{proof}

\begin{remark}
	A direct implication of the Theorem \ref{thm:global_Lp_bounds_v} are the following bounds:
	\begin{align}
	\label{eq:v_bounds}
		\| v_t \|_{\L_{2}} &\leq \sqrt{u d}, \nonumber\\[3pt]
		\| v_t \|_{\L_{4}} &\leq \sqrt[4]{3} \sqrt{u d}, \\[3pt]
		\| v_t \|_{\L_{6}} &\leq \sqrt[6]{15} \sqrt{u d}.\nonumber
	\end{align}
\end{remark}

\begin{theorem}
	Let $\{ t_n \}_{n \geq 0}$ be a sequence of times with $t_0 = 0$, $t_{n+1} > t_n$ and $h_n = t_{n+1} - t_n$. Similar to the random variables given by \eqref{eq:W_H_K_definition}, we define
	\begin{align}
	\label{eq:M_definition}
		M_n = \frac{1}{h^3_n} \int^{t_{n+1}}_{t_n} \bigg(\frac{1}{10} h^2_n - \frac{1}{2}(t-t_n)h_n + \frac{1}{2} (t - t_n)^2 \bigg) \bigg( W_{t_n, t} - \frac{t - t_n}{h_n} W_n \bigg) \, dt .
	\end{align}
	Then $M_n \sim \mathcal{N} \left( 0, \frac{1}{100800} h_n I_d \right)$ is also independent of $W_n$, $H_n$ and $K_n\m$. Furthermore,
	\begin{align}
	\label{eq:levy_space_time_3}
		\int^{t_{n+1}}_{t_n} \int^{r_1}_{t_n} \int^{r_2}_{t_n} W_{t_n, r_3} \, dr_3 \, dr_2 \, dr_1 = \frac{1}{24} h_n^3 W_n+ \frac{3}{20} h_n^3 H_n + \frac{1}{2} h_n^3 K_n + h_n^3 M_n\m.
	\end{align}
\end{theorem}
\begin{proof}
	One can derive \eqref{eq:M_definition} as a coefficient in a polynomial expansion of Brownian motion (see \citet[Theorem 2.2]{foster2020optimal}). It is then straightforward to prove the identity \eqref{eq:levy_space_time_3} using integration by parts and Theorem \ref{thm:levy_space_time_1_2}.
\end{proof}

\begin{corollary}
	Let $W_n$, $H_n$, $K_n$ and $M_n$ be the centred Gaussian random variables defined in equations \eqref{eq:W_H_K_definition} and \eqref{eq:M_definition}. Then, for all $p\geq 1$, we have the bounds
	\begin{align}
	\label{eq:levy_area_bounds}
	\begin{split}
			\| W_n \|_{\L_{2p}} & \leq [ (2p - 1)!! ]^{\frac{1}{2p}} \sqrt{dh_n}\,, \quad \quad \quad \; \; \, \| H_n \|_{\L_{2p}} \leq [ (2p - 1)!! ]^{\frac{1}{2p}} \sqrt{\frac{1}{12} dh_n}\,, \\[3pt]
			\| K_n \|_{\L_{2p}} & \leq [ (2p - 1)!! ]^{\frac{1}{2p}} \sqrt{\frac{1}{720} dh_n}\,, \quad \quad \| M_n \|_{\L_{2p}} \leq [ (2p - 1)!! ]^{\frac{1}{2p}} \sqrt{\frac{1}{100800} dh_n}\,.
	\end{split}
	\end{align} 
\end{corollary}
\begin{proof}
	The result follows immediately from Theorem \ref{thm:general_multivariate_bound}.
\end{proof}

\noindent
To help simplify the notation (especially for long calculations), we make the below definition.
\begin{definition}
For all time intervals $[t_n, t_{n+1}]$, we define the following functions,
\begin{align}
	\label{eq:phi_functions}
	\phi_0 (x) = e^{-x\gamma h_n}, \quad \quad \phi_1 (x) = \frac{1 - e^{-x\gamma h_n}}{\gamma}\m, \quad \quad \phi_2 (x) = \frac{e^{-x\gamma h_n} + x\gamma h_n - 1}{\gamma^2}\m.
\end{align}
\end{definition}

\begin{theorem}
\label{thm:phi_function_bounds}
	Consider a combination of the $\phi_i$ functions given by equation \eqref{eq:phi_functions},
	\begin{align*}
		p(h) = \sum^n_{i=1} \prod^m_{j=1} \phi_{a_j} (b_{i, j}),
	\end{align*}
	where $a \in \{ 0, \, 1, \, 2 \}^m$ and $b \in \{ \lambda_+\m, \, \lambda_-\m, \, \frac{1}{3} \}^{n \, \times \, m}$. Then
	\begin{align}
	\label{eq:phi_functions_bound}
		\vert p(h) \vert &\leq h^{\sum^m_{j=1} a_j} \, 2^{-\sum^m_{j=1} \mathds{1}_{\{2\}} (a_j)} \sum^n_{i=1} \prod^m_{j=1} (b_{i, j})^{a_j}.
	\end{align}
\end{theorem}
\begin{proof}
	Without loss of generality, we write
	\begin{align*}
		\phi_0 (x, h) = e^{-x\gamma h}, \quad \quad \phi_1 (x, h) = \frac{1 - e^{-x\gamma h}}{\gamma}\m, \quad \quad \phi_2 (x, h) = \frac{e^{-x\gamma h} + x\gamma h - 1}{\gamma^2}\m,
	\end{align*}
	and note that
	\begin{align*}
		 \phi_1 (x, h) = \int^{xh}_0 \phi_0 (1, t) \, dt, \quad \quad \phi_2 (x,h) = \int^{xh}_0 \int^{r_1}_0 \phi_0 (1, r_2) \, dr_2 \, dr_1\m.
	\end{align*}
	Since $\vert \phi_0 (x, h) \vert \leq 1$ for all $x, h \geq 0$, then
	\begin{align*}
			\vert \phi_1 (x, h) \vert &= \bigg\vert  \int^{xh}_0 \phi_0 (1, t) \, dt \bigg\vert \leq \int^{xh}_0 \vert   \phi_0 (1, t) \vert \, dt \leq \int^{xh}_0 1 \, dt = xh, \\[3pt]
			\vert \phi_2 (x, h) \vert &= \bigg\vert  \int^{xh}_0 \int^{r_1}_0 \phi_0 (1, r_2) \, dr_2 \, dr_1 \bigg\vert \leq \int^{xh}_0 \int^{r_1}_0 \vert\phi_0 (1, r_2) \vert \, dr_2 \, dr_1 \leq \frac{1}{2}x^2h^2.
	\end{align*}
	Therefore, by the triangle inequality
	\begin{align*}
	\Bigg\vert \sum^n_{i=1} \prod^m_{j=1} \phi_{a_j} (b_{i, j}) \Bigg\vert \leq \sum^n_{i=1} \Bigg\vert \prod^m_{j=1} \phi_{a_j} (b_{i, j}) \Bigg\vert = \sum^n_{i=1} \prod^m_{j=1} \big\vert \phi_{a_j} (b_{i, j}) \big\vert.
	\end{align*}
	The result now follows after taking the inequalities of the $\phi_i$ functions and rearranging. 
\end{proof}

\noindent
To end this section, we present a few useful results for estimating terms in Taylor expansions.

\begin{theorem}
	Under the assumptions \eqref{eq:assumption_a1} and \eqref{eq:assumption_a2}, we have
	\begin{align}
	\label{eq:n2_f_x_v}
		m \|v\|_{\L_p} \leq \big\| \nabla^2 f(x) v \big\|_{\L_p} \leq M_1 \|v\|_{\L_p}\m,
	\end{align}
	for all $x,v \in \mathbb{R}^d$ and $p \geq 1$. Furthermore, for all $x,v_1, v_2 \in \mathbb{R}^d$ and $p \geq 1$,
	\begin{align}
	\label{eq:n2_f_x_v1_v2}
		\big\| \nabla^2 f(x) v_1 - \nabla^2 f(x) v_2 \big\|_{\L_p} \leq M_1 \| v_1 - v_2 \|_{\L_p}\m.
	\end{align}
	Under the additional assumption \eqref{eq:assumption_a3}
	\begin{align}
	\label{eq:n2_f_x1_x2_v}
		\big\| \nabla^2 f(x_1) v - \nabla^2 f(x_2) v \big\|_{\L_p} \leq M_2 \| x_1 - x_2 \|_{\L_{2p}} \|v\|_{\L_{2p}}\m,
	\end{align}
	for all $x_1, x_2, v \in \mathbb{R}^d$ and $p \geq 1$. In addition, we have
	\begin{align}
	\label{eq:n3_f_x_v1_v2}
		\big\|\nabla^3 f(x) (v_1, v_2)\big\|_{\L_p} \leq M_2 \| v_1 \|_{\L_{2p}} \| v_2 \|_{\L_{2p}}\m,
	\end{align}
	and
	\begin{align}
	\label{eq:n3_f_x_v1_v2_v3}
		\big\| \nabla^3 f(x) (v_1, v_2) - \nabla^3 f(x) (v_1, v_3)\big\|_{\L_p} \leq M_2 \| v_1 \|_{\L_{2p}} \| v_2 - v_3\|_{\L_{2p}}\m,
	\end{align}
	for all $x, v_1, v_2, v_3 \in \mathbb{R}^d$ and $p \geq 1$. Finally, under the additional assumption \eqref{eq:assumption_a4}
	\begin{align}
	\label{eq:n3_f_x1_x2_v1_v2}
		\big\| \nabla^3 f(x_1) (v_1, v_2) - \nabla^3 f(x_2) (v_1, v_2) \big\|_{\L_p} \leq M_3  \| x_1 - x_2 \|_{\L_{3p}} \| v_1 \|_{\L_{3p}} \| v_2 \|_{\L_{3p}}\m,
	\end{align}
	for all $x_1, x_2, v_1, v_2 \in \mathbb{R}^d$ and $p \geq 1$.
\end{theorem}
\begin{proof}
	These inequalities follow from the definitions of $\L_p$ norm and Lipschitz continuity, along with an application of Hölder's inequality.  
\end{proof}

\begin{theorem}\label{thm:cauchy_schartz_young_inequality}
Let $c > 0$ be a fixed constant. Then for vectors $x,y \in \mathbb{R}^d$, we have
\begin{align}
\label{eq:cauchy_schwarz_young_inequality}
		\langle x, y \rangle \leq \frac{c}{2} \| x \|^2_2 + \frac{1}{2c} \| y \|^2_2\,.
	\end{align}
\end{theorem}
\begin{proof}
	Using the Cauchy-Schwarz and Young inequalities, we obtain
	\begin{align*}
		 \langle x, y \rangle = \Big\langle \sqrt{c} \m x, \frac{1}{\sqrt{c}}\m y \Big\rangle \leq \big\| \sqrt{c} \m x \big\|_2 \bigg\| \frac{1}{\sqrt{c}}\m y \bigg\|_2 \leq \frac{1}{2}\bigg(\big\| \sqrt{c} \m x \big\|_2^2 + \bigg\| \frac{1}{\sqrt{c}}\m y \bigg\|_2^2\,\bigg) = \frac{1}{2}\Big(c\|x\|_2^2 + \frac{1}{c}\|y\|_2^2\Big),
	\end{align*}
	as required.
\end{proof}


\section{Exponential contractivity of the QUICSORT method}
\label{sec:appendix_B}

In this section, we will prove Theorem \ref{thm:main_contractivity_theorem}, which gives the contractivity of the QUICSORT method. Our first step will be to write the QUICSORT method \eqref{eq:QUICSORT} as the sum of two terms,
\begin{align}
	\label{eq:approximation_contraction_split}
	X_{n+1} = \widetilde{X}_{n+1} + \overline{X}_n\m, \quad \quad V_{n+1} = \widetilde{V}_{n+1} + \overline{V}_n\m,
\end{align}
where
\begin{align*}
		\widetilde{V}^{\m (1)}_n  &:= V_n + \sigma (H_n + 6K_n), \\[3pt]
		\widetilde{X}^{\m (1)}_n  &:= X_n + \frac{1 - e^{-\lambda_- \gamma h_n}}{\gamma} \m \widetilde{V}^{\m (1)}_n + \frac{e^{-\lambda_- \gamma h_n} + \lambda_- \gamma h_n - 1}{\gamma^2} \m C_n\m, \\[1pt]
		\widetilde{V}^{\m (2)}_n &:= e^{-\lambda_+ \gamma h_n} \widetilde{\m V}^{(1)}_n - \frac{1}{2} e^{- \frac{1}{3} \gamma h_n} u\nabla f\big( \widetilde{X}^{\m (1)}_n \big) h_n + \frac{1 - e^{-\lambda_+ \gamma h_n}}{\gamma}\m C_n\m, \\[3pt]
		\widetilde{X}^{\m (2)}_n  &:= X_n + \frac{1 - e^{-\lambda_+ \gamma h_n}}{\gamma}\m \widetilde{V}^{\m (1)}_n - \frac{1 - e^{-\frac{1}{3} \gamma h_n}}{\gamma} u \nabla f (\widetilde{X}^{\m (1)}_n )h_n + \frac{e^{-\lambda_+ \gamma h_n} + \lambda_+ \gamma h_n - 1}{\gamma^2}\m C_n\m, \\[1pt]
		\widetilde{V}^{\m (3)}_n &:= e^{-\lambda_- \gamma h_n} \widetilde{V}^{\m (2)}_n- \frac{1}{2} e^{-\lambda_- \gamma h_n}u \nabla f (\widetilde{X}^{\m (2)}_n ) h_n + \frac{1 - e^{-\lambda_- \gamma h_n}}{\gamma}\m C_n\m, \\[3pt]
		\widetilde{X}_{n+1} &:= \widetilde{X}^{\m (2)}_n - \frac{1 - e^{- \lambda_- \gamma h_n}}{2\gamma} u \nabla f (\widetilde{X}^{\m (2)}_n )h_n + \frac{1 - e^{-\lambda_- \gamma h_n}}{\gamma}\m \widetilde{V}^{\m (2)}_n + \frac{e^{-\lambda_- \gamma h_n} + \lambda_- \gamma h - 1}{\gamma^2}\m C_n\m, \\[3pt]
		\widetilde{V}_{n+1} &:= \widetilde{V}^{\m (3)}_n + \sigma(H_n - 6K_n),
\end{align*}
and
\begin{align*}
		\overline{X}_n &:= \frac{1}{2\gamma} \Big(1 - e^{-\frac{1}{3} \gamma h_n} - e^{- \big( \frac{1}{3} + \lambda_- \big) \gamma h_n} + e^{-\lambda_+ \gamma h_n} \Big) u \nabla f (\widetilde{X}^{\m (1)}_n )h_n\m, \\[3pt]
		\overline{V}_n &:= \frac{1}{2} \Big( e^{- \big( \frac{1}{3} + \lambda_- \big) \gamma h_n} - e^{- \lambda_+ \gamma h_n} \Big) u \nabla f (\widetilde{X}^{\m (1)}_n )h_n\m .
\end{align*}
The procedure for showing that the QUICSORT approximation is contractive is as follows. Since proving contractivity for a map with multiple gradient evaluations can be difficult, we will split the entire step $(X_n\m, V_n) \to (X_{n+1}\m, V_{n+1})$ into the following three smaller mappings:
\begin{align*}
( X_n\m, V_n ) \to \big(\widetilde{X}^{\m (2)}_n\m, \widetilde{V}^{\m (2)}_n \big),\,\,\,\,\, \big( \widetilde{X}^{\m (2)}_n\m, \widetilde{V}^{\m (2)}_n \big) \to \big( \widetilde{X}_{n+1}\m, \widetilde{V}_{n+1} \big),\,\,\,\,\, \big( \widetilde{X}_{n+1}\m, \widetilde{V}_{n+1} \big) \to (X_{n+1}\m, V_{n+1}).
\end{align*}
The first two maps are designed to resemble the UBU and Exponential Euler schemes \citep{sanz2021wasserstein} and so, by similar techniques, can be shown to be contractive. The final map is simply adding the correction term $\big(\m\overline{X}_n, \overline{V}_n \big)$, which we will show is small. 

\begin{theorem}
\label{thm:contractive_mapping_1}
	Consider the two approximations $(X, V)$ and $(Y, U)$ under the mappings $( X_n\m, V_n ) \to \big(\widetilde{X}^{(2)}_n\m, \widetilde{V}^{(2)}_n \big)$ and $( Y_n\m, U_n ) \to \big(\widetilde{Y}^{(2)}_n\m, \widetilde{U}^{(2)}_n \big)$ driven by the same Brownian motion but evolved from initial conditions $(X_n\m, V_n )$ and $(Y_n\m, U_n)$. We define the transformations,
	\begin{align}
		\label{eq:coordinate_transformation_mapping_1}
		\begin{split}
			\begin{pmatrix} W_n \\[3pt] Z_n \end{pmatrix} & := \begin{pmatrix} \big( w X_n + V_n \big) - \big( w Y_n + U_n \big) \\[3pt]  \big( z X_n + V_n \big) - \big( z Y_n + U_n \big) \end{pmatrix}, \\[5pt]
			\begin{pmatrix} W^{\m (1)}_n \\[3pt] Z^{\m (1)}_n \end{pmatrix} & := \begin{pmatrix} \big( w X^{\m (2)}_n + V^{\m (2)}_n \big) - \big( w Y^{\m (2)}_n + U^{\m (2)}_n \big) \\[3pt]  \big( z X^{\m (2)}_n + V^{\m (2)}_n \big) - \big( z Y^{\m (2)}_n + U^{\m (2)}_n \big) \end{pmatrix},
		\end{split}
	\end{align}
	where $w \in \big[0, \frac{1}{2} \gamma \big)$ and $z = \gamma - w$. Then, under assumptions \eqref{eq:assumption_a1} and \eqref{eq:assumption_a2}, we have
	\begin{align}
		\label{eq:coordinate_transformation_bound_1}
		\big( \| W^{(1)}_n \|^2_{\L_2} +  \| Z^{(1)}_n \|^2_{\L_2} \big) \leq (1 - \alpha_1 h_n) \big( \| W_n \|^2_{\L_2} +  \| Z_n \|^2_{\L_2} \big),
	\end{align}
	for all $0 < h_n \leq 1$ and $n \geq 0$ where
	\begin{align}
		\alpha_1 =  \frac{(2 \lambda_+ z^2 - u M_1 ) \wedge (u m - 2 \lambda_+ w^2 ) }{z - w} - K_1 \m h_n\m,
	\end{align}
	for some constant $K_1 > 0$ not depending on $h_n\m$.
\end{theorem}
\begin{proof}
	We can unfold the mappings and then simplify notation using the $\phi_i$ functions \eqref{eq:phi_functions}. The noise terms cancel out due to the synchronous coupling, which first yields
	\begin{align*}
			\widetilde{V}^{\m (1)}_n - \widetilde{U}^{\m (1)}_n &= V_n - U_n\m, \\[5pt]
			\widetilde{X}^{\m (1)}_n - \widetilde{Y}^{\m (1)}_n &= X_n - Y_n + \frac{1 - e^{-\lambda_- \gamma h_n}}{\gamma}\m( V_n - U_n ) \\[3pt]
										&= X_n - Y_n + \phi_1 (\lambda_-) ( V_n - U_n ).
	\end{align*}
	Since we assume that $f$ is twice differentiable, we can use Taylor's theorem to define\footnote{The following definition is still valid even if we do not assume that the potential $f$ is twice differentiable. By Rademacher's theorem \citep[Theorem 3.1.6]{federer1969geometric} since $\nabla f$ is Lipschitz continuous, then $\nabla f$ is almost everywhere differentiable. Therefore, as the probability of hitting a null set is 0, we can infer that the potential function's second derivative is sufficiently well defined to be applied in our setting.}
	\begin{align*}
			\nabla f\big( \widetilde{X}^{\m (1)}_n \big) - \nabla f\big( \widetilde{Y}^{\m (1)}_n \big) &= \underbrace{\int^1_0 \nabla^2 f \big( \widetilde{X}^{\m (1)}_n + r ( \widetilde{Y}^{\m (1)}_n - \widetilde{X}^{\m (1)}_n ) \big)  dr}_{=:\, \mathcal{H}^{\m (1)}_n} \, \big( \widetilde{X}^{\m (1)}_n - \widetilde{Y}^{\m (1)}_n \big) \\[3pt]
							&= \mathcal{H}^{\m (1)}_n \big[ ( X_n - Y_n ) + \phi_1 (\lambda_-) ( V_n - U_n ) \big], 
	\end{align*}
	where $\mathcal{H}^{\m (1)}_n \in \mathbb{R}^{d \times d}$ is a symmetric positive semidefinite matrix which, from the assumptions \eqref{eq:assumption_a1} and \eqref{eq:assumption_a2}, additionally satisfies $m I_d \preceq \mathcal{H}^{\m (1)}_n \preceq M_1 I_d$.\medbreak
    \noindent
    Then the next stage in the mappings yields following differences,
	\begin{align*}
			&\widetilde{V}^{(2)}_n - \widetilde{U}^{(2)}_n\\[-3pt]
			 &\mm = e^{-\lambda_+ \gamma h_n} ( V_n - U_n ) - \frac{1}{2} e^{-\frac{1}{3} \gamma h_n} u h_n \Big( \nabla f \big( \widetilde{X}^{(1)}_n \big) - \nabla f \big( \widetilde{Y}^{(1)}_n \big) \Big) \\[3pt]
						&\mm = \bigg[ - \frac{1}{2} u h_n \phi_0 \bigg(\frac{1}{3} \bigg) \mathcal{H}^{\m (1)}_n \bigg] ( X_n - Y_n ) + \bigg[ \phi_0(\lambda_+) I_d - \frac{1}{2} u h_n \phi_0 \big(\frac{1}{3} \big) \phi_1 (\lambda_-) \mathcal{H}^{\m (1)}_n \bigg] ( V_n - U_n ), \\[9pt]
			&\widetilde{X}^{(2)}_n - \widetilde{Y}^{(2)}_n\\[-3pt]
			 &\mm = (X_n - Y_n ) + \frac{1 - e^{-\lambda_+ \gamma h_n}}{\gamma}\m( V_n - U_n ) - \frac{1 - e^{-\frac{1}{3} \gamma h_n}}{\gamma} u h_n \Big( \nabla f \big( \widetilde{X}^{\m (1)}_n \big) - \nabla f \big( \widetilde{Y}^{\m (1)}_n \big) \Big) \\[3pt]
											&\mm = \bigg[ I_d - u h_n \phi_1 \bigg(\frac{1}{3}\bigg) \mathcal{H}^{\m (1)}_n \bigg] (X_n - Y_n ) + \bigg[ \phi_1 (\lambda_+) I_d - u h_n \phi_1 \bigg(\frac{1}{3}\bigg) \phi_1 (\lambda_-) \mathcal{H}^{\m (1)}_n \bigg] ( V_n - U_n ).
	\end{align*}
	From the definitions of $W_n$ and $Z_n\m$, we can write
	\begin{align*}
		X_n - Y_n = \frac{Z_n - W_n}{z - w}\m, \quad \quad V_n - U_n = \frac{z W_n - w Z_n}{z - w}\m.
	\end{align*}
	Thus, $W^{(1)}_n$ can be expressed as
	\begin{align*}
			W^{(1)}_n &= w \big( \widetilde{X}^{\m (2)}_n - \widetilde{Y}^{\m (2)}_n \big) + \big( \widetilde{V}^{\m (2)}_n - \widetilde{U}^{\m (2)}_n \big) \\[3pt]
					&= \bigg( w \bigg[ I_d - u h_n \phi_1 \bigg(\frac{1}{3}\bigg) \mathcal{H}^{\m (1)}_n \bigg]  + \bigg[ - \frac{1}{2} u h_n \phi_0 \bigg(\frac{1}{3}\bigg) \mathcal{H}^{\m (1)}_n \bigg] \bigg) (X_n - Y_n) \\
					&\mmm + w \left[ \phi_1 (\lambda_+) I_d - u h_n \phi_1 \left(\frac{1}{3}\right) \phi_1 (\lambda_-) \mathcal{H}^{\m (1)}_n \right] \left( V_n - U_n \right) \\
					&\mmm + \bigg[ \phi_0(\lambda_+) I_d - \frac{1}{2} u h_n \phi_0 \bigg(\frac{1}{3} \bigg) \phi_1 (\lambda_-) \mathcal{H}^{\m (1)}_n \bigg] ( V_n - U_n ) \\[4pt]
					&= \frac{1}{z - w} \bigg( \bigg( w \bigg[ I_d - u h_n \phi_1 \bigg(\frac{1}{3}\bigg) \mathcal{H}^{\m (1)}_n \bigg]  + \bigg[ - \frac{1}{2} u h_n \phi_0 \bigg(\frac{1}{3} \bigg) \mathcal{H}^{\m (1)}_n \bigg] \bigg) (Z_n - W_n)  \\
					&\hspace{20mm} + w \bigg[ \phi_1 (\lambda_+) I_d - u h_n \phi_1 \bigg(\frac{1}{3}\bigg) \phi_1 (\lambda_-) \mathcal{H}^{\m (1)}_n \bigg] (z W_n - w Z_n) \\
					&\hspace{20mm}  + \bigg[ \phi_0(\lambda_+) I_d - \frac{1}{2} u h_n \phi_0 \big(\frac{1}{3} \big) \phi_1 (\lambda_-) \mathcal{H}^{\m (1)}_n \bigg] (z W_n - w Z_n) \bigg) \\[3pt]
					&= \frac{1}{z - w} \Big[ \big( -w A - B + w z C + z D \big) W_n + \big( w A + B - w^2 C - w D \big) Z_n \Big],
	\end{align*}
	and similarly
    \vspace{-0.1mm}
	\begin{align*}
		Z^{(1)}_n = \frac{1}{z - w} \Big[ \big( -z A - B + z^2 C + z D \big) W_n + \big( z A + B - w z C - w D \big) Z_n \Big],
	\end{align*}
    \vspace{-0.6mm}
	where
	\begin{align*}
			A &=  I_d - u h_n \phi_1 \bigg(\frac{1}{3}\bigg) \mathcal{H}^{\m (1)}_n, \quad \quad \quad \quad \quad \quad \quad \; \; \; \; B =  - \frac{1}{2} u h_n \phi_0 \bigg(\frac{1}{3} \bigg) \mathcal{H}^{\m (1)}_n\m, \\
			C &=  \phi_1 (\lambda_+)I_d - u h_n \phi_1 \bigg(\frac{1}{3}\bigg) \phi_1 (\lambda_-) \mathcal{H}^{\m (1)}_n, \quad \quad D = \phi_0(\lambda_+)I_d - \frac{1}{2} u h_n \phi_0 \bigg(\frac{1}{3} \bigg) \phi_1 (\lambda_-) \mathcal{H}^{\m (1)}_n\m.
	\end{align*}
	Note the matrices $A,B,C, D$ are not only symmetric but also commute with one another. By expanding $\big\|W^{(1)}_n\big\|_2^2$ and $\big\| Z^{(1)}_n \big\|_2^2\m$, we obtain
	\begin{align}
	\big\| W^{(1)}_n \big\|^2_2 + \big\| Z^{(1)}_n \big\|^2_2 = \frac{1}{(z - w)^2} \big(W_n^{\top} E W_n + Z_n^{\top} F Z_n + 2 W_n^{\top} G  Z_n\big) ,
	\end{align}
	where
	\begin{align*}
	\begin{split}
		E &=  (w^2 + z^2) A^2 + 2 B^2 + z^2 (w^2 + z^2) C^2 + 2z^2 D^2 + 2 \gamma AB - 2z(w^2 + z^2) AC \\
				&\hspace{7.5mm}  - 2\gamma z AD - 2\gamma z BC - 4z BD +2 \gamma z^2 CD,  \\[3pt]
		F &=  (w^2+z^2) A^2 + 2B^2 + w^2 (w^2 + z^2) C^2 + 2 w^2 D^2 + 2\gamma AB - 2  w (w^2 + z^2) AC \\
				&\hspace{7.5mm} - 2 w \gamma AD -  2\gamma w BC - 4w BD + 2\gamma w^2 CD, \\[3pt]
		G &= -(w^2 + z^2)A^2 - 2 B^2 - z w (w^2 + z^2) C^2 - 2z w D^2 - 2 \gamma AB + \gamma (w^2 + z^2) AC  \\
				&\hspace{7.5mm} + \gamma^2 AD + \gamma^2 BC + 2 \gamma BD - 2 \gamma w z CD.
	\end{split}
	\end{align*}
	Whilst we could bound the terms involving $A$, $B$, $C$ and $D$ by applying the inequality \eqref{eq:phi_functions_bound}, since $-\phi_i (x) \leq \vert \phi_i (x) \vert$ for all $x \in \big\{\lambda_+\m, \lambda_-\m, \frac{1}{3}\big\}$ and $i \in \{0, 1\}$ this would produce too high an upper bound in order to prove the desired result. Instead, we will use the bounds below,
	\begin{align*}
		-\phi_0(x) \leq -1 + x \gamma h_n\m, \quad\quad -\phi_1(x) \leq -xh_n + \frac{1}{2} \gamma x^2 h^2_n\m.
	\end{align*}
	In the following, we will also use
	\begin{align*}
		- h^c_n \sum^m_{i=1} \phi_{a_i} (b_i) \leq 0,
	\end{align*}
	for terms with $c + \sum^{m}_{i=1} a_i \geq 2$ as they do not significantly impact the contraction rate $\alpha_1$ and all the $\phi$ functions are bounded below by $0$. Using the Loewner order $\preceq$ and the commutativity of the matrices $A, B, C$ and $D$, we can bound the matrices $E,F$ and $G$ as
\begin{align*}
		E  &\preceq (w^2 + z^2) \bigg( I_d + \frac{1}{9} u^2 h^4_n \big( \mathcal{H}^{\m (1)}_n \big)^2 \bigg)\\
		&\mmm + 2 \bigg( \frac{1}{4} u^2 h^2_n \big( \mathcal{H}^{\m (1)}_n \big)^2 \bigg) + z^2 (w^2 + z^2) \bigg(  \lambda^2_+ h^2_n I_d + \frac{2}{9} \lambda^2_- u^2 h^6_n \big( \mathcal{H}^{\m (1)}_n \big)^2 \bigg)\\
					&\hspace{7.5mm}  + 2z^2 \bigg( I_d - 2\lambda_+ \gamma h_n I_d+ 2\lambda^2_+ \gamma^2 h^2_n I_d + \frac{1}{4} \lambda^2_-  u^2_n h^4_n  \big( \mathcal{H}^{\m (1)}_n \big)^2 \bigg) \\
					&\hspace{7.5mm} + 2 \gamma \bigg(  - \frac{1}{2} u h_n \mathcal{H}^{\m (1)}_n  + \frac{1}{6} \gamma u h^2_n \mathcal{H}^{\m (1)}_n + \frac{1}{6} u^2 h^3_n \big( \mathcal{H}^{\m (1)}_n \big)^2 \bigg)\\
					&\hspace{7.5mm} + 2z(w^2 + z^2) \bigg(  \lambda_+ h_n I_d+ \frac{1}{2} \gamma \lambda^2_+ h^2_n I_d + \frac{1}{3}  u h^3_n \mathcal{H}^{\m (1)}_n \bigg) \\
					&\hspace{7.5mm} + 2 \gamma z \bigg( I_d - \lambda_+ \gamma h_n I_d + \frac{1}{2} \lambda^2_+ \gamma^2 h^2_n I_d+ \frac{1}{6} \lambda_- u^2 h^4_n \big( \mathcal{H}^{\m (1)}_n \big)^2 \bigg) \\
				&\hspace{7.5mm} + 2\gamma z \bigg(  \frac{1}{2} \lambda_+ u h^2_n \mathcal{H}^{\m (1)}_n \bigg) + 4z \bigg(  \frac{1}{2} u h_n \mathcal{H}^{\m (1)}_n \bigg) +2 \gamma z^2 \bigg( \lambda_+ h_n I_d + \frac{1}{6} \lambda^2_- u^2 h^5_n \big( \mathcal{H}^{\m (1)}_n \big)^2 \bigg) \\[3pt]
				&\preceq \Big( ( z - w )^2 + ( z - w ) h_n \big( u M_1  - 2 \lambda_+ z^2 \big) + h^2_n \Big( \gamma z + 3 \gamma^2 z^2 + z^2 w^2 + z^4  \big) \\
				&\hspace{7.5mm} + h^2_n \big[( z + \gamma w + 2 \gamma z) uM_1 + \big( 1 + \gamma + z^2 + w^2 + z^2 w^2 + z^4 \big)u^2 M^2_1\m \big] \Big) I_d \\[3pt]
				&= \Big( ( z - w )^2 + ( z - w ) h_n \big(u M_1 - 2 \lambda_+ z^2 \big) + E_+ \m h^2_n\m \Big) I_d\m, \\[5pt]
		F  &\preceq (w^2 + z^2) \bigg( I_d + \frac{1}{9} u^2 h^4_n \big( \mathcal{H}^{\m (1)}_n \big)^2 \bigg) + 2 \bigg( \frac{1}{4} u^2 h^2_n \big( \mathcal{H}^{\m (1)}_n \big)^2 \bigg)\\
		&\hspace{7.5mm}+ w^2 (w^2 + z^2) \bigg(  \lambda^2_+ h^2_n I_d+ \frac{2}{9} \lambda^2_- u^2 h^6_n \big( \mathcal{H}^{\m (1)}_n \big)^2 \bigg)\\
		&\hspace{7.5mm} + 2w^2 \bigg( I_d - 2\lambda_+ \gamma h_n I_d + 2\lambda^2_+ \gamma^2 h^2_n I_d + \frac{1}{4} \lambda^2_-  u^2_n h^4_n  \big( \mathcal{H}^{\m (1)}_n \big)^2 \bigg) \\
					&\hspace{7.5mm} + 2 \gamma \bigg(  - \frac{1}{2} u h_n \mathcal{H}^{\m (1)}_n  + \frac{1}{6} \gamma u h^2_n \mathcal{H}^{\m (1)}_n + \frac{1}{6} u^2 h^3_n \big( \mathcal{H}^{\m (1)}_n \big)^2 \bigg)  \\
					&\hspace{7.5mm} + 2w(w^2 + z^2) \bigg(  \lambda_+ h_n I_d+ \frac{1}{2} \gamma \lambda^2_+ h^2_n I_d+ \frac{1}{3}  u h^3_n \mathcal{H}^{\m (1)}_n \bigg)\\
					&\hspace{7.5mm} + 2 \gamma w \bigg( I_d - \lambda_+ \gamma h_n I_d+ \frac{1}{2} \lambda^2_+ \gamma^2 h^2_n I_d + \frac{1}{6} \lambda_- u^2 h^4_n \big( \mathcal{H}^{\m (1)}_n \big)^2 \bigg) \\
				&\hspace{7.5mm}+ 2\gamma w \bigg(  \frac{1}{2} \lambda_+ u h^2_n \mathcal{H}^{\m (1)}_n \bigg) + 4w \bigg(  \frac{1}{2} u h_n \mathcal{H}^{\m (1)}_n \bigg) +2 \gamma w^2 \bigg( \lambda_+ h_n I_d + \frac{1}{6} \lambda^2_- u^2 h^5_n \big( \mathcal{H}^{\m (1)}_n \big)^2 \bigg) \\[3pt]
				&\preceq \Big( ( z - w )^2 + ( z - w ) h_n \big( 2 \lambda_+ w^2 - u m \big) + h^2_n \big( \gamma w + 3 \gamma^2 w^2 + z^2 w^2 + w^4\big) \\
				&\hspace{7.5mm} + h_n^2 \big[\big( w + \gamma z + 2 \gamma w \big) u M_1 + \big( 1 + \gamma + z^2 + w^2 + z^2 w^2 + w^4 \big) u^2 M^2_1\m \big] \Big) I_d \\[3pt]
				&= \Big( ( z - w )^2 + ( z - w ) h_n \big( 2 \lambda_+ w^2 - u m \big) + F_+ \m h^2_n\m \Big) I_d\m, \\[5pt]
    	G & \preceq  (w^2 + z^2) \bigg( - I_d + \frac{2}{3} u h^2_n \mathcal{H}^{\m (1)}_n \bigg) \\
		&\hspace{7.5mm} + z w (w^2 + z^2) \bigg( \lambda^2_+ h^2_n I_d - \gamma \lambda^3_+ h^3_n I_d + \frac{1}{4} \gamma^2 \lambda^4_- h^4_n I_d + \frac{2}{3} \lambda_+ \lambda_- u h^4_n  \mathcal{H}^{\m (1)}_n \bigg)\\
		&\hspace{7.5mm} +  2z w \Big( -I_d + 2\lambda_+ \gamma h_n I_d + \lambda_- u h^2_n \mathcal{H}^{\m (1)}_n \Big) \\
		&\hspace{7.5mm}  + \gamma u h_n \mathcal{H}^{\m (1)}_n + \gamma (w^2 + z^2) \bigg( \lambda_+ h_n I_d+ \frac{1}{9} \lambda_- u^2 h^5_n \big( \mathcal{H}^{\m (1)}_n \big)^2 \bigg) \\
		&\hspace{7.5mm} + \gamma^2 \bigg( I_d - \lambda_+ \gamma h_n I_d + \frac{1}{2} \lambda^2_+ \gamma^2 h^2_n I_d + \frac{1}{6} \lambda_- u^2 h^4_n \big( \mathcal{H}^{\m (1)}_n \big)^2 \bigg) \\
		&\hspace{7.5mm} + \gamma^2 \bigg(  \frac{1}{6} \lambda_- u^2 h^4_n  \big( \mathcal{H}^{\m (1)}_n \big)^2 \bigg) + \gamma \bigg( - u h_n  \mathcal{H}^{\m (1)}_n + \bigg( \frac{1}{3} + \lambda_+ \bigg) \gamma u h^2_n  \mathcal{H}^{\m (1)}_n \bigg) \\
		&\hspace{7.5mm} + 2 \gamma w z \bigg( - \lambda_+ h_n I_d + 2 \gamma \lambda^2_+ h^2_n I_d + \frac{1}{2} \lambda_+ \lambda_- u h^3_n \mathcal{H}^{\m (1)}_n + \frac{1}{3} \lambda_- u h^3_n \mathcal{H}^{\m (1)}_n \bigg)\\
		&\hspace{7.5mm} + \frac{1}{2} \lambda_- \gamma u^2 h^3_n \big( \mathcal{H}^{\m (1)}_n \big)^2 \\[3pt]
		&\preceq h^2_n \Big[ z^3 w + z w^3 + 3 \gamma^2 z w + \gamma^2 z^3 w + \gamma^2 z w^3  +\gamma^4 + \big( z w + z^2 + w^2 + 2 \gamma^2\big) uM_1  \\
		&\mmm + \big(\gamma z w + z^3 w + z w^3 \big) u M_1 + \big( 1 + \gamma + \gamma^2 + \gamma z^2 + \gamma w^2 \big) u^2 M^2_1 \Big] I_d \\[3pt]
		&= G_+ h^2_n I_d\m,
\end{align*}
where we have used the assumption that $h_n \leq 1$ and also bounded $\lambda_+$ and $\lambda_-$ for readability. Next, we will establish an $O(h_n^2)$ upper bound for $W_n^{\top} G Z_n$ in terms of $\| W_n \|^2_{\L_2}$ and $ \| Z_n \|^2_{\L_2}\m$.
\begin{align*}
	4\m W_n^{\top} G Z_n & = (W_n + Z_n)^\top G(W_n + Z_n) -  (W_n - Z_n)^\top G(W_n - Z_n) \\
    & \leq \big|(W_n + Z_n)^\top G(W_n + Z_n)\big| + \big|(W_n - Z_n)^\top G(W_n - Z_n)\big|\\
    & \leq G_+ h^2_n \big\| W_n + Z_n \big\|_2^2 + G_+ h^2_n \big\| W_n - Z_n \big\|_2^2 \,,
\end{align*}
which, by the triangle and Young's inequalities, implies that
\begin{align*}
	W_n^{\top} G Z_n  & \leq G_+ h^2_n \Big(  \big\| W_n \big\|_2^2 +  \big\| Z_n \big\|_2^2 \Big).
\end{align*}
Therefore, putting this all together gives
\begin{align*}
		& \big\| W^{\m (1)}_n \big\|^2_2 + \big\| Z^{\m (1)}_n \big\|^2_2\\
		&\mm = \frac{1}{(z - w)^2} \Big( W_n^{\top} E W_n + Z_n^{\top} F Z_n + 2\m W_n^{\top} G  Z_n\Big) \\[3pt]
									&\mm\leq  \| W_n \|^2_2 + \| Z_n \|^2_2 +  \frac{h_n}{z - w} \Big( ( u M_1 - 2 \lambda_+ z^2 )  \| W_n \|^2_2  + ( 2 \lambda_+ w^2 - u m ) \| Z_n \|^2_2 \Big) \\
									&\hspace{14.5mm} + \frac{h^2_n}{(z - w)^2} \Big( E_+  \| W_n \|^2_2 + F_+  \| Z_n \|^2_2  +  2\m G_+ \big( \| W_n \|^2_2 +  \| Z_n \|^2_2 \big) \Big) \\[3pt]
									&\mm\leq \bigg(1 +  \frac{h_n}{z - w} \max \big( ( u M_1 - 2 \lambda_+ z^2 ),  ( 2 \lambda_+ w^2 - u m ) \big)  \\
									&\hspace{14mm}  + \frac{h^2_n}{(z - w)^2} \big( \max( E_+, F_+ )  +  2\m G_+ \big) \bigg) \big( \| W_n \|^2_2 +  \| Z_n \|^2_2 \big) \\[3pt]
									&\mm = (1 - \alpha_1 \m h_n ) \big(\| W_n \|^2_2+ \| Z_n \|^2_2 \big).
\end{align*}
The result now follows by taking expectations of both sides.
\end{proof}
\noindent
Similarly, we can show that the second map $\big( \widetilde{X}^{\m (2)}_n\m, \widetilde{V}^{\m (2)}_n \big) \to \big( \widetilde{X}_{n+1}\m, \widetilde{V}_{n+1} \big)$ is also contractive.

\begin{theorem}
\label{thm:contractive_mapping_2}
		Consider the two approximations $(X, V)$ and $(Y, U)$ under the mappings $\big(\widetilde{X}^{\m (2)}_n\m, \widetilde{V}^{\m (2)}_n \big) \to \big( \widetilde{X}_{n+1}\m, \widetilde{V}_{n+1} \big)$ and $\big(\widetilde{Y}^{\m (2)}_n\m, \widetilde{U}^{\m (2)}_n \big) \to \big( \widetilde{Y}_{n+1}\m, \widetilde{U}_{n+1} \big)$ that are driven by the same Brownian motion but are evolved from the initial conditions $\big(\widetilde{X}^{\m (2)}_n\m, \widetilde{V}^{\m (2)}_n \big)$ and $\big(\widetilde{Y}^{\m (2)}_n\m, \widetilde{U}^{\m (2)}_n\big)$. Just as before, we define the following coordinate transformations,
	\begin{align}
		\label{eq:coordinate_transformation_mapping_2}
		\begin{split}
			\begin{pmatrix} W^{\m (1)}_n \\[3pt] Z^{\m (1)}_n \end{pmatrix} &= \begin{pmatrix} \big( w X^{\m (2)}_n + V^{\m (2)}_n \big) - \big( w Y^{\m (2)}_n + U^{\m (2)}_n \big) \\[3pt]  \big( z X^{\m (2)}_n + V^{\m (2)}_n \big) - \big( z Y^{\m (2)}_n + U^{\m (2)}_n \big) \end{pmatrix}, \\[5pt]
			\begin{pmatrix} W^{\m (2)}_n \\[3pt] Z^{\m (2)}_n \end{pmatrix} &= \begin{pmatrix} \big( w  \widetilde{X}_{n+1} +  \widetilde{V}_{n+1} \big) - \big( w  \widetilde{Y}_{n+1} +  \widetilde{U}_{n+1} \big) \\[3pt]  \big( z  \widetilde{X}_{n+1} +  \widetilde{V}_{n+1} \big) - \big( z  \widetilde{Y}_{n+1} +  \widetilde{U}_{n+1} \big) \end{pmatrix},
		 \end{split}
	\end{align}
	where $w \in \big[0, \frac{1}{2} \gamma \big)$ and $z = \gamma - w$. Then, under assumptions \eqref{eq:assumption_a1} and \eqref{eq:assumption_a2}, we have
	\begin{align}
		\label{eq:coordinate_transformation_bound_2}
		\big( \| W^{\m (2)}_n \|^2_{\L_2} +  \| Z^{\m (2)}_n \|^2_{\L_2} \big) \leq (1 - \alpha_2 \m h_n) \big( \| W^{\m (1)}_n \|^2_{\L_2} +  \| Z^{\m (1)}_n \|^2_{\L_2} \big),
	\end{align}
	 for all $0 < h_n \leq 1$ and $n \geq 0$ where
	\begin{align}
		\alpha_2 =  \frac{(2 \lambda_- z^2 - u M_1 ) \wedge (u m - 2 \lambda_- w^2 ) }{z - w} - K_2 \m h_n\m,
	\end{align}
	for some constant $K_2 > 0$. 
\end{theorem}
\begin{proof}
	The proof strategy is identical to that of Theorem \ref{thm:contractive_mapping_1}. The main difference is simply the values of the constants $E_+$, $F_+$, and $G_+$. Specifically, these constants are given by
	\begin{align*}
		E_+ &= \gamma^2 z^2 + \gamma z^3 + \gamma z w^2 + z^2 w^2 + z^4 + \big( \gamma + \gamma z + z w^2 + z^3 \big) uM_1 + ( 1 + \gamma ) u^2 M^2_1\m, \\[5pt]
		F_+ &= \gamma^2 w^2 + \gamma w^3 + \gamma z^2 w + z^2 w^2 + w^4 + \big( \gamma + \gamma w + z^2 w + w^3 \big) u M_1 + ( 1 + \gamma ) u^2 M^2_1\m,  \\[5pt]
		G_+ &=  \gamma^2 z w + \gamma^4 + ( \gamma^2 + z^2 + w^2 ) u M_1 \m.
	\end{align*}
\end{proof}
Next, we will show that the third mapping $\big( \widetilde{X}_{n+1}\m, \widetilde{V}_{n+1} \big) \to (X_{n+1}\m, V_{n+1})$ is always small.

\begin{theorem}
\label{thm:correction_mapping}
	Consider the two approximations $(X, V)$ and $(Y, U)$ under the mappings $\big( \widetilde{X}_{n+1}\m, \widetilde{V}_{n+1} \big) \to \big( X_{n+1}\m, V_{n+1} \big)$ and $\big( \widetilde{Y}_{n+1}\m, \widetilde{U}_{n+1} \big) \to \left( Y_{n+1}\m, U_{n+1} \right)$ that are driven by the same Brownian motion but evolved from initial conditions $\big( \widetilde{X}_{n+1}\m, \widetilde{V}_{n+1} \big)$ and $\big( \widetilde{Y}_{n+1}\m, \widetilde{U}_{n+1} \big)$.  Just as before, we define the following coordinate transformations,
	\begin{align}
		\label{eq:coordinate_transformation_mapping_3}
		\begin{split}
		 	\begin{pmatrix} W_{n} \\[3pt] Z_{n} \end{pmatrix} &= \begin{pmatrix} \big( w X_{n} + V_{n} \big) - \big( w Y_{n} + U_{n} \big) \\[3pt]  \big( z X_{n} + V_{n} \big) - \big( z Y_{n} + U_{n} \big) \end{pmatrix}, \\[5pt]
			\begin{pmatrix} W^{\m (3)}_n \\[3pt] Z^{\m (3)}_n \end{pmatrix} &= \begin{pmatrix} \big( w  \overline{X}_{n} +  \overline{V}_{n} \big) - \big( w  \overline{Y}_{n} +  \overline{U}_{n} \big) \\[3pt]  \big( z  \overline{X}_{n} +  \overline{V}_{n} \big) - \big( z  \overline{Y}_{n} +  \overline{U}_{n} \big) \end{pmatrix},
		 \end{split}
	\end{align}
	where $w \in \left[0, \frac{1}{2} \gamma \right)$ and $z = \gamma - w$. Then, under assumptions \eqref{eq:assumption_a1} and \eqref{eq:assumption_a2}, we have
	\begin{align}
		\label{eq:coordinate_transformation_bound_3}
		\big( \| W^{\m (3)}_n \|^2_{\L_2} +  \| Z^{\m (3)}_n \|^2_{\L_2} \big) \leq K_3 \, h^4_n \big( \| W_n \|^2_{\L_2} +  \| Z_n \|^2_{\L_2} \big),
	\end{align}
	 for all $0 < h_n \leq 1$ and $n \geq 0$ where $K_3 > 0$ is a constant not depending on $h_n\m$.
\end{theorem}

\begin{proof}
Again, the proof strategy is nearly identical to that of Theorem \ref{thm:contractive_mapping_1}. Here, we use
\begin{align}
	E \leq E_+ \m h^4_n\m, \quad \quad \text{and} \quad \quad F \leq F_+ \m h^4_n\m, \nonumber
\end{align}
where
\begin{align*}
	E_+ &\leq \left( 1 + \gamma + z^2 + w^2 + \gamma z^2 + z^2 w^2 + z^4 \right) u^2 M^2_1, \\[5pt]
	F_+ &\leq \left( 1 + \gamma + z^2 + w^2 + \gamma w^2 + z^2 w^2 + w^4 \right) u^2 M^2_1, 
\end{align*}
and
\begin{align*}
	G_+ &\leq  ( \gamma + \gamma^2 + \gamma z^2 + \gamma w^2 ) u^2 M^2_1  h^2_n \m.
\end{align*}
\end{proof}
Finally, we are now in a position to prove our main contractivity result (Theorem \ref{thm:main_contractivity_theorem}). 
\begin{theorem}[Contractivity of the QUICSORT method \eqref{def:QUICSORT}]
\label{thm:main_contraction}
	Let $(X, V)$ and $(Y, U)$ denote two numerical approximations defined by the mappings $( X_n\m, V_n ) \to \big( X_{n+1}\m, V_{n+1})$ and $( Y_n\m, Y_n ) \to ( Y_{n+1}\m, U_{n+1})$ which are driven by the same Brownian motion but evolved from initial conditions $( X_n\m, V_n )$ and $( Y_n\m, U_n)$. We define the coordinate transformations,
	\begin{align}
		\begin{split}
			\begin{pmatrix} W_{n} \\[3pt] Z_{n} \end{pmatrix} &= \begin{pmatrix} \big( w X_{n} + V_{n} \big) - \big( w Y_{n} + U_{n} \big) \\[3pt]  \big( z X_{n} + V_{n} \big) - \big( z Y_{n} + U_{n} \big) \end{pmatrix}, \\[5pt]
			\begin{pmatrix} W_{n+1} \\[3pt] Z_{n+1} \end{pmatrix} &= \begin{pmatrix} \big( w X_{n+1} + V_{n+1} \big) - \big( w Y_{n+1} + U_{n+1} \big) \\[3pt]  \big( z X_{n+1} + V_{n+1} \big) - \big( z Y_{n+1} + U_{n+1} \big) \end{pmatrix},
		 \end{split}
	\end{align}
	where $w \in \big[0, \frac{1}{2} \gamma \big)$ and $z = \gamma - w$. Then, under the assumptions \eqref{eq:assumption_a1} and \eqref{eq:assumption_a2}, we have
	\begin{equation}
		\big( \| W_{n+1} \|^2_{\L_2} +  \| Z_{n+1} \|^2_{\L_2} \big) \leq (1 - 2\alpha h_n) \big( \| W_n \|^2_{\L_2} +  \| Z_n \|^2_{\L_2}\big),
	\end{equation}
	for all $0 < h_n \leq 1$ and $n \geq 0$ where
	\begin{equation}
		\alpha = \frac{1}{z - w} \min \Big[ z^2 - uM_1, \, um - w^2, \, (\lambda_- z^2 - \lambda_+ w^2) - \frac{1}{2} u( M_1 - m) \Big] - K h_n\m,
	\end{equation}
	for some constant $K > 0$ not depending on $h_n\m$. 
\end{theorem}

\begin{proof}
	From the decomposition \eqref{eq:approximation_contraction_split} and coordinate transformations \eqref{eq:coordinate_transformation_mapping_2} and \eqref{eq:coordinate_transformation_mapping_3}, we have
	\begin{align*}
		\begin{pmatrix} W_{n+1} \\[3pt] Z_{n+1} \end{pmatrix} &= \begin{pmatrix} \big( w  \widetilde{X}_{n+1} +  \widetilde{V}_{n+1} \big) - \big( w  \widetilde{Y}_{n+1} +  \widetilde{U}_{n+1} \big) \\[3pt]  \big( z  \widetilde{X}_{n+1} +  \widetilde{V}_{n+1} \big) - \big( z  \widetilde{Y}_{n+1} +  \widetilde{U}_{n+1} \big) \end{pmatrix} + \begin{pmatrix} \big( w  \overline{X}_{n} +  \overline{V}_{n} \big) - \big( w  \overline{Y}_{n} +  \overline{U}_{n} \big) \\[3pt]  \big( z  \overline{X}_{n} +  \overline{V}_{n} \big) - \big( z  \overline{Y}_{n} +  \overline{U}_{n} \big) \end{pmatrix} \\[3pt]
					&= \begin{pmatrix} W^{\m (2)}_n \\[3pt] Z^{\m (2)}_n \end{pmatrix} + \begin{pmatrix} W^{\m (3)}_n \\[3pt] Z^{\m (3)}_n \end{pmatrix}. \nonumber
	\end{align*}
	By applying the bounds \eqref{eq:coordinate_transformation_bound_1}, \eqref{eq:coordinate_transformation_bound_2} and \eqref{eq:coordinate_transformation_bound_3} given by the previous theorems along with the inequality \eqref{eq:cauchy_schwarz_young_inequality} in Theorem \ref{thm:cauchy_schartz_young_inequality} with $c=h^2_n\m$, we obtain
	\begin{align*}
		&\big( \| W_{n+1} \|^2_{\L_2} +  \| Z_{n+1} \|^2_{\L_2} \big) \\[3pt]
		&\mm \leq \big( \| W^{\m (2)}_{n} \|^2_{\L_2} +  \| Z^{\m (2)}_{n} \|^2_{\L_2} \big) + \big( \| W^{\m (3)}_{n} \|^2_{\L_2} +  \| Z^{\m (3)}_{n} \|^2_{\L_2} \big) + 2\m\E\Big[\big\langle W_n^{\m (2)}, W_n^{\m (3)}\big\rangle + \big\langle Z_n^{\m (2)}, Z_n^{\m (3)}\big\rangle\Big]\\[3pt]
		&\mm \leq (1 + h^2_n) \big( \| W^{\m (2)}_{n} \|^2_{\L_2} +  \| Z^{\m (2)}_{n} \|^2_{\L_2} \big) + \bigg(1 + \frac{1}{h^2_n} \bigg) \Big( \| W^{\m (3)}_{n} \|^2_{\L_2} +  \| Z^{\m (3)}_{n} \|^2_{\L_2} \Big) \\[3pt]
		&\mm \leq  (1 + h^2_n) (1 - \alpha_2 h_n) \big( \| W^{\m (1)}_n \|^2_{\L_2} +  \| Z^{\m (1)}_n \|^2_{\L_2} \big) + (h^2_n + h^4_n) K_3 \big( \| W_n \|^2_{\L_2} +  \| Z_n \|^2_{\L_2} \big) \\[3pt]
		&\mm \leq \big( (1 + h^2_n) (1 - \alpha_1 h_n)(1 - \alpha_2 h_n) + (h^2_n + h^4_n) K_3  \big) \big( \| W_n \|^2_{\L_2} +  \| Z_n \|^2_{\L_2} \big) \\[3pt]
				&\mm \leq (1 - 2\alpha h_n) \big( \| W_n \|^2_{\L_2} +  \| Z_n \|^2_{\L_2} \big),
	\end{align*}
	where
	\begin{align*}
			\alpha &= \frac{\big(\lambda_+ z^2 - \frac{1}{2}u M_1 \big) \wedge \big(\frac{1}{2}u m - \lambda_+ w^2 \big) }{z - w} + \frac{\big(\lambda_- z^2 - \frac{1}{2}u M_1 \big) \wedge \big(\frac{1}{2}u m - \lambda_- w^2 \big) }{z - w} - K h_n \\[3pt]
				&= \frac{1}{z - w} \min \Big[  z^2 - uM_1\m, \, \lambda_+ z^2 - \lambda_- w^2 - \frac{1}{2}uM_1 + \frac{1}{2}um, \\
				&\hspace{32mm}  \lambda_- z^2 - \lambda_+ w^2 - \frac{1}{2}uM_1 + \frac{1}{2}um, \, um - w^2 \Big] - K h_n \\[3pt]
				&= \frac{1}{z - w} \min \Big[ z^2 - uM_1, \, um - w^2, \, (\lambda_- z^2 - \lambda_+ w^2) - \frac{1}{2} u( M_1 - m) \Big] - K h_n\m,
	\end{align*}
	and
	\begin{align*}
		K = 1 + \alpha_1 + \alpha_2 + \alpha_1 \cdot \alpha_2 + K_1 + K_2 + K_3\m.
	\end{align*}
\end{proof}

\begin{remark}
	Provided that $\gamma^2 > \frac{3 + \sqrt{3}}{2} u M_1\m$, the contraction rate $\alpha$ becomes positive if $w = 0$ and $h_n$ is sufficiently small. \label{thm:positive_contractivity_constant}
\end{remark}

\begin{remark}
	We can recover the result from Theorem \ref{thm:main_contractivity_theorem} using the inequality $1 - x \leq e^{-x}$.	Although it is not strictly necessary for the proof of the global error, it does establish how the QUICSORT method undergoes exponential contractivity at each time interval $[t_n, t_{n+1}]$.  
\end{remark}


\section{Local error bounds}
\label{sec:appendix_C}

In this section we prove the local error bounds in Theorem \ref{thm:local_error_bounds}. To do so, we compare the Taylor expansions of underdamped Langevin dynamics \eqref{eq:ULD} and the QUICSORT scheme \eqref{eq:QUICSORT}.

\begin{theorem}[Stochastic Taylor expansion of Langevin dynamics]
	Let $\{ x_t, v_t \}_{t\m\geq\m 0}$ be the underdamped Langevin diffusion \eqref{eq:ULD}. Then, for all $n\geq 0$, over any interval $[t_n, t_{n+1}]$,\label{thm:ULD_taylor_expansion}	\vspace{-4mm}
	\begin{subequations}
	\begin{align}
			x_{t_{n+1}} = x_{t_n} & + h_n v_{t_n} + \sigma \int^{t_{n+1}}_{t_n} W_{t_n, r_1} \, dr_1 + R^{\m (1)}_x, \mfill \label{eq:ULD_position_expansion_1} \\[3pt]
					= x_{t_n} & + h_n v_{t_n} + \sigma \int^{t_{n+1}}_{t_n} W_{t_n, r_1} \, dr_1 - \frac{1}{2} h^2_n \big( -\gamma v_{t_n} - u \nabla f(x_{t_n}) \big) \nonumber\\
					&  - \gamma \sigma \int^{t_{n+1}}_{t_n} \int^{r_1}_{t_n} W_{t_n, r_2} \, dr_2 \, dr_1 + R^{\m (2)}_x, \label{eq:ULD_position_expansion_2}  \\[3pt]
					= x_{t_n} & + h_n v_{t_n} + \sigma \int^{t_{n+1}}_{t_n} W_{t_n, r_1} \, dr_1 + \frac{1}{2} h_n^2 \big( -\gamma v_{t_n} - u \nabla f(x_{t_n}) \big) \nonumber \\
					& - \gamma \sigma \int^{t_{n+1}}_{t_n} \int^{r_1}_{t_n} W_{t_n, r_2} \, dr_2 \, dr_1 + \frac{1}{6} h^3_n \Big( \gamma^2 v_{t_n} + \gamma u \nabla f(x_{t_n}) - u \nabla^2 f(x_{t_n}) v_{t_n} \Big) \nonumber \\
					& + \sigma \big(\gamma^2 - u \nabla^2 f(x_{t_n}) \big) \int^{t_{n+1}}_{t_n} \int^{r_1}_{t_n} \int^{r_2}_{t_n} W_{t_n, r_3} \, dr_3 \, dr_2 \, dr_1 + R^{\m (3)}_x. \label{eq:ULD_position_expansion_3} 
	\end{align}
	\end{subequations}
	\vspace{-3mm}
	\begin{subequations}
	\begin{align}
			v_{t_{n+1}} = v_{t_n} & + \sigma W_n + h_n \big( -\gamma v_{t_n} - u \nabla f(x_{t_n}) \big) - \gamma \sigma \int^{t_{n+1}}_{t_n} W_{t_n, r_1} \, dr_1 + R^{\m (1)}_v, \label{eq:ULD_momentum_expansion_1} \\[3pt]
					= v_{t_n} & + \sigma W_n + h_n \big( -\gamma v_{t_n} - u \nabla f(x_{t_n}) \big) - \gamma \sigma \int^{t_{n+1}}_{t_n} W_{t_n, r_1} \, dr_1 \nonumber\\
					& + \frac{1}{2} h^2_n \Big( \gamma^2 v_{t_n} + \gamma u \nabla f(x_{t_n}) - u \nabla^2 f(x_{t_n}) v_{t_n} \Big)\nonumber \\
					& + \sigma \big(\gamma^2 - u \nabla^2 f(x_{t_n}) \big) \int^{t_{n+1}}_{t_n} \int^{r_1}_{t_n} W_{t_n, r_2} \, dr_2 \, dr_1 + R^{\m (2)}_v, \label{eq:ULD_momentum_expansion_2} \\[3pt]
					= v_{t_n} & + \sigma W_n + h_n \big( -\gamma v_{t_n} - u \nabla f(x_{t_n}) \big) - \gamma \sigma \int^{t_{n+1}}_{t_n} W_{t_n, r_1} \, dr_1, \nonumber\\
					& + \frac{1}{2} h^2_n \Big( \gamma^2 v_{t_n} + \gamma u \nabla f(x_{t_n}) - u \nabla^2 f(x_{t_n}) v_{t_n} \Big)\nonumber \\
					& + \sigma \big(\gamma^2 - u \nabla^2 f(x_{t_n}) \big) \int^{t_{n+1}}_{t_n} \int^{r_1}_{t_n} W_{t_n, r_2} \, dr_2 \, dr_1 + \frac{1}{6} h^3_n \big( -\gamma^3 v_{t_n} - \gamma^2 u \nabla f(x_{t_n})\big)\nonumber \\[3pt]
					& + \frac{1}{6} h_n^3\Big(2 \gamma u \nabla^2 f(x_{t_n}) v_{t_n} - u^2 \nabla f(x_{t_n}) \nabla f(x_{t_n}) - u \nabla^3 f(x_{t_n}) (v_{t_n}, v_{t_n}) \Big) \nonumber\\[3pt]
					& - \gamma \sigma  \big(\gamma^2 - 2 u \nabla^2 f(x_{t_n}) \big) \int^{t_{n+1}}_{t_n} \int^{r_1}_{t_n} \int^{r_2}_{t_n} W_{t_n, r_3} \, dr_3 \, dr_2 \, dr_1 \nonumber\\
					& - \sigma u \int^{t_{n+1}}_{t_n} \int^{r_1}_{t_n} \int^{r_2}_{t_n} \nabla^3 f(x_{t_n}) \big( v_{r_2}, W_{t_n, r_3} \big) \, dr_3 \, dr_2 \, dr_1 \nonumber \\
					& - \sigma u \int^{t_{n+1}}_{t_n} \int^{r_1}_{t_n} \nabla^3 f(x_{t_n}) \big(v_{t_n}, (r_2 - t_n) W_{t_n, r_2} \big) \, dr_2 \, dr_1 + R^{\m (3)}_v, \label{eq:ULD_momentum_expansion_3}
	\end{align}
	\end{subequations}
	where the remainder terms $R^{\m (1)}_x\m$, $R^{\m (2)}_x$ and $R^{\m (3)}_x$ for the position component are given by
	\begin{subequations}
		\begin{align}
			R^{\m (1)}_x &:= -\gamma \int^{t_{n+1}}_{t_n} \int^{r_1}_{t_n} v_{r_2} \, dr_2 \, dr_1 - u \int^{t_{n+1}}_{t_n} \int^{r_1}_{t_n} \nabla f(x_{r_2}) \, dr_2 \, dr_1\m, \mmmmm \m \label{eq:ULD_position_remainder_expansion_1} \\[6pt]
			R^{\m (2)}_x &:= \gamma^2 \int^{t_{n+1}}_{t_n} \int^{r_1}_{t_n} \int^{r_2}_{t_n} v_{r_3} \, dr_3 \, dr_2 \, dr_1 + \gamma u \int^{t_{n+1}}_{t_n} \int^{r_1}_{t_n} \int^{t_{n+1}}_{t_n} \nabla f(x_{r_3}) \, dr_3 \, dr_2 \, dr_1 \nonumber \\
					&\mmm - u \int^{t_{n+1}}_{t_n} \int^{r_1}_{t_n} \int^{r_2}_{t_n} \nabla^2 f(x_{r_3}) v_{r_3} \, dr_3 \, dr_2 \, dr_1\m, \label{eq:ULD_position_remainder_expansion_2} \\[6pt]
			R^{\m (3)}_x &:= - \gamma \big(\gamma^2 - u \nabla^2 f(x_{t_n}) \big) \int^{t_{n+1}}_{t_n} \int^{r_1}_{t_n} \int^{r_2}_{t_n} \int^{r_3}_{t_n} v_{r_4} \, dr_4 \, dr_3 \, dr_2 \, dr_1 \nonumber \\
				&\mmm - u \big(\gamma^2 - u \nabla^2 f(x_{t_n}) \big) \int^{t_{n+1}}_{t_n} \int^{r_1}_{t_n} \int^{r_2}_{t_n} \int^{r_3}_{t_n} \nabla f( x_{r_4}) \, dr_4 \, dr_3 \, dr_2 \, dr_1 \nonumber \\
				&\mmm + \gamma u \int^{t_{n+1}}_{t_n} \int^{r_1}_{t_n} \int^{r_2}_{t_n} \int^{r_3}_{t_n} \nabla^2 f(x_{r_4})\m v_{r_4}  \, dr_4 \, dr_3 \, dr_2 \, dr_1 \nonumber \\
				&\mmm - u \int^{t_{n+1}}_{t_n} \int^{r_1}_{t_n} \int^{r_2}_{t_n} \int^{r_3}_{t_n} \nabla^3 f(x_{r_4}) ( v_{r_4}\m, v_{r_3} )  \, dr_4 \, dr_3 \, dr_2 \, dr_1\m, \label{eq:ULD_position_remainder_expansion_3}
		\end{align}
	\end{subequations}
	and the remainder terms $R^{\m (1)}_v\m$, $R^{\m (2)}_v$ and $R^{\m (3)}_v$ for the momentum component are given by
	\begin{subequations}
		\begin{align}
			R^{\m (1)}_v &:= \gamma^2 \int^{t_{n+1}}_{t_n} \int^{r_1}_{t_n} v_{r_2} \, dr_2 \, dr_1 + \gamma u \int^{t_{n+1}}_{t_n} \int^{r_1}_{t_n} \nabla f(x_{r_2}) \, dr_2 \, dr_1 \nonumber \\
					 &\mmm - u \int^{t_{n+1}}_{t_n} \big(\nabla f(x_{r_1}) - \nabla f(x_{t_n})\big) \, dr_1, \label{eq:ULD_momentum_remainder_expansion_1} \\[6pt] 
			R^{\m (2)}_v &:= -\m\gamma \big(\gamma^2 - u \nabla^2 f(x_{t_n}) \big) \int^{t_{n+1}}_{t_n} \int^{r_1}_{t_n} \int^{r_2}_{t_n} v_{r_3} \, dr_3 \, dr_2 \, dr_1 \nonumber \\
					&\mmm - u \big(\gamma^2 - u \nabla^2 f(x_{t_n})\big)  \int^{t_{n+1}}_{t_n} \int^{r_1}_{t_n} \int^{r_2}_{t_n}\nabla f(x_{r_3}) \, dr_3 \, dr_2 \, dr_1 \nonumber \\
					 &\mmm + \gamma u  \int^{t_{n+1}}_{t_n} \int^{r_1}_{t_n} \int^{r_2}_{t_n} \nabla^2 f(x_{r_3}) v_{r_3} \, dr_3 \, dr_2 \, dr_1 \nonumber \\
					 &\mmm - u  \int^{t_{n+1}}_{t_n} \int^{r_1}_{t_n} \big(\nabla^2 f(x_{r_2}) v_{r_2} - \nabla^2 f(x_{t_n}) v_{r_2}\big) \, dr_2 \, dr_1\m, \label{eq:ULD_momentum_remainder_expansion_2} \\[6pt]
			R^{\m (3)}_v &:= \gamma^2 \big(\gamma^2 - 2 u \nabla^2 f(x_{t_n}) \big) \int^{t_{n+1}}_{t_n} \int^{r_1}_{t_n} \int^{r_2}_{t_n} \int^{r_3}_{t_n} v_{r_4} \, dr_4 \, dr_3 \, dr_2 \, dr_1 \nonumber \\
				&\mmm + \gamma u \big(\gamma^2 - 2 u \nabla^2 f(x_{t_n}) \big) \int^{t_{n+1}}_{t_n} \int^{r_1}_{t_n} \int^{r_2}_{t_n} \int^{r_3}_{t_n} \nabla f( x_{r_4}) \, dr_4 \, dr_3 \, dr_2 \, dr_1 \nonumber \\
				&\mmm - u \big(\gamma^2 - u \nabla^2 f(x_{t_n}) \big)  \int^{t_{n+1}}_{t_n} \int^{r_1}_{t_n} \int^{r_2}_{t_n} \int^{r_3}_{t_n} \nabla^2 f(x_{r_4}) v_{r_4}  \, dr_4 \, dr_3 \, dr_2 \, dr_1 \nonumber \\
				&\mmm + \gamma u \int^{t_{n+1}}_{t_n} \int^{r_1}_{t_n} \int^{r_2}_{t_n} \int^{r_3}_{t_n} \nabla^3 f(x_{r_4}) ( v_{r_4}\m, v_{r_3} )  \, dr_4 \, dr_3 \, dr_2 \, dr_1 \nonumber \\
				&\mmm + \gamma u \int^{t_{n+1}}_{t_n} \int^{r_1}_{t_n} \int^{r_2}_{t_n} \int^{r_3}_{t_n} \nabla^3 f(x_{t_n}) ( v_{r_4}\m, v_{r_2} )  \, dr_4 \, dr_3 \, dr_2 \, dr_1 \nonumber \\
				&\mmm + u^2 \int^{t_{n+1}}_{t_n} \int^{r_1}_{t_n} \int^{r_2}_{t_n} \int^{r_3}_{t_n} \nabla^3 f(x_{t_n}) \big( \nabla f(x_{r_4})\m, v_{r_2} \big)  \, dr_4 \, dr_3 \, dr_2 \, dr_1 \nonumber \\
				&\mmm + \gamma u \int^{t_{n+1}}_{t_n} \int^{r_1}_{t_n} \int^{r_2}_{t_n} \nabla^3 f(x_{t_n}) \big( v_{r_3}, (r_2 - t_n) v_{t_n} \big)  \, dr_3 \, dr_2 \, dr_1 \nonumber \\ 
				&\mmm + u^2 \int^{t_{n+1}}_{t_n} \int^{r_1}_{t_n} \int^{r_2}_{t_n} \nabla^3 f(x_{t_n}) \big( \nabla f(x_{r_3}) , (r_2 - t_n) v_{t_n} \big)  \, dr_3 \, dr_2 \, dr_1 \nonumber \\ 
				&\mmm - u  \int^{t_{n+1}}_{t_n} \int^{r_1}_{t_n} \int^{r_2}_{t_n} \big(\nabla^3 f(x_{r_3})(v_{r_3}\m, v_{r_2}) - \nabla^3 f(x_{t_n})(v_{r_3}\m, v_{r_2}) \big)\, dr_3 \, dr_2 \, dr_1.  \label{eq:ULD_momentum_remainder_expansion_3}
		\end{align}
	\end{subequations}
\end{theorem}

\begin{proof}
	For any $r \in [t_n, t_{n+1}]$, we can rewrite the SDE \eqref{eq:ULD} into integral form:
	\begin{align*}
			x_r &= x_{t_n} + \int^r_{t_n} v_{r_1} \, dr_1\m, \\[3pt]
			v_r &= v_{t_n} - \gamma \int^r_{t_n} v_{r_1} dr_1 - u \int^r_{t_n} \nabla f(x_{r_1}) \, dr_1 + \sigma W_{t_n, r}\m. \nonumber
	\end{align*}
	Assuming $\nabla f$ is sufficiently differentiable, by the fundamental theorem of calculus, we have
	\begin{align*}
		\nabla f(x_{r_i}) &= \nabla f(x_{t_n}) + \int^{r_i}_{t_n} \nabla^2 f(x_{r_{i+1}}) \, dx_{r_{i+1}} \\
			   	&=  \nabla f(x_{t_n}) + \int^{r_i}_{t_n} \nabla^2 f(x_{r_{i+1}}) v_{r_{i+1}} \, dr_{i+1}\m, \\[3pt]
		\nabla^2 f(x_{r_i}) v_r &= \nabla^2 f(x_{t_n}) v_r + \int^{r_i}_{t_n} \nabla^3 f(x_{r_{i+1}}) ( v_r,  dx_{t_{i+1}} ) \\
			   	&=  \nabla^2 f(x_{t_n}) v_r + \int^{r_i}_{t_n} \nabla^3 f(x_{r_{i+1}}) ( v_r, v_{r_{i+1}}) \, dr_{i+1}\m, \\[3pt]  \nonumber
	\end{align*}
	for $t_n \leq r_4\leq r_3 \leq r_2 \leq r_1 \leq t_{n+1}$. The theorem now follows by repeatedly substituting the above identities into the integral form of the SDE and rearranging the resulting terms. 
\end{proof}

\begin{theorem}[ULD Remainder Bounds]
	Let $\{(x_t, v_t)\}_{t\geq 0}$ be the underdamped Langevin diffusion \eqref{eq:ULD} where we assume that $(x_0, v_0) \sim \pi$, the stationary distribution of the process. Let $0 < h_n \leq 1$ be fixed and consider the remainder terms given by \eqref{eq:ULD_position_remainder_expansion_1} and \eqref{eq:ULD_momentum_remainder_expansion_1}, defined over the interval $[t_n, t_{n+1}]$. Under assumptions \eqref{eq:assumption_a1} and \eqref{eq:assumption_a2}, we can bound their $\L_2$ norms and absolute means as\label{thm:ULD_remainder_bounds}
	\begin{align}
		\big\| R^{\m (1)}_x \big\|_{\L_2} & \leq C^{\m (1)}_{x} \m \sqrt{d} \m h_n^2\m, \mmm\,\,\,\,\,\, \big\| R^{\m (1)}_v \big\|_{\L_2} \leq C^{\m (1)}_{v}  \m \sqrt{d} \m h_n^2\m,	\label{eq:ULD_remainder_strong_bounds_1} \\[5pt]
		\big\| \E \big[ R^{\m (1)}_x \big] \big\|_2 & \leq C^{\m (1)}_{x} \m \sqrt{d} \m h_n^2\m, \mmm \big\| \E \big[ R^{\m (1)}_v \big] \big\|_2 \leq C^{\m (1)}_{v} \m \sqrt{d} \m h_n^2\m,\label{eq:ULD_remainder_weak_bounds_1}
	\end{align}
	where $C^{\m (1)}_{x}, C^{\m (1)}_{v} > 0$ are constants depending on $\gamma$, $u$, $m$ and $M_1\m$.\medbreak

	\noindent	
	Under the additional assumption \eqref{eq:assumption_a3}, we can bound the terms \eqref{eq:ULD_position_remainder_expansion_2} and \eqref{eq:ULD_momentum_remainder_expansion_2} as
	\begin{align}
		\big\| R^{\m (2)}_x \big\|_{\L_2} & \leq C^{\m (2)}_{x} \, \sqrt{d} \m h_n^3\m, \mmm\,\,\,\,\,\, \big\| R^{\m (2)}_v \big\|_{\L_2} \leq C^{\m (2)}_{v}  \, d h_n^3\m, 	\label{eq:ULD_remainder_strong_bounds_2}\\[5pt]
		\big\| \E\big[ R^{\m (2)}_x \big] \big\|_2 & \leq C^{\m (2)}_{x} \, \sqrt{d} \m h_n^3\m, \mmm \big\| \E \big[ R^{\m (2)}_v \big] \big\|_2 \leq C^{\m (2)}_{v} \, d h_n^3\m,	\label{eq:ULD_remainder_weak_bounds_2}
	\end{align}
	where $C^{\m (2)}_{x}, C^{\m (2)}_{v} > 0$ are constants which depend on $\gamma$, $u$, $m$, $M_1$ and $M_2\m$. Finally, under the additional assumption \eqref{eq:assumption_a4}, the remainder terms \eqref{eq:ULD_position_remainder_expansion_3} and \eqref{eq:ULD_momentum_remainder_expansion_3} can be bounded by
	\begin{align}
		\big\| R^{\m (3)}_x \big\|_{\L_2} & \leq C^{\m (3)}_{x} \, dh_n^4\m, \mmm\,\,\,\,\,\, \big\| R^{\m (3)}_v \big\|_{\L_2} \leq C^{\m (3)}_{v}  \, d^{1.5}\m  h_n^4\m,	\label{eq:ULD_remainder_strong_bounds_3} \\[5pt]
		\big\| \E \big[ R^{\m (3)}_x \big] \big\|_2 & \leq C^{\m (3)}_{x} \, dh_n^4\m, \mmm \big\| \E \big[ R^{\m (3)}_v \big] \big\|_2 \leq C^{\m (3)}_{v} \, d^{1.5} \m h_n^4\m,	\label{eq:ULD_remainder_weak_bounds_3}
	\end{align}
	where $C^{\m (3)}_{x}, C^{\m (3)}_{v} > 0$ are constants depending on $\gamma$, $u$, $m$, $M_1$, $M_2$ and $M_3\m$.
\end{theorem}
\begin{proof}
	Suppose that $\gamma^2 \geq 2uM_1\m$. Then for all $x,v \in \mathbb{R}^d$,
	\begin{align*}
		\big\| (\gamma^2 - 2u\nabla^2f (x))v \big\|_{\L_2} \leq \big\| (\gamma^2 - u\nabla^2f (x))v \big\|_{\L_2} \leq ( \gamma^2 - u m ) \| v \|_{\L_2}.
	\end{align*}
	We can derive the local error bounds \eqref{eq:ULD_remainder_strong_bounds_1}, \eqref{eq:ULD_remainder_strong_bounds_2} and \eqref{eq:ULD_remainder_strong_bounds_3} for the position variable using this inequality, Fubini's theorem and the bounds  \eqref{eq:nabla_f_bounds}, \eqref{eq:v_bounds}, \eqref{eq:n2_f_x_v} and \eqref{eq:n3_f_x_v1_v2} given in Appendix \ref{sec:appendix_A}.
	\begin{align*}
			\big\| R^{\m (1)}_x \big\|_{\L_2} &\leq \gamma \int^{t_{n+1}}_{t_n} \int^{r_1}_{t_n} \| v_{r_2} \|_{\L_2} \, dr_2 \, dr_1 + u \int^{t_{n+1}}_{t_n} \int^{r_1}_{t_n} \big\| \nabla f(x_{r_2}) \big\|_{\L_2} \, dr_2 \, dr_1 \\[5pt]
									&\leq \frac{1}{2} h^2_n \big( \gamma \sqrt{u d} + u \sqrt{M_1 \m d}\, \big) \\[5pt]
									&= C^{\m (1)}_x \sqrt{d} \m h_n^2\m, \\[7pt]
			\| R^{\m (2)}_x \|_{\L_2} &\leq \gamma^2 \int^{t_{n+1}}_{t_n} \int^{r_1}_{t_n} \int^{r_2}_{t_n} \| v_{r_3} \|_{\L_2} \, dr_3 \, dr_2 \, dr_1 \\[3pt]
					&\hspace{12.5mm} + \gamma u \int^{t_{n+1}}_{t_n} \int^{r_1}_{t_n} \int^{t_{n+1}}_{t_n} \big\| \nabla f(x_{r_3}) \big\|_{\L_2} \, dr_3 \, dr_2 \, dr_1 \\[3pt]
					&\hspace{12.5mm} + u \int^{t_{n+1}}_{t_n} \int^{r_1}_{t_n} \int^{r_2}_{t_n} \| \nabla^2 f(x_{r_3}) v_{r_3} \|_{\L_2} \, dr_3 \, dr_2 \, dr_1 \\[5pt]
					&\leq \frac{1}{6} h^3_n \Big( \gamma^2 \sqrt{ u d} + \gamma u \sqrt{M_1 \m d} + uM_1 \sqrt{ud}\,\Big) \\[5pt]
					&= C^{\m (2)}_x \sqrt{d} \m h_n^3\m, \\[7pt]
			\| R^{\m (3)}_x \|_{\L_2} &\leq  \gamma ( \gamma^2 - um ) \int^{t_{n+1}}_{t_n} \int^{r_1}_{t_n} \int^{r_2}_{t_n} \int^{r_3}_{t_n} \| v_{r_4} \|_{\L_2} \, dr_4 \, dr_3 \, dr_2 \, dr_1 \\[3pt]
	     	 			&\hspace{12.5mm} + u ( \gamma^2 - um ) \int^{t_{n+1}}_{t_n} \int^{r_1}_{t_n} \int^{r_2}_{t_n} \int^{r_3}_{t_n} \big\| \nabla f( x_{r_4}) \big\|_{\L_2} \, dr_4 \, dr_3 \, dr_2 \, dr_1 \\[3pt]
	      				&\hspace{12.5mm} + \gamma u \int^{t_{n+1}}_{t_n} \int^{r_1}_{t_n} \int^{r_2}_{t_n} \int^{r_3}_{t_n} \big\| \nabla^2 f(x_{r_4})v_{r_4} \big\|_{\L_2} \, dr_4 \, dr_3 \, dr_2 \, dr_1  \\[3pt]
	      				&\hspace{12.5mm} + u \int^{t_{n+1}}_{t_n} \int^{r_1}_{t_n} \int^{r_2}_{t_n} \int^{r_3}_{t_n} \big\| \nabla^3 f(x_{r_4}) (v_{r_4}, v_{r_3}) \big\|_{\L_2} \, dr_4 \, dr_3 \, dr_2 \, dr_1 \\[5pt]
	      				&\leq \frac{1}{24} h_n^4 \Big( \gamma ( \gamma^2 - um ) \sqrt{ud} + u ( \gamma^2 - um ) \sqrt{M_1 \m d} + \gamma u M_1 \sqrt{ud} + \sqrt{3} \m M_2 \m u d\m \Big)\\[5pt]
	      				&\leq \frac{1}{24} \Big( \gamma ( \gamma^2 - um ) \sqrt{u} + u ( \gamma^2 - um ) \sqrt{M_1} + \gamma u M_1 \sqrt{u} + \sqrt{3} \m M_2 \m u\m \Big) d h_n^4 \\[5pt]
	      				&= C^{\m (3)}_{x} d h_n^4\m.
	\end{align*}
	For the momentum variable, we can additionally use \eqref{eq:n3_f_x1_x2_v1_v2} and the bound
	\begin{align*}
		\| x_t - x_{t_n} \|_{\L_{2p}} \leq \int^t_{t_n} \| v_{r_1} \|_{\L_{2p}} \, dr_1\m,
	\end{align*}
	to show that
	\begin{align*}
			\big\| R^{\m (1)}_v \big\|_{\L_2} &\leq \gamma^2 \int^{t_{n+1}}_{t_n} \int^{r_1}_{t_n} \| v_{r_2} \|_{\L_2} \, dr_2 \, dr_1 + \gamma u \int^{t_{n+1}}_{t_n} \int^{r_1}_{t_n} \big\| \nabla f(x_{r_2}) \big\|_{\L_2} \, dr_2 \, dr_1 \\[3pt]
					 &\mmm + u \int^{t_{n+1}}_{t_n} \big\| \nabla f(x_{r_1}) - \nabla f(x_{t_n}) \big\|_{\L_2} \, dr_1 \\[3pt]
					 &\leq \frac{1}{2} h^2_n \Big( \gamma^2 \sqrt{ud} + \gamma u \sqrt{M_1 \m d} + uM_1 \sqrt{ud}\, \Big) \\[5pt]
					 &= C^{\m (1)}_v \sqrt{d} \m h_n^2\m, \\[7pt]
			\big\| R^{\m (2)}_v \big\|_{\L_2} &\leq \gamma (\gamma^2 - u m ) \int^{t_{n+1}}_{t_n} \int^{r_1}_{t_n} \int^{r_2}_{t_n} \| v_{r_3} \|_{\L_2} \, dr_3 \, dr_2 \, dr_1 \\[3pt]
					&\mmm + u (\gamma^2 - u m )  \int^{t_{n+1}}_{t_n} \int^{r_1}_{t_n} \int^{r_2}_{t_n} \big\| \nabla f(x_{r_3}) \big\|_{\L_2} \, dr_3 \, dr_2 \, dr_1 \\[3pt]
					 &\mmm + \gamma u  \int^{t_{n+1}}_{t_n} \int^{r_1}_{t_n} \int^{r_2}_{t_n} \big\| \nabla^2 f(x_{r_3}) v_{r_3} \big\|_{\L_2} \, dr_3 \, dr_2 \, dr_1 \\[3pt]
					 &\mmm + u  \int^{t_{n+1}}_{t_n} \int^{r_1}_{t_n} \big\| \nabla^2 f(x_{r_2}) v_{r_2} - \nabla^2 f(x_{t_n}) v_{r_2} \big\|_{\L_2} \, dr_2 \, dr_1 \\[5pt]
					 &\leq \frac{1}{6} h_n^3 \Big( \gamma ( \gamma^2 - um ) \sqrt{ud} + u ( \gamma^2 - um ) \sqrt{M_1 \m d} + \gamma u M_1 \sqrt{ud} + \sqrt{3} \m M_2 \m u d \Big)\\[5pt]
	      				&\leq \frac{1}{6} \Big( \gamma \left( \gamma^2 - um \right) \sqrt{u} + u ( \gamma^2 - um ) \sqrt{M_1} + \gamma u M_1 \sqrt{u} + \sqrt{3} \m M_2 \m u \Big) d h_n^3 \\[5pt]
	      				&= C^{\m (2)}_{v} d h_n^3\m,  \\[5pt]
	\big\| R^{\m (3)}_v \big\|_{\L_2} &\leq \gamma^2 ( \gamma^2 - u m ) \int^{t_{n+1}}_{t_n} \int^{r_1}_{t_n} \int^{r_2}_{t_n} \int^{r_4}_{t_n} \| v_{r_4} \|_{\L_2} \, dr_4 \, dr_3 \, dr_2 \, dr_1  \\[3pt]
	      &\mmm + \gamma u  ( \gamma^2 - u m ) \int^{t_{n+1}}_{t_n} \int^{r_1}_{t_n} \int^{r_2}_{t_n} \int^{r_4}_{t_n} \big\| \nabla f(x_{r_4}) \big\|_{\L_2}\, dr_4 \, dr_3 \, dr_2 \, dr_1  \\[3pt]
	      &\mmm  + u ( \gamma^2 - u m ) \int^{t_{n+1}}_{t_n} \int^{r_1}_{t_n} \int^{r_2}_{t_n} \int^{r_3}_{t_n} \big\| \nabla^2 f(x_{r_4}) v_{r_4} \big\|_{\L_2}\, dr_4 \, dr_3 \, dr_2 \, dr_1   \\[3pt]
	      &\mmm + \gamma u \int^{t_{n+1}}_{t_{n}} \int^{r_1}_{t_n} \int^{r_2}_{t_n} \int^{r_3}_{t_n} \big\| \nabla^3 f(x_{r_4}) ( v_{r_4}, v_{r_3}) \big\|_{\L_2} \, dr_4 \, dr_3 \, dr_2 \, dr_1  \\[3pt]
	      &\mmm + \gamma u \int^{t_{n+1}}_{t_n} \int^{r_1}_{t_n} \int^{r_2}_{t_n} \int^{r_3}_{t_n} \big\| \nabla^3 f(x_{t_n}) (v_{r_4}, v_{r_2}) \big\|_{\L_2} \, dr_4 \, dr_3 \, dr_2 \, dr_1  \\[3pt]
	     &\mmm + u^2 \int^{t_{n+1}}_{t_n} \int^{r_1}_{t_n} \int^{r_2}_{t_n} \int^{r_3}_{t_n} \big\| \nabla^3 f(x_{t_n}) \big( \nabla f(x_{r_4}), v_{r_2} \big) \big\|_{\L_2} \, dr_4 \, dr_3 \, dr_2 \, dr_1  \\[3pt]
	      &\mmm + \gamma u \int^{t_{n+1}}_{t_n} \int^{r_1}_{t_n} \int^{r_2}_{t_n} \big\| \nabla^3 f(x_{t_n}) ( v_{r_3}, (r_2 - t_n) v_{t_n} ) \big\|_{\L_2} \, dr_3 \, dr_2 \, dr_1  \\[3pt]
	     &\mmm + u^2 \int^{t_{n+1}}_{t_n} \int^{r_1}_{t_n} \int^{r_2}_{t_n} \big\| \nabla^3 f(x_{t_n}) \big( \nabla f(x_{r_3}) , (r_2 - t_n) v_{t_n} \big) \big\|_{\L_2} \, dr_3 \, dr_2 \, dr_1  \\[3pt]
	      &\mmm - u  \int^{t_{n+1}}_{t_n} \int^{r_1}_{t_n} \int^{r_2}_{t_n} \big\|  \nabla^3 f(x_{r_3}) (v_{r_3}\m, v_{r_2}) - \nabla^3 f(x_{t_n})  (v_{r_3}\m, v_{r_2} ) \big\|_{\L_2} \, dr_3 \, dr_2 \, dr_1  \\[5pt]
	       &\leq \frac{1}{24}  h^4_n \Big( \gamma^2 ( \gamma^2 - u m ) \sqrt{ud} + \gamma u ( \gamma^2 - u m ) \sqrt{M_1 \m d} + u ( \gamma^2 - u m ) M_1 \sqrt{u d}   \\[3pt]
	       &\hspace{20mm} + 3 \sqrt{3} \m \gamma u^2 M_2 \m d + 2\sqrt{3} \m u^{2.5} \sqrt{M_1} M_2 d + \sqrt{15}\m  u^{1.5} M_3 \m d^{1.5}  \Big) \\[5pt]
	       &\leq \frac{1}{24}  \Big( \gamma^2 ( \gamma^2 - u m ) \sqrt{u} + \gamma u^{1.5} ( \gamma^2 - u m ) \sqrt{M_1} + u ( \gamma^2 - u m ) M_1 \sqrt{u}   \\[3pt]
	       &\hspace{15mm} + 3 \sqrt{3} \m \gamma u^2 M_2 + 2\sqrt{3} \m u^{2.5} \sqrt{M_1} M_2 + \sqrt{15}\m  u^{1.5} M_3  \Big) d^{1.5} \m h^4_n \\[5pt]
	       &= C_v^{\m (3)} d^{1.5} h_n^4\m.
	\end{align*}
In order to prove the bounds \eqref{eq:ULD_remainder_weak_bounds_1}, \eqref{eq:ULD_remainder_weak_bounds_2} and \eqref{eq:ULD_remainder_weak_bounds_3} for the mean errors, we simply note that
	\begin{align*}
		\big\| \E [ R ] \big\|_2 & \leq \E\big[\|R\|_2\big] \leq \E\big[\|R\|_2^2\big]^\frac{1}{2} = \big\| R \big\|_{\L_2}\m,
	\end{align*}
	by Jensen's inequality. 
\end{proof}

\begin{theorem}[Taylor expansion of the QUICSORT method]
	Let $\{ X_n\m, V_n \}_{n \geq 0}$ denote the QUICSORT method \eqref{eq:QUICSORT}. Then for all $n\geq 0$, we have the following Taylor expansions:\label{thm:QUICSORT_Taylor_expansion}
	\begin{subequations}
		\begin{align}
			X_{n+1} = X_n & + h_n V_n + \sigma \int^{t_{n+1}}_{t_n} W_{t_n, r_1} \, dr_1 + R^{\m (1)}_X, \label{eq:QUICSORT_position_expansion_1} \\[6pt]
				=  X_n & + h_n V_n + \sigma \int^{t_{n+1}}_{t_n} W_{t_n, r_1} \, dr_1 + \frac{1}{2} h^2_n \big[ -\gamma V_n - u \nabla f(X_n) \big] \nonumber \\
				& -\gamma \sigma \int^{t_{n+1}}_{t_n} \int^{r_1}_{t_n} W_{t_n, r_2} \, dr_2 \, dr_1 + R^{\m (2)}_X, \label{eq:QUICSORT_position_expansion_2} \\[6pt]
				= X_n & + h_n V_n + \sigma \int^{t_{n+1}}_{t_n} W_{t_n, r_1} \, dr_1 + \frac{1}{2} h^2_n \big[ -\gamma V_n - u \nabla f(X_n) \big] \nonumber \\
					& -\gamma \sigma \int^{t_{n+1}}_{t_n} \int^{r_1}_{t_n} W_{t_n, r_2} \, dr_2 \, dr_1 + \frac{1}{6} h^3_n \big[ \gamma^2 V_n + \gamma u \nabla f(X_n) - u \nabla^2 f(X_n) V_n \big] \nonumber \\
					& + \sigma \big( \gamma^2 - u \nabla^2 f(X_n) \big) \bigg[ \frac{1}{24} h_n^3 W_n + \frac{1}{6} h_n^3 H_n + \frac{1}{2} h_n^3 K_n \bigg] + R^{\m (3)}_X, \label{eq:QUICSORT_position_expansion_3}
		\end{align}
	\end{subequations}\vspace{-3mm}
	\begin{subequations}
		\begin{align}
			V_{n+1} = V_n & + \sigma W_n + h_n \big[ -\gamma V_n - u\nabla f(X_n) \big] - \gamma \sigma \int^{t_{n+1}}_{t_n} W_{t_n, r_1} \, dr_1 + R^{\m (1)}_V, \label{eq:QUICSORT_momentum_expansion_1} \\[6pt]
				 =  V_n & + \sigma W_n + h_n \big[ -\gamma V_n - u\nabla f(X_n) \big] - \gamma \sigma \int^{t_{n+1}}_{t_n} W_{t_n, r_1} \, dr_1 \nonumber \\
					& + \frac{1}{2} h^2_n \big[ \gamma^2 V_n + \gamma u \nabla f(X_n) - u \nabla^2 f(X_n) V_n \big] \nonumber \\
					& + \sigma \big( \gamma^2 - u\nabla^2 f(X_n) \big) \int^{t_{n+1}}_{t_n} \int^{r_1}_{t_n} W_{t_n, r_2} \, dr_2 \, dr_1 + R^{\m (2)}_V, \label{eq:QUICSORT_momentum_expansion_2} \\[6pt]
				= V_n & + \sigma W_n + h_n \big[ -\gamma V_n - u\nabla f(X_n) \big] - \gamma \sigma \int^{t_{n+1}}_{t_n} W_{t_n, r_1} \, dr_1 \nonumber \\
					& + \frac{1}{2} h^2_n \big[ \gamma^2 V_n + \gamma u \nabla f(X_n) - u \nabla^2 f(X_n) V_n \big] \nonumber \\
					& + \sigma \big( \gamma^2 - u\nabla^2 f(X_n) \big) \int^{t_{n+1}}_{t_n} \int^{r_1}_{t_n} W_{t_n, r_2} \, dr_2 \, dr_1 - \frac{1}{6} h^3_n \big[\gamma^3 V_n + \gamma^2 u \nabla f(X_n)\big] \nonumber \\
					& + \frac{1}{6} h^3_n \big[2 \gamma u \nabla^2 f(X_n) V_n - u^2 \nabla f(X_n) \nabla f(X_n) - u \nabla^3 f(X_n) (V_n, V_n) \big] \nonumber \\
					& - \gamma \sigma \big( \gamma^2 - 2u \nabla^2 f(X_n) \big) \bigg[ \frac{1}{24} h^3_n W_n + \frac{1}{6} h^3_n H_n + \frac{1}{2} h^3_n K_n \bigg] \nonumber \\
					& -\sigma u \nabla^3 f(X_n) \bigg( V_n\m, \bigg[ \frac{1}{8} h^3_n W_n + \frac{1}{3} h^3_n H_n + \frac{1}{2} h^3_n K_n \bigg] \bigg) + R^{\m (3)}_V, \label{eq:QUICSORT_momentum_expansion_3}
		\end{align}
	\end{subequations}
	where
	\begin{subequations}
		\begin{align}
			R^{\m (1)}_X &= ( \phi_1 (1) - h_n ) \big( V_n + \sigma ( H_n + 6 K_n ) \big) + \bigg[ \frac{1}{h_n} \phi_2(1) - \frac{1}{2} h_n \bigg] \sigma (W_n - 12K_n) \nonumber \mmm \\
						&\mmm - \frac{1}{2} u h_n \big( \phi_1(\lambda_+) + \phi_1(\lambda_-) \big) \nabla f(X_n) \nonumber \\
                        &\mmm - \frac{1}{2} u h_n \Big[ \phi_1(\lambda_+) I^{\m (-)}_1 + \phi_1(\lambda_-) I^{\m (+)}_1 + \phi_1 (\lambda_-) \widetilde{I}_1 \Big], \label{eq:QUICSORT_position_remainder_expansion_1} \\[6pt]
			R^{\m (2)}_X &= \bigg[ \phi_1 (1) - \bigg( h_n - \frac{1}{2} \gamma h^2_n \bigg) \bigg] \big( V_n + \sigma ( H_n + 6 K_n ) \big) \nonumber \\[2pt]
						&\mmm + \bigg[ \frac{1}{h_n} \phi_2(1) - \bigg( \frac{1}{2} h_n - \frac{1}{6} \gamma h^2_n \bigg) \bigg] \sigma (W_n - 12K_n ) \nonumber \\[2pt]
						&\mmm + \bigg[ - \frac{1}{2} h_n \big( \phi_1 (\lambda_+) + \phi_1 (\lambda_-) \big) - \bigg( - \frac{1}{2} h_n^2 \bigg) \bigg] u \nabla f(X_n) \nonumber  \\[2pt]
						&\mmm - \frac{1}{2} u h_n \Big[ \phi_1(\lambda_+) I^{\m (-)}_1 + \phi_1(\lambda_-) I^{\m (+)}_1 + \phi_1 (\lambda_-) \widetilde{I}_1 \Big], \label{eq:QUICSORT_position_remainder_expansion_2} \\[6pt]
			R^{\m (3)}_X &= \bigg[ \phi_1 (1) - \bigg( h_n - \frac{1}{2} \gamma h_n^2 + \frac{1}{6} \gamma^2 h_n^3 \bigg) \bigg] \big( V_n + \sigma ( H_n + 6 K_n ) \big) \nonumber  \\
				&\mmm + \bigg[ \frac{1}{h_n} \phi_2 (1) - \bigg( \frac{1}{2} h_n - \frac{1}{6} \gamma h_n^2 + \frac{1}{24} \gamma^2 h_n^3 \bigg) \bigg] \sigma (W_n - 12K_n ) \nonumber \\[2pt]
				&\mmm + \bigg[ - \frac{1}{2} h_n \bigg( \phi_1 (\lambda_+) + \phi_1 (\lambda_-) \bigg) - \bigg( - \frac{1}{2} h_n^2 + \frac{1}{6} \gamma h_n^3 \bigg) \bigg] u \nabla f(X_n) \nonumber \\[2pt]
				&\mmm + \bigg[ - h_n \phi_1 (\lambda_+) \phi_1 (\lambda_-) - \bigg( - \frac{1}{6} h_n^3 \bigg) \bigg] u \nabla^2 f(X_n) \bigg[ V_n + \sigma ( H_n + 6K_n ) \bigg] \nonumber \\[2pt]
				&\mmm + \bigg[ - \frac{1}{2} \big( \phi_1 (\lambda_+) \phi_2 (\lambda_-) + \phi_1 (\lambda_-) \phi_2 (\lambda_+) \big) - \bigg( - \frac{1}{24} h_n^3 \bigg) \bigg] \nonumber \\
				&\hspace{20mm} \times \sigma u \nabla^2 f(X_n) ( W_n - 12 K_n ) \nonumber \\
				&\mmm  -\frac{1}{2} u h_n \Big[ \phi_1 (\lambda_+) I^{\m (-)}_2 + \phi_1 (\lambda_-) I^{\m (+)}_2 + \phi_1 (\lambda_-) \widetilde{I}_1 \Big], \label{eq:QUICSORT_position_remainder_expansion_3}\\[20pt]
			R^{\m (1)}_V &= \big( \phi_0 (1) - (1 - \gamma h_n) \big) \big( V_n + \sigma ( H_n + 6 K_n ) \big) \nonumber \\
						&\mmm + \bigg[ \frac{1}{h_n} \phi_1 (1) - \bigg(1 - \frac{1}{2} \gamma h_n \bigg) \bigg] \sigma (W_n - 12K_n ) \nonumber \\
						&\mmm + \bigg[ -\frac{1}{2} h_n \big( \phi_0 (\lambda_+) + \phi_0 (\lambda_-) \big) - (- h_n ) \bigg] u \nabla f(X_n) \nonumber \\
						&\mmm - \frac{1}{2} u h_n \Big[ \phi_0 (\lambda_+) I^{\m (-)}_1 + \phi_0 (\lambda_-) I^{\m (+)}_1 + \phi_0 (\lambda_-) \widetilde{I}_1\Big], \label{eq:QUICSORT_momentum_remainder_expansion_1} \\[6pt]
			R^{\m (2)}_V &= \bigg[ \phi_0 (1) - (1 - \gamma h_n + \frac{1}{2} \gamma^2 h^2_n) \bigg] \big( V_n + \sigma ( H_n + 6 K_n ) \big) \nonumber \\
						&\mmm + \bigg[ \frac{1}{h_n} \phi_1 (1) - \bigg(1 - \frac{1}{2} \gamma h_n + \frac{1}{6} \gamma^2 h^2_n \bigg) \bigg] \sigma (W_n - 12K_n ) \nonumber \\[2pt]
						&\mmm + \bigg[ -\frac{1}{2} h_n \big( \phi_0 (\lambda_+) + \phi_0 (\lambda_-) \big) - \bigg(- h_n + \frac{1}{2} \gamma h^2 \bigg) \bigg] u \nabla f(X_n) \nonumber \\[2pt]
						&\mmm + \bigg[ -\frac{1}{2} h_n \big( \phi_0 (\lambda_+) \phi_1 (\lambda_-) + \phi_0 (\lambda_-) \phi_1 (\lambda_+) \big) -  \bigg(-\frac{1}{2} h_n^2 \bigg) \bigg] \nonumber \\[2pt]
						&\hspace{20mm} \times u \nabla^2 f(X_n)  \big(V_n + \sigma ( H_n + 6 K_n ) \big) \nonumber \\[2pt]
						&\mmm + \bigg[ -\frac{1}{2} \Big( \phi_0 (\lambda_+) \phi_2 (\lambda_-) + \phi_0 (\lambda_-) \phi_2 (\lambda_+) \Big) - \bigg(-\frac{1}{6} h_n^2 \bigg) \bigg] \nonumber \\[2pt]
						&\hspace{20mm} \times \sigma u \nabla^2 f(X_n) ( W_n - 12K_n ) \nonumber \\
						&\mmm - \frac{1}{2} u h_n \Big[ \phi_0 (\lambda_+) I^{\m (-)}_2 + \phi_0 (\lambda_-) I^{\m (+)}_2 + \phi_0 (\lambda_-) \widetilde{I}_1 \Big], \label{eq:QUICSORT_momentum_remainder_expansion_2} \\[6pt]
			R^{\m (3)}_V &= \bigg[ \phi_0(1) - \bigg(1 - \gamma h_n + \frac{1}{2} \gamma^2 h_n^2 - \frac{1}{6}\gamma^3 h_n^3\bigg)\bigg] \big(V_n + \sigma (H_n + 6K_n)\big) \nonumber  \\
				&\mmm + \bigg[\frac{1}{h_n}\phi_1(1) - \bigg( 1 - \frac{1}{2}\gamma h_n + \frac{1}{6}\gamma^2 h_n^2 - \frac{1}{24} \gamma^3 h_n^3 \bigg) \bigg] \sigma(W_n - 12K_n) \nonumber \\[2pt]
				&\mmm + \bigg[ -\frac{1}{2} h_n \big( \phi_0 (\lambda_+) + \phi_0 (\lambda_-) \big) - \bigg(-h + \frac{1}{2} \gamma h^2 - \frac{1}{6} \gamma^2 h_n^3 \bigg) \bigg] u \nabla f(X_n) \nonumber \\[2pt]
				&\mmm + \bigg[ -\frac{1}{2} h_n \left( \phi_0 (\lambda_+) \phi_1 (\lambda_-) + \phi_0 (\lambda_-) \phi_1 (\lambda_+) \right) - \left(-\frac{1}{2} h_n^2 + \frac{1}{3} \gamma h_n^3\right) \bigg] \nonumber \\[2pt]
				&\hspace{20mm} \times u \nabla^2 f(X_n) \big( V_n + \sigma ( H_n + 6K_n ) \big) \nonumber \\[2pt]
				&\mmm + \bigg[ -\frac{1}{2} \big( \phi_0 (\lambda_+) \phi_2 (\lambda_-) + \phi_0 (\lambda_-) \phi_2 (\lambda_+) \big) - \bigg(-\frac{1}{6} h_n^2 + \frac{1}{12} \gamma h_n^3\bigg) \bigg] \nonumber \\
				&\hspace{20mm} \times \sigma u \nabla^2 f(X_n) ( W_n - 12 K_n ) \nonumber \\[2pt]
				&\mmm + \bigg[ \frac{1}{2} h_n^2 \phi_0 (\lambda_-) \phi_1 \bigg(\frac{1}{3}\bigg) - \frac{1}{6} h_n^3 \bigg] u^2 \nabla^2 f(X_n) \nabla f(X_n) \nonumber \\[2pt]
				&\mmm + \bigg[ -\frac{1}{4} h_n \big( \phi_0 (\lambda_+) \phi^2_1 (\lambda_-) + \phi_0 (\lambda_-) \phi^2_1 (\lambda_+) \big) - \bigg( - \frac{1}{6} h_n^3 \bigg) \bigg] \nonumber \\
				&\hspace{20mm} \times u \nabla^3 f(X_n) \big( V^{\otimes 2}_n + 2\sigma V_n\otimes ( H_n + 6K_n ) \big) \nonumber \\[2pt]
				&\mmm + \bigg[ -\frac{1}{4} h_n \big( \phi_0 (\lambda_+) \phi^2_1 (\lambda_-) + \phi_0 (\lambda_-) \phi^2_1 (\lambda_+) \big) \bigg] \sigma^2 u \nabla^3 f(X_n) \big( H_n + 6K_n \big)^{\otimes 2} \nonumber \\[2pt]
				&\mmm + \bigg[ -\frac{1}{2} \big( \phi_0 (\lambda_+) \phi_1 (\lambda_-) \phi_2 (\lambda_-) + \phi_0 (\lambda_-) \phi_1 (\lambda_+) \phi_2 (\lambda_+) \big) - \bigg( -\frac{1}{8} h_n^3 \bigg) \bigg] \nonumber \\
				&\hspace{20mm} \times \sigma u \nabla^3 f(X_n) \big(V_n\m,  W_n - 12K_n \big) \allowbreak \nonumber \\[2pt]
				&\mmm + \bigg[ -\frac{1}{2} \bigg( \phi_0 (\lambda_+) \phi_1 (\lambda_-) \phi_2 (\lambda_-) + \phi_0 (\lambda_-) \phi_1 (\lambda_+) \phi_2 (\lambda_+) \bigg) \bigg] \nonumber \\
				&\hspace{20mm} \times \sigma^2 u \nabla^3 f(X_n) \big( H_n + 6K_n\m, W_n - 12K_n \big) \nonumber \\[2pt]
				&\mmm + \bigg[ -\frac{1}{4} \cdot \frac{1}{h_n} \big( \phi_0 (\lambda_+) \phi^2_2 (\lambda_-) + \phi_0 (\lambda_-) \phi^2_2 (\lambda_+) \big) \bigg] \sigma^2 u \nabla^3 f(X_n) (W_n - 12K_n)^{\otimes 2} \nonumber \\[2pt]
				&\mmm -\frac{1}{2} u h_n \Big[ \phi_0 (\lambda_+) I^{\m (-)}_3 + \phi_0 (\lambda_-) I^{\m (+)}_3 + \phi_0 (\lambda_-) \widetilde{I}_2 + \phi_0 (\lambda_-) \widetilde{I}_3 \Bigg] \nonumber \\[2pt]
				&\mmm + \frac{1}{2} u^2 h_n^2 \phi_0 (\lambda_-) \phi_1 \bigg(\frac{1}{3} \bigg)  \nabla^2 f(X_n) \widetilde{I}^{(1)}, \label{eq:QUICSORT_momentum_remainder_expansion_3}
		\end{align}
	\end{subequations}
	for $\lambda_\pm = \frac{3 \pm \sqrt{3}}{6}$, along with the functions $\phi_0\m$, $\phi_1$ and $\phi_2$ defined in \eqref{eq:phi_functions}, and
	\begin{align}
	\label{eq:QUICSORT_expansion_integrals}
		\begin{split}
            I^{\m (\pm)}_1 &:= \int^1_0 \nabla^2 f \big( X_n + t L^{\m (\pm)}_n \big) L^{\m (\pm)}_n \, dt\\[3pt]
	                   & = \Big[\m\nabla f\big( X_n +  L^{\m (\pm)}_n \big) - \nabla f(X_n) \Big],\\[5pt]
			I^{\m (\pm)}_2 & := \int^1_0 (1-t) \nabla^3 f \big( X_n + t  L^{\m (\pm)}_n \big) \big(  L^{\m (\pm)}_n \big)^{\otimes 2} dt\\[3pt]
						&= \int^1_0 \Big[\m\nabla^2 f \big( X_n +  L^{\m (\pm)}_n \big) - \nabla^2 f(X_n) \Big]  L^{\m (\pm)}_n dt, \\[5pt]
			I^{\m (\pm)}_3 & := \int^1_0 (1-t) \Big[\m\nabla^3 f \big(X_n + t L^{\m (\pm)}_n \big) - \nabla^3 f(X_n) \Big] \big(  L^{\m (\pm)}_n \big)^{\otimes 2} dt, \\[5pt]
			\widetilde{I}_1 \, & := \int^1_0 \nabla^2 f \big( X_n +  L^{\m (+)}_n + t P_n \big) P_n\m dt \\[3pt]
						& = \Big[\m\nabla f \big( X_n +  L^{\m (+)}_n + P_n \big) - \nabla f \big( X_n +  L^{\m (+)}_n \big) \Big], \\[5pt]
			\widetilde{I}_2 \, & := \int^1_0 (1-t) \nabla^3 f \big( X_n +  L^{\m (+)}_n + t P_n \big) ( P_n )^{\otimes 2} dt \\[3pt]
						& = \int^1_0 \Big[\m\nabla^2 f \big( X_n +  L^{\m (+)}_n + t P_n \big) - \nabla^2 f \big( X_n +  L^{\m (+)}_n \big) \Big] P_n dt, \\[5pt]
			\widetilde{I}_3 \, & := \int^1_0 \nabla^3 f \big( X_n + t L^{\m (+)}_n \big) \big( L^{\m (+)}_n\m, P_n \big) dt \\[3pt]
								& = \Big[\m\nabla^2 f \big( X_n +  L^{\m (+)}_n \big) - \nabla^2 f( X_n ) \Big] P_n\m,
		\end{split}
	\end{align}
	where
	\begin{align*}	
	L^{\m (\pm)}_n & := \phi_1 (\lambda_\pm) \big( V_n + \sigma ( H_n + 6K_n ) \big) + \phi_2 (\lambda_\pm) C_n\m,\\[3pt]	
	P_n & := -u h_n \phi_1\bigg(\frac{1}{3}\bigg) \nabla f\big(X^{(1)}_n\big) = -u h_n\phi_1\bigg(\frac{1}{3}\bigg) \Big(\nabla f(X_n) + I^{\m (-)}_1 \Big). 
	\end{align*}
\end{theorem}
\begin{proof}
	We can rewrite the QUICSORT method \eqref{eq:QUICSORT} as
	\begin{align*}
			X_{n+1} &= X_n + \phi_1(1) \big( V_n + \sigma (H_n + 6 K_n) \big) + \phi_2(1) C_n \\
					&\mmm - \frac{1}{2} u h_n \Big( \phi_1 (\lambda_+) \nabla f\big(X^{(1)}_n\big) + \phi_1 (\lambda_-) \nabla f\big(X^{(2)}_n\big) \Big),  \\[3pt]
			V_{n+1} &= \phi_0 (1) \big(V_n + \sigma (H_n + 6 K_n) \big) - \sigma (H_n - 6 K_n) + \phi_1 (1) C_n \\
					&\mmm - \frac{1}{2} u h_n \Big( \phi_0 (\lambda_+) \nabla f\big(X^{(1)}_n\big) + \phi_0 (\lambda_-) \nabla f\big(X^{(2)}_n\big) \Big),
	\end{align*}
	where
	\begin{align*}
		X^{(1)}_n & = X_n + L^{\m (-)}_n,\hspace{10mm}
		X^{(2)}_n = X_n + L^{\m (+)}_n + P_n\m.
	\end{align*}
	Using Taylor expansions with an integral remainder, we can derive the identities in \eqref{eq:QUICSORT_expansion_integrals},
	\begin{align*}
	\begin{split}
		\nabla f \big(X^{(1)}_n\big) = \nabla f(X_n) & + I^{\m (-)}_1 \\[1pt]
						= \nabla f(X_n) & + \nabla^2 f(X_n) L^{\m (-)}_n + I^{\m (-)}_2 \\[1pt]
						= \nabla f(X_n) & + \nabla^2 f(X_n) L^{\m (-)}_n + \frac{1}{2} \nabla^3 f(X_n) \big( L^{\m (-)}_n,  L^{\m (-)}_n \big) + I^{\m (-)}_3, \\[3pt]
		\nabla f\big(X^{(2)}_n\big) = \nabla f(X_n) & + I^{\m (+)}_1 + \widetilde{I}_1 \\
						= \nabla f(X_n) & + \nabla^2 f(X_n) L^{\m (+)}_n + I^{\m (+)}_2 + \widetilde{I}_1 \\[1pt]
						= \nabla f(X_n) & + \nabla^2 f(X_n) L^{\m (+)}_n + \nabla^2 f(X_n) P_n + \frac{1}{2} \nabla^3 f(X_n)\big( L^{\m (+)}_n, L^{\m (+)}_n \big) \\[1pt]
						& + I^{\m (+)}_3 +  \widetilde{I}_2 +  \widetilde{I}_3\m.
	\end{split}
	\end{align*}
	The result now follows by substituting the expansions for $\nabla f\big(X^{(1)}_n\big)$ and $\nabla f\big(X^{(2)}_n\big)$ into the QUICSORT method, and Taylor expanding the $\phi$ functions \eqref{eq:phi_functions} up to the desired order. 
\end{proof}

\begin{theorem}[Remainder bounds for the QUICSORT method]
	Let $\{ (X_n, V_n) \}_{n \geq 0}$ denote the QUICSORT method \eqref{eq:QUICSORT} and let $\{(x_t, v_t)\}_{t \geq 0}$ denote the underdamped Langevin diffusion \eqref{eq:ULD} where we assume that $(x_0, v_0) \sim \pi$, the stationary distribution of the process. Suppose that at time $t_n\m$, we have $(X_n\m, V_n) = (x_{t_n}\m, v_{t_n})$. Let $0 < h_n \leq 1$ be fixed, and consider the remainder terms \eqref{eq:QUICSORT_position_remainder_expansion_1} and \eqref{eq:QUICSORT_momentum_remainder_expansion_1} defined over the time interval $[t_n\m, t_{n+1}]$. Under assumptions \eqref{eq:assumption_a1} and \eqref{eq:assumption_a2} we can bound their $\L_2$ norms and absolute means as\label{thm:QUICSORT_remainder_bounds}
	\begin{align}
		\big\| R^{\m (1)}_X \big\|_{\L_2} & \leq C^{\m (1)}_{X}  \sqrt{d} \m h_n^2\m, \mmm\,\,\,\,\,\, \big\| R^{\m (1)}_V \big\|_{\L_2} \leq C^{\m (1)}_{V}   \sqrt{d} \m h_n^2\m,	\label{eq:QUICSORT_remainder_strong_bounds_1} \\[5pt]
		\big\| \E \big[ R^{\m (1)}_X \big] \big\|_2 &\leq C^{\m (1)}_{X}  \sqrt{d} \m h_n^2\m,\mmm \big\| \E \big[ R^{\m (1)}_V \big] \big\|_2 \leq C^{\m (1)}_{V} \sqrt{d} \m h_n^2\m,	\label{eq:QUICSORT_remainder_weak_bounds_1}
	\end{align}
	where $C^{\m (1)}_{X}, C^{\m (1)}_{V} > 0$ are constants depending on $\gamma$, $u$, $m$ and $M_1\m$.\medbreak
	
	\noindent
	Under the additional assumption \eqref{eq:assumption_a3}, we can bound \eqref{eq:QUICSORT_position_remainder_expansion_2} and \eqref{eq:QUICSORT_momentum_remainder_expansion_2} as
	\begin{align}
		\big\| R^{\m (2)}_X \big\|_{\L_2} & \leq C^{\m (2)}_{X} \sqrt{d} \m h_n^3\m,  \mmm\,\,\,\,\,\, \big\| R^{\m (2)}_V \big\|_{\L_2} \leq C^{\m (2)}_{V} d h_n^3\m, 	\label{eq:QUICSORT_remainder_strong_bounds_2}\\[5pt]
		\big\| \E \big[ R^{\m (2)}_X \big] \big\|_2 & \leq C^{\m (2)}_{X}  \sqrt{d} \m h_n^3\m, \mmm\big\| \E \big[ R^{\m (2)}_V \big] \big\|_2 \leq C^{\m (2)}_{V} d h_n^3\m,	\label{eq:QUICSORT_remainder_weak_bounds_2}
	\end{align}
	where $C^{\m (2)}_{X}, C^{\m (2)}_{V} > 0$ are constants depending on $\gamma$, $u$, $m$, $M_1$ and $M_2\m$. Finally, under the additional assumption \eqref{eq:assumption_a4}, then the remainder terms  \eqref{eq:QUICSORT_position_remainder_expansion_3} and \eqref{eq:QUICSORT_momentum_remainder_expansion_3} can be bounded by
	\begin{align}
		\big\| R^{\m (3)}_X \big\|_{\L_2} \leq C^{\m (3)}_{X} dh_n^4\m, \mmm\,\,\,\,\,\,  \big\| R^{\m (3)}_V \big\|_{\L_2} \leq C^{\m (3)}_{V}  d^{1.5}\m  h_n^4\m, 	\label{eq:QUICSORT_remainder_strong_bounds_3} \\[5pt]
		\big\| \E \big[ R^{\m (3)}_X \big] \big\|_2 \leq C^{\m (3)}_{X} dh_n^4\m, \mmm \big\| \E \big[ R^{\m (3)}_V \big] \big\|_2 \leq C^{\m (3)}_{V} d^{1.5} \m h_n^4\m,	\label{eq:QUICSORT_remainder_weak_bounds_3}
	\end{align}
	where $C^{\m (3)}_{X}, C^{\m (3)}_{V} > 0$ are constants depending on $\gamma$, $u$, $m$, $M_1$, $M_2$ and $M_3\m$.
\end{theorem}
\begin{proof}
	We can bound the integrals given in \eqref{eq:QUICSORT_expansion_integrals} using the Cauchy-Schwarz and Minkowski inequalities along with assumptions \eqref{eq:assumption_a1}, \eqref{eq:assumption_a2}, \eqref{eq:assumption_a3} and \eqref{eq:assumption_a4}.
	\begin{align*}
			\big\| I^{\m (\pm)}_1 \big\|_{\L_2} &\leq M_1 \m \big\| L^{\m (\pm)}_n \big\|_{\L_2} \\[3pt]
										&\leq M_1 \m \vert \phi_1 (\lambda_\pm) \vert \| V_n \|_{\L_2} + \sigma M_1 \m \vert \phi_1 (\lambda_\pm) \vert \| H_n + 6K_n \|_{\L_2} \\
										&\hspace{7.5mm} + \sigma M_1 \m \Big\vert \frac{1}{h_n} \phi_2 (\lambda_{\pm}) \Big\vert \| W_n - 12 K_n \|_{\L_2}\m, \\[3pt]
			\big\| I^{\m (\pm)}_2 \big\|_{\L_2} &\leq M_2 \m \big\| L^{\m (\pm)}_n \big\|^2_{\L_4} \\[3pt]
									&\leq M_2 \m \vert \phi_1 (\lambda_\pm) \vert^2 \| V_n \|^2_{\L_4} + \sigma^2 M_2 \m \vert \phi_1 (\lambda_\pm) \vert^2 \| H_n + 6K_n \|^2_{\L_4} \\
									&\hspace{7.5mm} + \sigma^2  M_2 \m\Big\vert \frac{1}{h_n} \phi_2 (\lambda_{\pm}) \Big\vert^2 \| W_n - 12 K_n \|^2_{\L_4} \\
									&\hspace{7.5mm} + 2 \sigma M_2 \m \big\vert \phi_1 (\lambda_\pm) \big\vert^2 \| V_n \|_{\L_4} \| H_n + 6K_n \|_{\L_4} \\
									&\hspace{7.5mm} + 2 \sigma M_2 \m \Big\vert \frac{1}{h_n} \phi_1 (\lambda_\pm) \phi_2 (\lambda_{\pm}) \Big\vert \| V_n \|_{\L_4} \| W_n - 12 K_n \|_{\L_4} \\
									&\hspace{7.5mm} + 2 \sigma^2 M_2 \m \Big\vert \frac{1}{h_n} \phi_1 (\lambda_\pm) \phi_2 (\lambda_{\pm}) \Big\vert \| H_n + 6K_n \|_{\L_4} \| W_n - 12 K_n \|_{\L_4}, \\[3pt]
			\big\| I^{\m (\pm)}_3 \big\|_{\L_2} &\leq M_3 \m \big\| L^{\m (\pm)}_n \big\|^3_{\L_6} \\[3pt]
									&\leq M_3 \m \big\vert \phi_1 (\lambda_\pm) \big\vert^3 \| V_n \|^3_{\L_6} + \sigma^3 M_3 \m \vert \phi_1 (\lambda_\pm) \vert^3 \| H_n + 6K_n \|^3_{\L_6} \\
									&\hspace{7.5mm} + \sigma^3 M_3 \m \Big\vert \frac{1}{h_n} \phi_2 (\lambda_{\pm}) \Big\vert^3 \| W_n - 12 K_n \|^3_{\L_6} \\
									&\hspace{7.5mm} + 3 \sigma M_3 \m \vert \phi_1 (\lambda_\pm) \vert^3 \| V_n \|^2_{\L_6} \| H_n + 6K_n \|_{\L_6} \\[3pt]
									&\hspace{7.5mm} + 3 \sigma^2 M_3 \m \vert  \phi_1 (\lambda_\pm) \vert^3 \| V_n \|_{\L_6} \| H_n + 6K_n \|^2_{\L_6} \\
									&\hspace{7.5mm} + 3 \sigma M_3 \m \Big\vert \frac{1}{h_n} \phi^2_1 (\lambda_\pm) \phi_2 (\lambda_{\pm}) \Big\vert  \| V_n \|^2_{\L_6} \| W_n - 12 K_n \|_{\L_6} \\
									&\hspace{7.5mm} + 3 \sigma^3 M_3 \m \Big\vert \frac{1}{h_n} \phi^2_1 (\lambda_\pm) \phi_2 (\lambda_{\pm}) \Big\vert \| H_n + 6K_n \|^2_{\L_6} \| W_n - 12 K_n \|_{\L_6} \\
									&\hspace{7.5mm} + 3 \sigma^2 M_3 \m \Big\vert \frac{1}{h^2_n}  \phi_1 (\lambda_\pm) \phi^2_2 (\lambda_{\pm}) \Big\vert  \| V_n \|_{\L_6} \| W_n - 12 K_n \|^2_{\L_6} \\
									&\hspace{7.5mm} + 3 \sigma^3 M_3 \m \Big\vert \frac{1}{h^2_n}  \phi_1 (\lambda_\pm) \phi^2_2 (\lambda_{\pm}) \Big\vert  \| H_n + 6K_n \|_{\L_6} \| W_n - 12 K_n \|^2_{\L_6} \\
									&\hspace{7.5mm} + 6 \sigma^2 M_3 \m \Big\vert \frac{1}{h_n}  \phi^2_1 (\lambda_\pm) \phi_2 (\lambda_{\pm}) \Big\vert \| V_n \|_{\L_6}  \| H_n + 6K_n \|_{\L_6} \| W_n - 12 K_n \|_{\L_6}, \\[3pt]
			\big\| \widetilde{I}_1 \big\|_{\L_2} &\leq M_1 \m \| P_n \|_{\L_2} \\[3pt]
									&\leq u M_1\m h_n \Big\vert \phi_1 \Big(\frac{1}{3}\Big) \Big\vert \| \nabla f(X_n) \|_{\L_2} + u M^2_1\m h_n \Big\vert \phi_1 \Big(\frac{1}{3}\Big)  \phi_1 (\lambda_-) \Big\vert \| V_n \|_{\L_2} \\[1pt]
									&\hspace{7.5mm} + \sigma u M^2_1\m h_n \Big\vert \phi_1 \Big(\frac{1}{3}\Big)  \phi_1 (\lambda_-) \Big\vert \| H_n + 6K_n \|_{\L_2} \\[1pt]
									&\hspace{7.5mm} + \sigma u M^2_1\m \Big\vert \phi_1 \Big(\frac{1}{3}\Big) \phi_2 (\lambda_{-}) \Big\vert \| W_n - 12 K_n \|_{\L_2}\m, \nonumber \\[4pt]
			\| \widetilde{I}_2 \|_{\L_2} &\leq M_2 \m \| P_n \|^2_{\L_4} \\[3pt]
									&\leq u^2 M_2\m h^2_n \Big\vert \phi_1 \Big(\frac{1}{3}\Big) \Big\vert^2 \big\| \nabla f(X_n) \big\|^2_{\L_4} \\[1pt]
									&\hspace{7.5mm} + 2 u^2 M_1 \m M_2\m h^2_n \Big\vert \phi^2_1 \Big(\frac{1}{3}\Big) \phi_1 (\lambda_-) \Big\vert \| \nabla f(X_n) \|_{\L_4} \| V_n \|_{\L_4} \\[1pt]
									&\hspace{7.5mm} + 2 \sigma u^2 M_1 \m M_2\m h^2_n \Big\vert \phi^2_1 \Big(\frac{1}{3}\Big) \phi_1 (\lambda_-) \Big\vert \| \nabla f(X_n) \|_{\L_4} \| H_n + 6K_n \|_{\L_4} \\[1pt]
									&\hspace{7.5mm} + 2 \sigma u^2 M_1 \m M_2\m h_n \Big\vert \phi^2_1 \Big(\frac{1}{3}\Big) \phi_2 (\lambda_-) \Big\vert \| \nabla f(X_n) \|_{\L_4} \| W_n - 12K_n \|_{\L_4} \\[1pt]
									&\hspace{7.5mm} + u^2 M^2_1 \m M_2 \m h^2_n \Big\vert \phi^2_1 \Big(\frac{1}{3}\Big) \phi^2_1 (\lambda_-) \Big\vert \| V_n \|^2_{\L_4} \\[1pt]
									&\hspace{7.5mm} + \sigma^2 u^2 M^2_1 \m M_2 \m h^2_n \Big\vert \phi^2_1 \Big(\frac{1}{3}\Big) \phi^2_1 (\lambda_-) \Big\vert \| H_n + 6K_n \|^2_{\L_4} \\[1pt]
									&\hspace{7.5mm} + \sigma^2 u^2 M^2_1 \m M_2 \m  \Big\vert \phi^2_1 \Big(\frac{1}{3}\Big) \phi^2_2 (\lambda_{-}) \Big\vert \| W_n - 12 K_n \|^2_{\L_4} \\[1pt]
									&\hspace{7.5mm} + 2 \sigma u^2 M^2_1 \m M_2 \m h^2_n  \Big\vert \phi^2_1 \Big(\frac{1}{3}\Big) \phi^2_1 (\lambda_-) \Big\vert \| V_n \|_{\L_4} \| H_n + 6K_n \|_{\L_4} \\[1pt]
									&\hspace{7.5mm} + 2 \sigma u^2 M^2_1 \m M_2 \m h_n \Big\vert \phi^2_1 \Big(\frac{1}{3}\Big)  \phi_1 (\lambda_-) \phi_2 (\lambda_{-}) \Big\vert  \| V_n \|_{\L_4} \| W_n - 12 K_n \|_{\L_4} \\[1pt]
									&\hspace{7.5mm} + 2 \sigma^2 u^2 M^2_1 \m M_2 \m h_n \Big\vert \phi^2_1 \Big(\frac{1}{3}\Big) \phi_1 (\lambda_-) \phi_2 (\lambda_{-}) \Big\vert \| H_n + 6K_n \|_{\L_4} \| W_n - 12 K_n \|_{\L_4}\m, \\[4pt]
			\big\| \widetilde{I}_3 \big\|_{\L_2} &\leq M_2 \m \big\| L^{\m (+)}_n \big\|_{\L_4} \| P_n \|_{\L_4} \\[3pt]
									&\leq u M_2 \m h_n  \Big\vert \phi_1 \Big(\frac{1}{3}\Big)  \phi_1 (\lambda_+) \Big\vert \big\| \nabla f(X_n) \big\|_{\L_4} \| V_n \|_{\L_4} \\[1pt]
									&\hspace{7.5mm} +  \sigma u M_2 \m h_n  \Big\vert \phi_1 (\frac{1}{3}) \phi_1 (\lambda_+)\Big\vert \big\| \nabla f(X_n) \big\|_{\L_4}  \| H_n + 6K_n \|_{\L_4} \\[1pt]
									&\hspace{7.5mm} + \sigma u M_2 \m  \Big\vert \phi_1 \Big(\frac{1}{3}\Big) \phi_2 (\lambda_{+}) \Big\vert \big\| \nabla f(X_n) \big\|_{\L_4} \| W_n - 12 K_n \|_{\L_4} \\[1pt]
									&\hspace{7.5mm} + u M_1 \m M_2 \m h_n \vert \phi_1 \Big(\frac{1}{3}\Big) \phi_1 (\lambda_+) \phi_1 (\lambda_-) \vert \| V_n \|^2_{\L_4} \\[1pt]
									&\hspace{7.5mm} + \sigma^2 u M_1 \m M_2 \m h_n \vert \phi_1 \Big(\frac{1}{3}\Big) \phi_1 (\lambda_+) \phi_1 (\lambda_-) \vert \| H_n + 6K_n \|^2_{\L_4} \\[1pt]
									&\hspace{7.5mm} + \sigma^2 u M_1 \m M_2 \m \Big\vert \frac{1}{h_n} \phi_1 \Big(\frac{1}{3}\Big) \phi_2 (\lambda_+) \phi_2 (\lambda_-) \Big\vert \| W_n - 12 K_n \|^2_{\L_4} \\[1pt]
									&\hspace{7.5mm} + 2 \sigma u M_1 \m M_2 \m h_n \Big\vert \phi_1 \Big(\frac{1}{3}\Big) \phi_1 (\lambda_+) \phi_1 (\lambda_-) \Big\vert \| V_n \|_{\L_4} \| H_n + 6K_n \|_{\L_4} \\[2pt]
									&\hspace{7.5mm} + \sigma u M_1 \m M_2 \m \bigg( \Big\vert \phi_1 \Big(\frac{1}{3}\Big) \phi_1 (\lambda_+) \phi_2 (\lambda_-) \Big\vert + \Big\vert \phi_1 \Big(\frac{1}{3}\Big) \phi_1 (\lambda_-) \phi_2 (\lambda_+) \Big\vert \bigg) \\[3pt]
									&\hspace{7.5mm} \mm \times \Big[\| V_n \|_{\L_4} \| W_n - 12K_n \|_{\L_4} +  \sigma \| H_n + 6K_n \|_{\L_4} \| W_n - 12K_n \|_{\L_4} \Big]\m. 
	\end{align*}
We can derive the remainder bounds \eqref{eq:QUICSORT_remainder_strong_bounds_1}, \eqref{eq:QUICSORT_remainder_strong_bounds_2} and \eqref{eq:QUICSORT_remainder_strong_bounds_3} for the position variable using the above inequalities along with equations \eqref{eq:nabla_f_bounds}, \eqref{eq:v_bounds}  and \eqref{eq:levy_area_bounds}. In addition, we will use Theorem \ref{thm:phi_function_bounds} and Taylor expand the $\phi_i$ functions using the Lagrange form of the remainder.

\begin{align*}
		\big\| R^{\m (1)}_X \big\|_{\L_2} &\leq \frac{1}{2} \gamma h^2_n \| V_n \|_{\L_2} + \frac{1}{2} \gamma \sigma h^2_n \| H_n + 6K_n  \|_{\L_2} + \frac{1}{6} \gamma \sigma h^2_n \| W_n - 12 K_n  \|_{\L_2}  \\[1pt]
								&\mmm + \frac{1}{2} u h^2_n \big\| \nabla f(X_n) \big\|_{\L_2} + \frac{1}{2} u h_n \Big( \vert \phi_1 (\lambda_+) \big\vert \| I^{\m (-)}_1 \big\|_{\L_2} + \vert \phi_1 (\lambda_-) \vert \big\| I^{\m (+)}_1 \big\|_{\L_2}\Big)  \\[2pt]
								&\mmm + \frac{1}{2} u h_n \vert \phi_1 (\lambda_-) \vert \big\| \widetilde{I}_1 \big\|_{\L_2}\\[3pt]
								&\leq \bigg[ \frac{1}{2} \gamma \sqrt{u} + \frac{1}{2} u \sqrt{ M_1} + \frac{2 + \sqrt{3} + \sqrt{5}}{12} \gamma \sigma \sqrt{h_n} + \frac{1}{6} u^{1.5} \m M_1 \m h_n \\
								&\mmm + \frac{15 + 10\sqrt{3} + 9\sqrt{5}}{360} \sigma u M_1 \m h^{1.5}_n + \frac{3 - \sqrt{3}}{36} u^2 M^{1.5}_1 \m h^2_n \\[3pt]
								&\mmm + \frac{2 - \sqrt{3}}{36} u^{2.5} M^2_1 \m h^3_n + \frac{15 - 5\sqrt{3} + 21 \sqrt{5} - 11\sqrt{15}}{2160} \sigma u^2 M^2_1 \m h^{3.5}_n \bigg] \sqrt{d} \m h^2_n \\[3pt]
								&\leq \Big[ \gamma \sqrt{u} + \gamma \sigma + u\sqrt{M_1} + u^{1.5} M_1 + \sigma u M_1 + u^2 M^{1.5}_1 + u^{2.5} M^2_1 + \sigma u^2 M^2_1 \Big] \sqrt{d} h^2_n\m, \\[7pt]
		\big\| R^{\m (2)}_X \big\|_{\L_2} &\leq \frac{1}{6} \gamma^2 h^3_n \| V_n \|_{\L_2} + \frac{1}{6} \gamma^2 \sigma h^3_n \| H_n + 6K_n  \|_{\L_2} + \frac{1}{24} \gamma^2 \sigma h^3_n \| W_n - 12 K_n  \|_{\L_2}  \\[1pt]
								&\mmm + \frac{1}{6} \gamma u h^3_n \big\| \nabla f(X_n) \big\|_{\L_2} + \frac{1}{2} u h_n \Big( \vert \phi_1 (\lambda_+) \vert \big\| I^{\m (-)}_1 \big\|_{\L_2} + \vert \phi_1 (\lambda_-) \vert \big\| I^{\m (+)}_1 \big\|_{\L_2} \Big)  \\[2pt]
								&\mmm + \frac{1}{2} u h_n \vert \phi_1 (\lambda_-) \vert \big\| \widetilde{I}_1 \big\|_{\L_2} \\[3pt]
								&\leq \bigg[ \frac{1}{6} \gamma^2 \sqrt{u} + \frac{1}{6} \gamma u \sqrt{M_1} + \frac{1}{6} u^{1.5} M_1 +  \frac{15 + 10\sqrt{3} + 9\sqrt{5}}{360} \gamma^2 \sigma \sqrt{h_n}  \\
								&\mmm + \frac{15 + 5 \sqrt{3} + 6 \sqrt{5}}{180} \sigma u M_1 \sqrt{h_n} + \frac{3 - \sqrt{3}}{36} u^{2} M^{1.5}_1 h_n \\[2pt]
								&\mmm + \frac{2 - \sqrt{3}}{36} u^{2.5} M^2_1 h^2_n + \frac{15 - 5\sqrt{3} + 21\sqrt{5} - 11\sqrt{15}}{2160} \sigma u^2 M^2_1 h^{2.5}_n \bigg] \sqrt{d} \m h^3_n \\[3pt]
								&\leq \Big[ \gamma^2 \sqrt{u} + \gamma^2 \sigma + \gamma u \sqrt{M_1} + u^{1.5} M_1 + \sigma u M_1 + u^{2} M^{1.5}_1 + u^{2.5} M^2_1 + \sigma u^2 M^2_1 \Big] \sqrt{d} \m h^3_n\m, \\[7pt]
		\big\| R^{\m (3)}_X \big\|_{\L_2} &\leq  \frac{1}{24} \gamma^3 h^4_n \| V_n \|_{\L_2} + \frac{1}{24} \gamma^3 \sigma h^4_n \| H_n + 6K_n  \|_{\L_2} + \frac{1}{120} \gamma^3 \sigma h^4_n \| W_n - 12 K_n  \|_{\L_2} \\[2pt]
								&\mmm + \frac{1}{24} \gamma^2 u h^4_n \big\| \nabla f(X_n) \big\|_{\L_2} + \frac{1}{12} \gamma u h^4_n \big\| \nabla^2 f(X_n) V_n  \big\|_{\L_2} \\[2pt]
								&\mmm +  \frac{1}{12} \gamma \sigma u h^4_n \big\| \nabla^2 f(X_n) ( H_n + 6K_n) \big\|_{\L_2} \\[2pt]
								&\mmm + \frac{7}{432} \gamma \sigma u h^4_n \big\| \nabla^2 f(X_n)( W_n - 12K_n )\big\|_{\L_2} \\[2pt]
								&\mmm  + \frac{1}{2} u h_n \Big( \vert \phi_1 (\lambda_+) \vert \big\| I^{\m (-)}_2 \big\|_{\L_2} + \vert \phi_1 (\lambda_-) \vert \big\| I^{\m (+)}_2 \big\|_{\L_2} \Big) + \frac{1}{2} u h_n \vert \phi_1 (\lambda_-) \big\| \widetilde{I}_1 \big\|_{\L_2} \\[3pt]
								&\leq \bigg[ \frac{1}{24} \gamma^3 \sqrt{u} + \frac{1}{24} \gamma^2 u \sqrt{M_1} + \frac{1}{12} \gamma u^{1.5} M_1 + \frac{\sqrt{3}}{24} u^2 M_2 \\[1pt]
								&\mmm + \frac{30 + 25\sqrt{3} + 21 \sqrt{5}}{3600} \gamma^3 \sigma \sqrt{h_n} + \frac{7 + 6\sqrt{3} + 5 \sqrt{5}}{432} \gamma u M_1 \sqrt{h_n} \\
								&\mmm + \frac{3 + 2\sqrt{3} + \sqrt{15}}{72} \sigma u^{1.5} M_2 \sqrt{h_n} + \frac{20 + 21\sqrt{3} + 10 \sqrt{5} + 7\sqrt{15}}{1440} \sigma^2 u M_2 \\[2pt]
								&\mmm + \frac{3 - \sqrt{3}}{36} u^{2} M^{1.5}_1 + \frac{2 - \sqrt{3}}{36} u^{2.5} M^2_1 h_n \\[2pt]
								&\mmm + \frac{15 - 5\sqrt{3} + 21\sqrt{5} - 11\sqrt{15}}{2160} \sigma u^2 M^2_1 h^{1.5}_n  \bigg] d h^4_n \\[3pt]
								&\leq \Big[ \gamma^3 \sqrt{u} + \gamma^3 \sigma + \gamma^2 u \sqrt{M_1} + \gamma u^{1.5} M_1 + \gamma u M_1 + u^2 M_2 + \sigma u^{1.5} M_2 \\
								&\mmm + \sigma^2 u M_2 + u^{2} M^{1.5}_1 + u^{2.5} M^2_1 + \sigma u^2 M^2_1 \Big] d h^4_n\m.
\end{align*}
Similarly for the momentum variable,
\begin{align*}
		\big\| R^{\m (1)}_V \big\|_{\L_2} &\leq \frac{1}{2} \gamma^2 h^2_n \| V_n \|_{\L_2} + \frac{1}{2} \gamma^2 \sigma h^2_n \| H_n + 6K_n  \|_{\L_2} + \frac{1}{6} \gamma^2 \sigma h^2_n \| W_n - 12 K_n  \|_{\L_2} \\
								&\mmm + \frac{1}{2} \gamma u h^2 \big\| \nabla f(X_n) \big\|_{\L_2} + \frac{1}{2} u h_n \Big( \vert \phi_0 (\lambda_+) \vert \big\| I^{\m (-)}_1 \big\|_{\L_2} + \vert \phi_0 (\lambda_-) \vert \big\| I^{\m (+)}_1 \big\|_{\L_2} \Big) \\
								&\mmm + \frac{1}{2} u h_n \vert \phi_0 (\lambda_-) \vert \big\| \widetilde{I}_1 \big\|_{\L_2} \\
								&\leq \bigg[ \frac{1}{2} \gamma^2 \sqrt{u} + \frac{1}{2} \gamma u \sqrt{M_1} + \frac{1}{2} u^{1.5} M_1 + \frac{2 + \sqrt{3} + \sqrt{5}}{12} \gamma^2 \sigma \sqrt{h_n} \\
								&\mmm + \frac{2 + \sqrt{3} + \sqrt{5}}{12} \sigma u M_1 \m \sqrt{h_n} + \frac{1}{6} u^{2} M^{1.5}_1 \m h_n + \frac{3 - \sqrt{3}}{36} u^{2.5} M^2_1 \m h^2_n \\[2pt]
								&\mmm + \frac{5 + 5\sqrt{5} - 2 \sqrt{15}}{360} \sigma u^2 M^2_1 \m h^{2.5}_n \bigg] \sqrt{d} \m h^2_n \\
								&\leq \Big[ \gamma^2 \sqrt{u} + \gamma^2 \sigma + \gamma u \sqrt{M_1} + u^{1.5} M_1 + \sigma u M_1 + u^{2} M^{1.5}_1 + u^{2.5} M^2_1 + \sigma u^2 M^2_1 \Big]  \sqrt{d} \m h^2_n\m, \\[6pt]
		\big\| R^{\m (2)}_V \big\|_{\L_2} &\leq \frac{1}{6} \gamma^3 h^3_n \| V_n \|_{\L_2} + \frac{1}{6} \gamma^3 \sigma h^3_n \| H_n + 6K_n \|_{\L_2} + \frac{1}{24} \gamma^3 \sigma h^3_n \| W_n - 12 K_n \|_{\L_2} \\[1pt]
								&\mmm + \frac{1}{6} \gamma^2 u h^3_n \big\| \nabla f(X_n) \big\|_{\L_2} + \frac{1}{3} \gamma u h^3_n \big\| \nabla^2 f(X_n) V_n \big\|_{\L_2} \\[2pt]
								&\mmm + \frac{1}{3} \gamma \sigma u h^3_n \big\| \nabla^2 f(X_n) ( H_n + 6K_n) \big\|_{\L_2}+ \frac{1}{12} \gamma \sigma u h^3_n \big\| \nabla^2 f(X_n) ( W_n - 12 K_n ) \big\|_{\L_2} \\[2pt]
								&\mmm + \frac{1}{2} u h_n \Big( \vert \phi_0 (\lambda_+) \vert \big\| I^{\m (-)}_2 \big\|_{\L_2} + \vert \phi_0 (\lambda_-) \vert \big\| I^{\m (+)}_2 \big\|_{\L_2} \Big) + \frac{1}{2} u h_n \vert \phi_0 (\lambda_-) \vert \big\| \widetilde{I}_1 \big\|_{\L_2} \\[2pt]
								&\leq \bigg[ \frac{1}{6} \gamma^3 \sqrt{u} + \frac{1}{6} \gamma^2 u \sqrt{M_1} + \frac{1}{3} \gamma u^{1.5} M_1 + \frac{\sqrt{3}}{3} u^2 M_2 + \frac{1}{6} u^{2} M^{1.5}_1  \\[1pt]
								&\mmm + \frac{15 + 10\sqrt{3} + 9\sqrt{5}}{360} \gamma^3 \sigma \sqrt{h_n} + \frac{15 + 10\sqrt{3} + 9\sqrt{5}}{180} \gamma \sigma u M_1 \sqrt{h_n} \\[2pt]
								&\mmm + \frac{20 + 15\sqrt{3} + 7\sqrt{15}}{60} \sigma u^{1.5} M_2 \m \sqrt{h_n} + \frac{3 - \sqrt{3}}{36} u^{2.5} M^2_1 h_n \\[2pt]
								&\mmm + \frac{45 + 97\sqrt{3} + 21\sqrt{5} + 33\sqrt{15}}{360} \sigma^2 u M_2 h_n \\[2pt]
								&\mmm + \frac{5 + 5\sqrt{5} - 2\sqrt{15}}{360} \sigma u^2 M^2_1 \m h^{1.5}_n \bigg] d h^3_n\\[2pt]
								&\leq \Big[ \gamma^3 \sqrt{u} + \gamma^3 \sigma + \gamma^2 u \sqrt{M_1} + \gamma u^{1.5} M_1 + \gamma \sigma u M_1 + u^2 M_2 + \sigma u^{1.5} M_2 \\
								&\mmm + \sigma^2 u M_2 + u^{2} M^{1.5}_1 + u^{2.5} M^2_1 + \sigma u^2 M^2_1\m\Big] d h^3_n\m,\\[5pt]
	\big\| R^{\m (3)}_V \big\|_{\L_2} &\leq \frac{1}{24} \gamma^4 h^4_n \| V_n \|_{\L_2} + \frac{1}{24} \gamma^4 \sigma h^4_n \| H_n + 6K_n \|_{\L_2} + \frac{1}{120} \gamma^4 \sigma h^4_n \| W_n - 12 K_n \|_{\L_2} \\[1pt]
								&\mmm + \frac{1}{24} \gamma^3 u h^4_n \big\| \nabla f(X_n) \big\|_{\L_2} + \frac{1}{8} \gamma^2 u h^4_n \big\| \nabla^2 f(X_n) V_n \big\|_{\L_2} \\[2pt]
								&\mmm + \frac{1}{8} \gamma^2 \sigma u h^4_n \big\| \nabla^2 f(X_n) (H_n + 6K_n ) \big\|_{\L_2} \\[2pt]
								&\mmm + \frac{7}{288} \gamma^2 \sigma u h^4_n \big\| \nabla^2 f(X_n) ( W_n - 12 K_n ) \big\|_{\L_2} \\[2pt]
								&\mmm + \frac{4 - \sqrt{3}}{36} \gamma u^2 h^5_n \big\| \nabla^2 f(X_n) \nabla f(X_n) \big\|_{\L_2} + \frac{1}{6} \gamma u h^4_n \big\| \nabla^3 f(X_n) V^{\otimes 2}_n  \big\|_{\L_2} \\[2pt]
								&\mmm + \frac{1}{3} \gamma \sigma u h^4_n \big\| \nabla^3 f(X_n) (V_n,  H_n + 6K_n)  \big\|_{\L_2} \\[2pt]
								&\mmm + \frac{47}{432} \gamma \sigma u h^4_n \big\| \nabla^3 f(X_n) (V_n, W_n - 12K_n)  \big\|_{\L_2} \\[2pt]
								&\mmm + \frac{1}{6} \sigma^2 u h^3_n \big\| \nabla^3 f(X_n) (H_n + 6K_n)^{\otimes 2} \big\|_{\L_2} \\[2pt]
								&\mmm + \frac{7}{288} \sigma^2 u h^3_n \big\| \nabla^3 f(X_n) (W_n - 12K_n)^{\otimes 2} \big\|_{\L_2} \\[2pt]
								&\mmm + \frac{1}{8} \sigma^2 u h^3_n \big\| \nabla^3 f(X_n) (H_n + 6K_n, W_n - 12K_n) \big\|_{\L_2} \\[2pt]
								&\mmm + \frac{1}{2} u^2 h^2_n \Big\vert \phi_0 (\lambda_-) \phi_1 \Big(\frac{1}{3}\Big) \Big\vert \big\| \nabla^2 f(X_n) I^{\m (-)}_1 \big\|_{\L_2} \\[2pt]
								&\mmm + \frac{1}{2} u h_n \Big[ \vert \phi_0 (\lambda_+) \vert \big\| I^{\m (-)}_3 \big\|_{\L_2} + \vert \phi_0 (\lambda_-) \vert \big\| I^{\m (+)}_3 \big\|_{\L_2} \Big] \\[2pt]
								&\mmm + \frac{1}{2} u h_n \vert \phi_0 (\lambda_-) \vert \big\| \widetilde{I}_2 \big\|_{\L_2} + \frac{1}{2} u h_n \vert \phi_0 (\lambda_-) \vert \big\| \widetilde{I}_3 \big\|_{\L_2} \\[3pt]
								&\leq \bigg[ \frac{1}{24} \gamma^4 \sqrt{u} + \frac{1}{24} \gamma^3 u \sqrt{M_1} + \frac{1}{8} \gamma^2 u^{1.5} M_1 + \frac{\sqrt{3}}{6} \gamma u^2 M_2 \\[1pt]
								&\mmm + \frac{45 + 46\sqrt{3} + 21\sqrt{5} + 16\sqrt{15}}{720} \sigma^2 u M_2 + \frac{3 - \sqrt{3}}{36} u^{2.5} M_1 + \frac{1}{36} u^3 M_1 \m M_2 \\[1pt]
								&\mmm + \frac{\sqrt{3} + 1}{12} u^{2.5} \sqrt{M_1} \m M_2 + \frac{\sqrt{15}}{24} u^{2.5} M_3 + \frac{30 + 25\sqrt{3} + 21\sqrt{5}}{3600} \gamma^4 \sigma \sqrt{h_n} \\[2pt]
								&\mmm + \frac{7 + 6\sqrt{3} + 5 \sqrt{5}}{288} \gamma^2 \sigma u M_1 \m \sqrt{h_n} + \frac{360 + 235\sqrt{3} + 119\sqrt{5}}{2160} \gamma \sigma u^{1.5} M_2 \m \sqrt{h_n} \\[3pt]
								&\mmm + \frac{5 + 5\sqrt{3} - 2\sqrt{15}}{360} \sigma u^2 M_1 \m \sqrt{h_n} + \frac{16\sqrt{3} + 9 \sqrt{5} + 7\sqrt{15}}{144} \sigma u^2 M_3 \m \sqrt{h_n}  \\[2pt]
								&\mmm + \frac{30 + 15\sqrt{3} + 6\sqrt{5} + 5\sqrt{15}}{360} \sigma u^{2} \sqrt{M_1} \m M_2 \m \sqrt{h_n} + \frac{4 - \sqrt{3}}{36} \gamma u^{2} M^{1.5}_1 \m h_n \\[1pt]
								&\mmm + \frac{\sqrt{3} - 1}{36} u^{3.5} M^{1.5}_1 \m M_2 \m h_n + \frac{\sqrt{3}}{18} u^{3} M_1 \m M_2 \m h_n \\[2pt]
								&\mmm + \frac{160 + 125\sqrt{3} + 70\sqrt{5} + 71 \sqrt{15}}{720} \sigma^2 u^{1.5} M_3 \m h_n \\[3pt]
								&\mmm + \frac{5\sqrt{3} - 6\sqrt{5} + 5 \sqrt{15}}{1080} \sigma u^{3} M^{1.5}_1 \m M_2 \m h^{1.5}_n + \frac{10 + 5\sqrt{3} + 3\sqrt{15}}{360} \sigma u^{2.5} M_1 \m M_2 \m h^{1.5}_n \\[2pt]
								&\mmm + \frac{2\sqrt{3} - 3}{216} u^2 M^2_1 \m M_2 \m h^2_n + \frac{30 + 14\sqrt{3} + 9\sqrt{5} + 4\sqrt{15}}{2160} \sigma^2 u^2 M_1 \m M_2 \m h^2_n \\[3pt]
								&\mmm + \frac{-5 + 5\sqrt{3} - 11\sqrt{5} + 7\sqrt{15}}{2160} \sigma u^{1.5} M^2_1 \m M_2 \m h^{2.5}_n \\[3pt]
								&\mmm + \frac{1125 + 835\sqrt{3} + 549\sqrt{5} + 371\sqrt{15}}{25920} \sigma^3 u M_3 \m h^{2.5}_n \\[3pt]
								&\mmm + \frac{-30 + 21\sqrt{3} - 6\sqrt{5} + 5\sqrt{15}}{12960} \sigma^2 u^3 M^2_1 \m M_2 h^3_n \bigg] d^{1.5} h^4_n \\[2pt]
								&\leq \Big[ \m\gamma^4 \sqrt{u} + \gamma^4 \sigma + \gamma^3 u \sqrt{M_1} + \gamma^2 u^{1.5} M_1 + \gamma^2 \sigma u M_1 + u^{2.5} M_1 + \sigma u^2 M_1  \\
								&\mmm + \gamma u^2 M_2 + \gamma \sigma u^{1.5} M_2 + \sigma^2 u M_2 + u^{2.5} M_3 + \sigma u^2 M_3 + \sigma^2 u^{1.5} M_3 + \sigma^3 u M_3 \\[1pt]
								&\mmm + \gamma u^{2} M^{1.5}_1 + u^{2.5} \sqrt{M_1} \m M_2 + \sigma u^{2} \sqrt{M_1} \m M_2 + u^3 M_1 \m M_2 + \sigma u^{2.5} M_1 \m M_2 \\[1pt]
								&\mmm + \sigma^2 u^2 M_1 \m M_2 + u^{3.5} M^{1.5}_1 \m M_2 + \sigma u^{3} M^{1.5}_1 \m M_2 + u^2 M^2_1 M_2 + \sigma u^{1.5} M^2_1 \m M_2 \\
								&\mmm + \sigma^2 u^3 M^2_1 \m M_2\m \Big] d^{1.5} h^4_n\m.
\end{align*}
Similar to the ULD remainder bounds (Theorem \ref{thm:ULD_remainder_bounds}), we can prove \eqref{eq:QUICSORT_remainder_weak_bounds_1}, \eqref{eq:QUICSORT_remainder_weak_bounds_2} and \eqref{eq:QUICSORT_remainder_weak_bounds_3} using
	\begin{align*}
		\big\| \E [ R ] \big\|_2 & \leq \E\big[\|R\|_2\big] \leq \E\big[\|R\|_2^2\big]^\frac{1}{2} = \big\| R \big\|_{\L_2}\m,
	\end{align*}
	by Jensen's inequality. 
\end{proof}

We are now in a position to obtain local error bounds for QUICSORT (Theorem \ref{thm:local_error_bounds}). 

\begin{theorem}[Local error bounds for the QUICSORT method]
Let $\{(X_n, V_n) \}_{n\m\geq\m 0}$ denote the QUICSORT method \eqref{eq:QUICSORT} and $\{(x_t, v_t ) \}_{t\m\geq\m 0}$ be the underdamped Langevin diffusion \eqref{eq:ULD} with the same underlying Brownian motion.  Suppose the assumptions \eqref{eq:assumption_a1} and \eqref{eq:assumption_a2} hold and assume further that $(x_0, v_0) \sim \pi$, the unique stationary measure of the process. Suppose that at time $t_n\m$, we have $(X_n\m, V_n) = (x_{t_n}\m, v_{t_n})$ and let $0 < h_n \leq 1$ be fixed. Then, for any $n\geq 0$, the local strong error for the QUICSORT method at time $t_{n+1}$ has the bound\label{thm:local_error_bounds_app}
\begin{align}
	\label{eq:position_momentum_strong_error_bounds_1}
	\big\| X_{n+1} - x_{t_{n+1}} \big\|_{\L_2} \leq C^{\m (1)}_{s, \, x} \sqrt{d} \m h^{2}_n\m,\mmm \big\| V_{n+1} - v_{t_{n+1}} \big\|_{\L_2} \leq C^{\m (1)}_{s, \, v} \sqrt{d} \m h^{2}_n\m,
\end{align}
and the local weak error is bounded by
\begin{align}
	\label{eq:position_momentum_weak_error_bounds_1}
	\big\| \E \big[ X_{n+1} - x_{t_{n+1}} \big] \big\|_2 \leq C^{\m (1)}_{w, \, x} \sqrt{d} \m h^{2}_n, \mmm \big\| \E \big[ V_{n+1} - v_{t_{n+1}} \big] \big\|_2 \leq C^{\m (1)}_{w, \, v}\m \sqrt{d} \m h^{2}_n\m,
\end{align}
where $C^{\m (1)}_{s, \, x}\m, C^{\m (1)}_{s, \, v}\m, C^{\m (1)}_{w, \, x}\m, C^{\m (1)}_{w, \, v} > 0$ are constants depending on $\gamma$, $u$, $m$ and $M_1\m$.\medbreak

\noindent
Under the additional assumption \eqref{eq:assumption_a3}, the local strong error at time $t_{n+1}$ has the bound
\begin{align}
	\label{eq:position_momentum_strong_error_bounds_2}
	\big\| X_{n+1} - x_{t_{n+1}} \big\|_{\L_2} \leq C^{\m (2)}_{s, \, x} \sqrt{d} \m h^{3}_n\m, \mmm \big\| V_{n+1} - v_{t_{n+1}} \big\|_{\L_2} \leq C^{\m (2)}_{s, \, v}\m d h^{3}_n\m,
\end{align}
and the local weak error is bounded by
\begin{align}
	\label{eq:position_momentum_weak_error_bounds_2}
	\big\| \E \big[ X_{n+1} - x_{t_{n+1}} \big] \big\|_2 \leq C^{\m (2)}_{w, \, x} \sqrt{d} \m h^{3}_n\m, \mmm \big\| \E \big[ V_{n+1} - v_{t_{n+1}} \big] \big\|_2 \leq C^{\m (2)}_{w, \, v}\m d h^{3}_n\m,
\end{align}
where $C^{\m (2)}_{s, \, x}\m, C^{\m (2)}_{s, \, v}\m, C^{\m (2)}_{w, \, x}\m, C^{\m (2)}_{w, \, v} > 0$ are constants depending on $\gamma$, $u$, $m$, $M_1$ and $M_2\m$. Finally, under the additional assumption \eqref{eq:assumption_a4}, the local strong error at time $t_{n+1}$ has the bound
\begin{align}
	\label{eq:position_momentum_strong_error_bounds_3}
	\big\| X_{n+1} - x_{t_{n+1}} \big\|_{\L_2} \leq C^{\m (3)}_{s, \, x} \m d h^{3.5}_n\m, \mmm \big\| V_{n+1} - v_{t_{n+1}} \big\|_{\L_2} \leq C^{\m (3)}_{s, \, v} \m d^{1.5} \m h^{3.5}_n\m,
\end{align}
and the local weak error is bounded by
\begin{align}
	\label{eq:position_momentum_weak_error_bounds_3}
	\big\| \E \big[ X_{n+1} - x_{t_{n+1}} \big] \big\|_2 \leq C^{\m (3)}_{w, \, x} \m d h^{4}_n\m, \mmm \big\| \E \big[ V_{n+1} - v_{t_{n+1}} \big] \big\|_2 \leq C^{\m (3)}_{w, \, v} \m d^{1.5} \m h^{4}_n\m,
\end{align}
where $C^{\m (3)}_{s, \, x}\m, C^{\m (3)}_{s, \, v}\m, C^{\m (3)}_{w, \, x}\m, C^{\m (3)}_{w, \, v} > 0$ are constants depending on $\gamma$, $u$, $m$, $M_1$, $M_2$ and $M_3\m$.
\end{theorem}

\begin{proof}
	We first note that the Taylor expansions of ULD and QUICSORT are identical for all terms that are $O(h^3)$ or larger. Then, under the assumptions \eqref{eq:assumption_a1} and \eqref{eq:assumption_a2}, we have
	\begin{align*}
			\big\| \E \big[ X_{n+1} - x_{t_{n+1}} \big] \big\|_2 = \big\| X_{n+1} - x_{t_{n+1}} \big\|_{\L_2} &\leq \big\| R^{\m (1)}_X \big\|_{\L_2} + \big\| R^{\m (1)}_x \big\|_{\L_2} \\[3pt]
														&\leq \Big[ C^{\m (1)}_X + C^{\m (1)}_x \Big] \sqrt{d} \m h^2_n\m, \\[5pt]
			\big\| \E \big[ V_{n+1} - v_{t_{n+1}} \big] \big\|_2 = \big\| V_{n+1} - v_{t_{n+1}} \big\|_{\L_2} &\leq \big\| R^{\m (1)}_V \big\|_{\L_2} + \big\| R^{\m (1)}_v \big\|_{\L_2} \\[3pt]
														&\leq \Big[ C^{\m (1)}_V + C^{\m (1)}_v \Big] \sqrt{d} \m h^2_n\m,
	\end{align*}
	and, if we further assume \eqref{eq:assumption_a3}, then
	\begin{align*}
			\big\| \E \big[ X_{n+1} - x_{t_{n+1}} \big] \big\|_2 = \big\| X_{n+1} - x_{t_{n+1}} \big\|_{\L_2} &\leq \big\| R^{\m (2)}_X \big\|_{\L_2} + \big\| R^{\m (2)}_x \big\|_{\L_2} \\[3pt]
														&\leq \Big[ C^{\m (2)}_X + C^{\m (2)}_x \Big] \sqrt{d} \m h^3_n\m, \\[5pt]
			\big\| \E \big[ V_{n+1} - v_{t_{n+1}} \big] \big\|_2 = \big\| V_{n+1} - v_{t_{n+1}} \big\|_{\L_2} &\leq \big\| R^{\m (2)}_V \big\|_{\L_2} + \big\| R^{\m (2)}_v \big\|_{\L_2} \\[3pt]
														&\leq \Big[ C^{\m (2)}_V + C^{\m (2)}_v \Big] d \m h^3_n\m.
	\end{align*}
	In the case where we additionally assume \eqref{eq:assumption_a4}, we note that the triple time integral of the Brownian motion is centred, and we can write it in terms of $W_n$, $H_n$, $K_n$, $M_n$ using $\eqref{eq:levy_space_time_3}$. Moreover, the other stochastic integrals are centred and we can bound the following terms,
	\begin{align*}
	\begin{split}
			&\bigg\| \sigma u \int^{t_{n+1}}_{t_n} \int^{r_1}_{t_n} \int^{r_2}_{t_n} \nabla^3 f(x_{t_n}) ( v_{r_2}, W_{t_n, r_3} ) \, dr_3 \, dr_2 \, dr_1 \bigg\|_{\L_2} \\
			&\mmm \leq \sigma u M_2  \int^{t_{n+1}}_{t_n} \int^{r_1}_{t_n} \int^{r_2}_{t_n} \| v_{r_2}  \|_{\L_4} \| W_{t_n, r_3}  \|_{\L_4}  \, dr_3 \, dr_2 \, dr_1 \\
			&\mmm \leq \sqrt{3} \m \sigma u^{1.5} M_2 \m d \int^{t_{n+1}}_{t_n} \int^{r_1}_{t_n} \int^{r_2}_{t_n} ( r_3 - t_n )^{\frac{1}{2}}  \, dr_3 \, dr_2 \, dr_1 \\
			&\mmm \leq \frac{8\sqrt{3}}{105} \sigma u^{1.5} M_2 \m d \m h^{3.5}_n,
	\end{split}
	\end{align*}
	and
	\begin{align*}
			&\bigg\| \sigma u \int^{t_{n+1}}_{t_n} \int^{r_1}_{t_n} \nabla^3 f(x_{t_n}) \left( v_{t_n}, (r_2 - t_n) W_{t_n, r_2} \right) \, dr_2 \, dr_1 \bigg\|_{\L_2} \\
			&\mmm  \leq \sigma u M_2 \int^{t_{n+1}}_{t_n} \int^{r_1}_{t_n} (r_2 - t_n) \| v_{t_n} \|_{\L_4} \big\| W_{t_n, r_2}  \big\|_{\L_4} \, dr_2 \, dr_1 \\
			&\mmm \leq \sqrt{3} \m \sigma u^{1.5} M_2 \m d \int^{t_{n+1}}_{t_n} \int^{r_1}_{t_n} ( r_2 - t_n )^{\frac{3}{2}} \, dr_2 \, dr_1 \\
			&\mmm  \leq \frac{4\sqrt{3}}{35} \sigma u^{1.5} M_2 \m d \m h^{3.5}_n,
	\end{align*}
	by Fubini's theorem along with the bounds \eqref{eq:v_bounds}, \eqref{eq:levy_area_bounds} and \eqref{eq:n3_f_x_v1_v2}. We can show \eqref{eq:position_momentum_strong_error_bounds_3} and \eqref{eq:position_momentum_weak_error_bounds_3} for the position variable as follows
	\begin{align*}
			\big\| X_{n+1} - x_{t_{n+1}} \big\|_{\L_2} &\leq \bigg\| \sigma \big( \gamma^2 - u \nabla^2 f(x_{t_n}) \big) \bigg( \frac{1}{60} h^3_n H_n - h^3_n M_n \bigg) \bigg\|_{\L_2} + \big\| R^{\m (3)}_X \big\|_{\L_2}\hspace{-2mm} + \big\| R^{\m (3)}_x \big\|_{\L_2} \\[3pt]
					&\leq \bigg[ \frac{7 \sqrt{3} + 3 \sqrt{7}}{2520} \sigma ( \gamma^2 - um ) + C^{\m (3)}_{X} \m \sqrt{h_n} + C^{\m (3)}_{x} \m \sqrt{h_n}\m \bigg] d \m h^{3.5}_n \\[3pt]
					&\leq \Big[\m \sigma ( \gamma^2 - um ) + C^{\m (3)}_{X} + C^{\m (3)}_{x}\m\Big] d \m h^{3.5}_n, \\[5pt]
			\big\| \E \big[ X_{n+1} - x_{t_{n+1}} \big] \big\|_2  &\leq \big\| R^{\m (3)}_X \big\|_{\L_2} + \big\| R^{\m (3)}_x \big\|_{\L_2} \\[3pt]
			& \leq \Big[ C^{\m (3)}_{X} + C^{\m (3)}_{x} \Big] d \m h^4_n\m,
	\end{align*}
	and for the momentum variable,
	\begin{align*}
			\big\| V_{n+1} - v_{t_{n+1}} \big\|_{\L_2} &\leq \bigg\| \sigma \big( \gamma^2 - 2 u \nabla^2 f(x_{t_n}) \big) \bigg( \frac{1}{60} h^3_n H_n - h^3_n M_n \bigg) \bigg\|_{\L_2} \\
									&\mmm  + \bigg\| \sigma u \int^{t_{n+1}}_{t_n} \int^{r_1}_{t_n} \int^{r_2}_{t_n} \nabla^3 f(x_{t_n}) ( v_{r_2}, W_{t_n, r_3} ) \, dr_3 \, dr_2 \, dr_1 \bigg\|_{\L_2} \\
									&\mmm + \bigg\| \sigma u \int^{t_{n+1}}_{t_n} \int^{r_1}_{t_n} \nabla^3 f(x_{t_n}) ( v_{t_n}, (r_2 - t_n) W_{t_n, r_2} ) \, dr_2 \, dr_1 \bigg\|_{\L_2} \\
									&\mmm + \big\| R^{\m (3)}_V \big\|_{\L_2} +  \big\| R^{\m (3)}_v \big\|_{\L_2} \\[3pt]
									&\leq \bigg[ \frac{7 \sqrt{3} + 3 \sqrt{7}}{2520} \sigma ( \gamma^2 - um ) + \frac{12\sqrt{3}}{35} \sigma u^{1.5} M_2  \\
									&\mmm + C^{\m (3)}_{V} \m \sqrt{h_n} + C^{\m (3)}_{v} \m \sqrt{h_n}\m \bigg] d^{1.5} \m h^{3.5}_n \\[3pt]
									&\leq \Big[ \sigma ( \gamma^2 - um ) + \sigma u^{1.5} M_2 + C^{\m (3)}_{V} + C^{\m (3)}_{v}  \Big] d^{1.5} h^{3.5}, \\[5pt]
			\big\| \E \big[ V_{n+1} - v_{t_{n+1}} \big] \big\|_2  &\leq \big\| R^{\m (3)}_V \big\|_{\L_2} + \big\| R^{\m (3)}_v \big\|_{\L_2} \leq \Big[ C^{\m (3)}_{V} + C^{\m (3)}_{v} \Big] d^{1.5} \m h^4_n\m,
	\end{align*}
	where we have used Minkowski's inequality and the bounds on the remainder terms. 
\end{proof}

\begin{corollary}
    Under the same settings as used in Theorem \ref{thm:local_error_bounds_app}, assuming \eqref{eq:assumption_a1} and \eqref{eq:assumption_a2}, we have the additional bounds\label{thm:conditional_local_error_bounds_app}
	\begin{align}
		\label{eq:position_momentum_L2_weak_error_bounds_1}
		\big\| \E \big[ X_{n+1} - x_{t_{n+1}} \mid \mathcal{F}_n \big] \big\|_{\mathbb{L}_2} & \leq C^{\m (1)}_{w, \m x}  \sqrt{d} \m h^{2}_n\m, \mm \big\| \E \big[ V_{n+1} - v_{t_{n+1}} \mid \mathcal{F}_n \big] \big\|_{\mathbb{L}_2} \leq C^{\m (1)}_{w, \m v}  \sqrt{d} \m  h^{2}_n\m,
	\end{align}
	where $C^{\m (1)}_{w, \, x}\m, C^{\m (1)}_{w, \, v} > 0$ are constants depending on $\gamma$, $u$, $m$ and $M_1\m$.\medbreak
\noindent
Under the additional assumption \eqref{eq:assumption_a3}, we have the bounds
\begin{align}
		\label{eq:position_momentum_L2_weak_error_bounds_2}
		\big\| \E \big[ X_{n+1} - x_{t_{n+1}} \mid \mathcal{F}_n \big] \big\|_{\mathbb{L}_2} & \leq C^{\m (2)}_{w, \m x}  \m d  h^{3}_n\m, \mm \big\| \E \big[ V_{n+1} - v_{t_{n+1}} \mid \mathcal{F}_n \big] \big\|_{\mathbb{L}_2} \leq C^{\m (2)}_{w, \m v} \m d \m  h^{3}_n\m,
	\end{align}
where $C^{\m (2)}_{w, \, x}\m, C^{\m (2)}_{w, \, v} > 0$ are constants depending on $\gamma$, $u$, $m$, $M_1$ and $M_2\m$. \medbreak

\noindent
Finally, under the additional assumption \eqref{eq:assumption_a4}, we have the bounds
\begin{align}
		\label{eq:position_momentum_L2_weak_error_bounds_3}
		\big\| \E \big[ X_{n+1} - x_{t_{n+1}} \mid \mathcal{F}_n \big] \big\|_{\mathbb{L}_2} & \leq C^{\m (3)}_{w, \m x}  \m d^{1.5}\m  h^{4}_n\m, \hspace{3.3mm} \big\| \E \big[ V_{n+1} - v_{t_{n+1}} \mid \mathcal{F}_n \big] \big\|_{\mathbb{L}_2} \leq C^{\m (3)}_{w, \m v} \m d^{1.5}\m \m  h^{4}_n\m,
	\end{align}
where $C^{\m (3)}_{w, \, x}\m, C^{\m (3)}_{w, \, v}\m > 0$ are constants depending on $\gamma$, $u$, $m$, $M_1$, $M_2$ and $M_3\m$.
\end{corollary}
\begin{proof}
    Notice that the Taylor expansion for the ULD and the QUICSORT method are identical up to order $h^3$, and the only mismatch are the stochastic integrals of order $h^{3.5}$. However, these integrals have a conditional mean of zero, and thus do not contribute to the above errors. Moreover, for both the position and momentum variables' remainder terms, under any of the smoothness assumptions, 
    \begin{align*}
        \big\| \E \big[ R \mid \mathcal{F}_n \big] \big\|_{\mathbb{L}_2} = \Big( \E \Big[ \big\| \E \big[ R \mid \mathcal{F}_n \big] \big\|_2^2 \ \Big] \Big)^{\frac{1}{2}} \leq  \Big( \E \Big[ \E \big[ \,\| R \|_2^2 \mid \mathcal{F}_n \big]\ \Big] \Big)^{\frac{1}{2}} = \| R \|_{\mathbb{L}_2} \m,
    \end{align*}
    by Jensen's inequality and the tower property of conditional expectations. The result now follows by applying the same logic as used in the proof of Theorem \ref{thm:local_error_bounds_app} to establish the local weak error bounds.
\end{proof}

\begin{remark}
    The bounds for the $\L_2$ norm of the absolute conditional mean are identical to the bounds for the local weak error. This will allow us to use them interchangeably when proving the global error result (Theorem \ref{thm:global_error_bound}).
\end{remark}


\section{Global error bounds}
\label{sec:appendix_D}

In this section, we prove the global error bounds for the QUICSORT method (Theorem \ref{thm:global_error_bound}). To do so, we use the methodology outlined in Section \ref{sec:error_analysis}. We will also use the following theorem. 

\begin{theorem}
\label{thm:QUICSORT_linear_growth_bound}
	Consider two approximations $Q = (X, V)$ and $\widehat{Q} = (Y, U)$ under the mapping $Q_{n} \to Q_{n+1}$ and $\widehat{Q}_{n} \to \widehat{Q}_{n+1}$ defined in equation \eqref{eq:QUICSORT}, driven by the same Brownian motion but evolved from the initial conditions $Q_n = (X_n\m, V_n)$ and $\widehat{Q}_{n} = ( Y_n\m, U_n)$ respectively. Suppose assumption \eqref{eq:assumption_a2} holds and let $0 < h_n \leq 1$. Then for all $n \geq 0$
	\begin{align}
		\label{eq:approximation_difference_bound}
			\big\| (Q_{n+1} - Q_n ) - (\widehat{Q}_{n+1} - \widehat{Q}_n) \big\|_{\L_2} \leq C_q \big\| Q_n - \widehat{Q}_n \big\|_{\L_2} \m h_n\m,
	\end{align}
	where $C_q > 0$ is a constant depending on $\gamma$, $u$, and $M_1\m$. 
\end{theorem}

\begin{proof}
	We first rewrite the QUICSORT method \eqref{eq:QUICSORT} as
	\begin{align*}
			X_{n+1} &= X_n + \phi_1(1) \big(V_n + \sigma (H_n + 6 K_n) \big) + \phi_2(1) C_n  \\
					&\mmm - \frac{1}{2} u h_n \Big( \phi_1 (\lambda_+) \nabla f\big(X^{(1)}_n\big) + \phi_1 (\lambda_-) \nabla f\big(X^{(2)}_n\big) \Big),  \\[5pt]
			V_{n+1} &= \phi_0 (1) \Big( V_n + \sigma (H_n + 6 K_n) \big) + \sigma (H_n - 6 K_n) + \phi_1 (1) C_n \\
					&\mmm - \frac{1}{2} u h_n \Big( \phi_0 (\lambda_+) \nabla f(X^{(1)}_n) + \phi_0 (\lambda_-) \nabla f\big(X^{(2)}_n\big) \Big),
	\end{align*}
	where
	\begin{align*}
			X^{\m (1)}_n &= X_n + \phi_1 (\lambda_-) \big( V_n + \sigma ( H_n + 6K_n ) \big) + \phi_2 (\lambda_-) C_n\m, \\[3pt]
			X^{\m (2)}_n &= X_n + \phi_1 (\lambda_+) \big( V_n + \sigma( H_n + 6K_n ) \big) + \phi_2 (\lambda_+) C_n - h_n \phi_1\bigg(\frac{1}{3}\bigg) u \nabla f\big(X^{\m (1)}_n\big).
	\end{align*}
	Since the approximations $Q$ and $\widehat{Q}$ use the same Brownian motion, we can show that
	\begin{align*}
			&\big\| (X_{n+1} - X_n ) - (Y_{n+1} - Y_n) \big\|_{\L_2} \\[3pt]
			&\mm \leq \vert \phi_1(1) \vert \| V_n - U_n \|_{\L_2} + \frac{1}{2} u h_n \vert \phi_1 (\lambda_+) \vert \big\| \nabla f\big(X^{\m (1)}_n\big) - \nabla f\big(Y^{\m (1)}_n\big) \big\|_{\L_2} \\
			&\hspace{12.5mm} + \frac{1}{2} u h_n \vert \phi_1 (\lambda_-) \vert \big\| \nabla f\big(X^{\m (2)}_n\big) - \nabla f\big(Y^{\m (2)}_n\big) \big\|_{\L_2} \\[3pt]
			&\mm \leq \frac{1}{2} u M_1 \m h_n \bigg(\m\vert \phi_1 (\lambda_+) \vert + \vert \phi_1 (\lambda_-) \vert + u M_1 \m h_n \Big\vert\m\phi_1 \Big(\frac{1}{3}\Big) \phi_1 (\lambda_-)\m \Big\vert\m\bigg) \big\| X_n - Y_n \big\|_{\L_2} \\
			&\hspace{12.5mm} + \bigg(\vert \phi_1(1) \vert + u M_1 \m h_n \vert \phi_1 (\lambda_+) \phi_1 (\lambda_-) \vert + \frac{1}{2} u^2 M^2_1 \m h^2_n \Big\vert \phi_1 \Big(\frac{1}{3}\Big) \phi^2_1 (\lambda_-) \Big\vert \bigg) \big\| V_n - U_n \big\|_{\L_2} \\[3pt]
			&\mm \leq \big( 1 + u M_1 + u^2 M^2_1 \big)  \big\| Q_n - \widehat{Q}_n \big\|_{\L_2} h_n\m,
	\end{align*}
	and, similarly for the momentum variable, we have
	\begin{align*}
		\begin{split}
			&\big\| (V_{n+1} - V_n ) - (U_{n+1} - U_n) \big\|_{\L_2} \\[3pt]
			&\mm \leq \vert \phi_0 (1) - 1 \vert \big\| V_n - U_n \big\|_{\L_2} + \frac{1}{2} u h_n \vert \phi_0 (\lambda_+) \vert \big\| \nabla f\big(X^{\m (1)}_n\big) - \nabla f\big(Y^{\m (1)}_n\big) \big\|_{\L_2} \\
			&\hspace{12.5mm} + \frac{1}{2} u h_n \vert \phi_0 (\lambda_-) \vert \big\| \nabla f\big(X^{\m (2)}_n\big) - \nabla f\big(Y^{\m (2)}_n\big) \big\|_{\L_2} \\[3pt]
			&\mm \leq \frac{1}{2} u M_1 \m h_n \bigg[ \m\vert \phi_0 (\lambda_+) \vert + \vert \phi_0 (\lambda_-) \vert + u M_1 \m h_n \Big\vert \phi_0 (\lambda_-) \phi_1 \Big(\frac{1}{3}\Big) \Big\vert\m \bigg]  \big\| X_n - Y_n \big\|_{\L_2} \\
			&\hspace{12.5mm} + \bigg(\m\vert \phi_0 (1) - 1 \vert + \frac{1}{2} u M_1 \m h_n \big( \vert \phi_0 (\lambda_+) \phi_1 (\lambda_-) \vert + \vert \phi_0 (\lambda_1) \phi_1 (\lambda_+) \vert \big) \\
			&\hspace{20mm} + \frac{1}{2} u^2 M^2_1 \m h^2_n \Big\vert \phi_0 (\lambda_-) \phi_1 \Big(\frac{1}{3}\Big) \phi_1 (\lambda_-) \Big\vert\m\bigg)\big\| V_n - U_n \big\|_{\L_2} \\[3pt]
			&\mm \leq \big(\gamma + 2 u M_1 + u^2 M^2_1 \big)  \big\| Q_n - \widehat{Q}_n \big\|_{\L_2} h_n\m .
		\end{split}
	\end{align*}
	The result follows by summing the bounds for the position and momentum variables.
\end{proof}

\noindent
Finally, we will show the main result of our paper (global error bounds for QUICSORT).

\begin{theorem}[Global error bounds for the QUICSORT method]
	Let $\{(X_n, V_n) \}_{n\m\geq\m 0}$ denote the QUICSORT method \eqref{eq:QUICSORT} and let $\{(x_t, v_t ) \}_{t\m\geq\m 0}$ be the underdamped Langevin diffusion \eqref{eq:ULD} driven by the same underlying Brownian motion. Suppose assumptions \eqref{eq:assumption_a1} and \eqref{eq:assumption_a2} hold. Further suppose that $(x_0, v_0) \sim \pi$, the unique stationary measure of the diffusion process and that the approximation and the SDE solution have the same initial velocity $V_0 = v_0 \sim \mathcal{N} (0, u I_d)$. Let $0 < h \leq 1$ be fixed and $\gamma \geq 2 \sqrt{uM_1}$. Then for all $n \geq 0$, the global error at time $t_n$ has the bound\label{thm:global_error_bound_app}
	\begin{align*}
		\| X_n - x_{t_n} \|_{\L_2} \leq  c e^{- n \beta h} \| X_0 - x_0 \|_{\L_2} +  C_1 \m \sqrt{d} \m h,
	\end{align*}
	where $c, \beta, C_1 > 0$ are constants depending on $\gamma$, $u$, $m$ and $M_1\m$.\medbreak
	
	\noindent
	Under the additional assumption \eqref{eq:assumption_a3}, the global error at time $t_n$ satisfies the bound
	\begin{align*}
		\| X_n - x_{t_n} \|_{\L_2} \leq  c e^{- n \beta h} \| X_0 - x_0 \|_{\L_2} +  C_2 \m d \m h^2,
	\end{align*}
	where $C_2 > 0$ is a constant depending on $\gamma$, $u$, $m$, $M_1$ and $M_2\m$.\medbreak
	
	\noindent
	Finally, under the additional assumption \eqref{eq:assumption_a4}, the global error at time $t_n$ has the bound
	\begin{align*}
		\| X_n - x_{t_n} \|_{\L_2} \leq  c e^{- \frac{1}{2} n \beta h} \| X_0 - x_0 \|_{\L_2} +  C_3 \m d^{1.5} \m h^3,
	\end{align*}
	where $C_3 > 0$ is a constant depending on $\gamma$, $u$, $m$, $M_1$, $M_2$ and $M_3\m$. 
\end{theorem}
\begin{proof}
	The approach for this proof will follow the strategy that was outlined in Section \ref{sec:error_analysis}.
	We will write $Q_n = (X_n\m, V_n)^\top$, $q_{t_n} = (x_{t_n}\m, v_{t_n})^\top$ and define the coordinate transformation
	\begin{align*}
		A := \begin{pmatrix} w I_d & I_d \\ z I_d & I_d \end{pmatrix},
	\end{align*}
	where $w \in [0, \frac{1}{2} \gamma )$, $z = \gamma - w$. We will define the errors with and without this transformation,
	\begin{align*}
			e_n &:= \big\| Q_n - q_{t_n} \big\|_{\L_2}, \\[3pt]
			\widetilde{e}_n &:= \big\| A ( Q_n - q_{t_n} ) \big\|_{\L_2}.
	\end{align*}
	Let $\sigma_{\max} = \|A\|_2$ and $\sigma_{\min}$ denote the smallest and largest singular values of $A$ respectively. Since $A$ is invertible, the singular value $\sigma_{\min}$ is non-zero and will satisfy $\sigma_{\min} = \|A^{-1}\|_2^{-1}$. Let $\kappa = \frac{\sigma_{\max}}{\sigma_{\min}}$ denote the condition number of $A$. We can use $\widetilde{e}_n$ to bound the error $e_n$ as
	\begin{align*}
		 \sigma_{\min} \m e_n \leq \widetilde{e}_n \leq \sigma_{\max} \m e_n\m.
	\end{align*}
We now consider a third process $\widehat{Q} =  (\widehat{X}, \widehat{V})^\top$ where $(\widehat{X}_{n+1}\m, \widehat{V}_{n+1})$ is obtained by applying the QUICSORT method \eqref{eq:QUICSORT} to the ULD diffusion at time $t_n$ with the same underlying Brownian motion as both $Q$ and $q$. This allows us to estimate the local error between the coordinate-changed approximation $AQ_{n+1}$ and diffusion process $Aq_{t_{n+1}}$ at time $t_{n+1}$ in terms of the same error at time $t_n\m$. By Theorem \ref{thm:local_error_bounds}, we can bound the local strong error as
	\begin{align*}
		\big\| \widehat{Q}_n - q_{t_n} \big\|_{\L_2} \leq C_s \m d^{\m q} \m h^{\m p_s},
	\end{align*}
	and the local weak error as
	\begin{align*}
		\big\| \E \big[ \widehat{Q}_n - q_{t_n} \big] \big\|_2 \leq C_w \m d^{\m q} \m h^{\m p_w},
	\end{align*}
	where $C_s\m, C_w\m, q, p_s\m, p_w > 0$ are the following constants (which do not depend on $d$ and $h$).
   \begin{center}
  \begin{tabular}{cccccc}
    \toprule
    Assumptions  & $C_s$ & $C_w$ &  $q$ & $p_s$ & $p_w$\\[1pt]
    \midrule
    \eqref{eq:assumption_a1},  \eqref{eq:assumption_a2} & $C^{\m (1)}_{s, x} + C^{\m (1)}_{s, v}$  & $C^{\m (1)}_{w, x} + C^{\m (1)}_{w, v}$ & 0.5 & 2 & 2\\
     \midrule
     \eqref{eq:assumption_a1},  \eqref{eq:assumption_a2}, \eqref{eq:assumption_a3} & $C^{\m (2)}_{s, x} + C^{\m (2)}_{s, v}$ & $C^{\m (2)}_{w, x} + C^{\m (2)}_{w, v}$ & 1 & 3 & 3 \\
     \midrule
     \eqref{eq:assumption_a1},  \eqref{eq:assumption_a2}, \eqref{eq:assumption_a3}, \eqref{eq:assumption_a4} &  $C^{\m (3)}_{s, x} + C^{\m (3)}_{s, v}$  & $C^{\m (3)}_{w, x} + C^{\m (3)}_{w, v}$ & 1.5 & 3.5 & 4\\
    \bottomrule
  \end{tabular}
  \end{center}
  In the above table, we also note that $p_s + \frac{1}{2} \geq p_w\m$. We now define the following variables
	\begin{align*}
			y_n &:= A \big( Q_n - q_{t_n} \big) = \begin{pmatrix} ( w X_n + V_n ) - ( w x_{t_n} + v_{t_n} ) \\[3pt] ( z X_n + V_n ) - ( z x_{t_n} + v_{t_n} ) \end{pmatrix}, \\[5pt]
			\widetilde{y}_n &:= A \big( Q_n - \widehat{Q}_n \big) = \begin{pmatrix} ( w X_n + V_n ) -  ( w \widehat{X}_n + \widehat{V}_n ) \\[3pt] ( z X_n + V_n ) -  ( z \widehat{X}_n + \widehat{V}_n ) \end{pmatrix}, \\[5pt]
			\widehat{y}_n &:= A \big( \widehat{Q}_n - q_{t_n} \big) = \begin{pmatrix} ( w \widehat{X}_n + \widehat{V}_n ) - ( w x_{t_n} + v_{t_n} ) \\[3pt] ( z \widehat{X}_n + \widehat{V}_n ) - ( z x_{t_n} + v_{t_n} ) \end{pmatrix},
	\end{align*}
	and note that $y_n = \widetilde{y}_n + \widehat{y}_n\m$. Thus, the coordinate-transformed error $\widetilde{e}_n$ can be estimated as
    \begin{align*}
	\widetilde{e}_{n+1} &= \| y_{n+1} \|_{\L_2} = \big\| \widetilde{y}_{n+1} + \widehat{y}_{n+1} \big\|_{\L_2} \leq \big\| \widetilde{y}_{n+1} \big\|_{\L_2} + \big\| \widehat{y}_{n+1} \big\|_{\L_2}\m.
    \end{align*}
    We bound the first term using the contractivity of the QUICSORT method (Theorem \ref{thm:main_contraction}),
    \begin{align*}
		\big\| \widetilde{y}_{n+1} \big\|_{\L_2} \leq \sqrt{1 - 2\alpha h}\, \widetilde{e}_n\m,
	\end{align*} 
    and the second term can be bounded using the local strong error,
	\begin{align*}
		\big\|\m \widehat{y}_{n+1} \big\|_{\L_2} \leq \sigma_{\max} \m C_s \m d^{\m q} h^{p_s}.
	\end{align*}
    Therefore, we obtain the following estimate of the coordinate-transformed error,
    \begin{align}\label{eq:simple_inequality}
	\widetilde{e}_{n+1} \leq \sqrt{1 - 2\alpha h}\, \widetilde{e}_n + \sigma_{\max} \m C_s \m d^{\m q} h^{p_s}.
\end{align}
By setting $w = 0$ and using Remark \ref{thm:positive_contractivity_constant}, we have $\alpha >0$ provided $h$ is sufficiently small. 
Similarly, by assuming $h$ is sufficiently small so that $\alpha h \leq \frac{1}{2}$, we obtain the below inequality.
\begin{align*}
\sqrt{1-2\alpha h}\leq \sqrt{1 - 2\alpha h + \alpha^2 h^2} = 1 - \alpha h.
\end{align*}
Therefore, using this inequality and by applying the inequality \eqref{eq:simple_inequality} $n$ times to $\widetilde{e}_n\m$, we obtain
\begin{align*}
	\widetilde{e}_n & \leq \big(\sqrt{1-2\alpha h}\,\big)^n\m\widetilde{e}_0 + \sum_{k=0}^{n-1}\big(\sqrt{1-2\alpha h}\,\big)^k \sigma_{\max} \m C_s \m d^{\m q} h^{p_s}\\
    &\leq \big(1 - \alpha h\big)^n\m\widetilde{e}_0 + \sum_{k=0}^{n-1}\big(1 - \alpha h\big)^k \sigma_{\max} \m C_s \m d^{\m q} h^{p_s}\\
    & =\big(1 - \alpha h\big)^n\m\widetilde{e}_0  + \bigg(\frac{1-(1-\alpha h)^n}{1-(1-\alpha h)}\bigg)\sigma_{\max} \m C_s \m d^{\m q} h^{p_s}\\
    & \leq  \big(1 - \alpha h\big)^n\m\widetilde{e}_0 + \bigg(\frac{1}{1-(1-\alpha h)}\bigg)\sigma_{\max} \m C_s \m d^{\m q} h^{p_s}\\[2pt]
    & \leq e^{-n\alpha h}\m\widetilde{e}_0 + \frac{1}{\alpha}\sigma_{\max} \m C_s \m d^{\m q} h^{p_s - 1}.
\end{align*}
Since $e_n \leq \frac{1}{\sigma_{\min}}\m\widetilde{e}_n$ and $\widetilde{e}_0 \leq \sigma_{\max} \m e_0\m$, it follows that the error $e_n$ can be bounded as
  \begin{align}\label{eq:simple_case}
			e_n & \leq  \kappa_{A} e^{-n\alpha h}e_0 + \kappa_{A} \bigg(\frac{1}{\alpha}C_s\bigg) d^{\m q} h^{p_s - 1},
\end{align}
for $n\geq 0$. Using $e_0 = \| X_0 - x_0 \|_{\L_2}$ and $\| X_n - x_{t_n} \|_{\L_2} \leq e_n\m$, the inequality \eqref{eq:simple_case} gives the first two bounds (under the assumptions (\eqref{eq:assumption_a1},  \eqref{eq:assumption_a2}) and (\eqref{eq:assumption_a1},  \eqref{eq:assumption_a2}, \eqref{eq:assumption_a3}) respectively).
However, the estimate would only be $O(h^{2.5})$ in the third case. Thus, we now expand $\widetilde{e}_n$ as
	\begin{align*}
	\begin{split}
			\widetilde{e}_{n+1}^{\, 2} & = \big\| \widetilde{y}_{n+1} + \widehat{y}_{n+1} \big\|^2_{\L_2} \\[3pt]
					&= \big\| \widetilde{y}_{n+1} \big\|^2_{\L_2} + \big\| \widehat{y}_{n+1} \big\|^2_{\L_2} + 2\m \E\big[\langle \m\widetilde{y}_{n+1}, \widehat{y}_{n+1} \rangle\big].
	\end{split}
	\end{align*}
	We bound the first term using the contractivity of the QUICSORT method (Theorem \ref{thm:main_contraction}),
	\begin{align*}
		\big\| \widetilde{y}_{n+1} \big\|^2_{\L_2} \leq (1 - 2\alpha h)\m \widetilde{e}_n^{\, 2}\m,
	\end{align*} 
	and the second term can be bounded using the local strong error,
	\begin{align*}
		\big\|\m \widehat{y}_{n+1} \big\|^2_{\L_2} \leq \sigma^2_{\max} \m C^2_s \m d^{\m 2q} h^{2p_s}.
	\end{align*}
	Let $\mathcal{F}_n$ be the filtration underlying the processes at time $t_n\m$. Then, by the tower property of conditional expectation and the Cauchy-Schwarz inequality, we can bound the third term
	\begin{align*}
			 \E\big[\langle \m\widetilde{y}_{n+1}, \widehat{y}_{n+1} \rangle\big] & =  \E \big[\langle \m y_{n}\m, \widehat{y}_{n+1} \rangle\big] +  \E\big[\langle \m\widetilde{y}_{n+1} - y_n\m, \widehat{y}_{n+1} \rangle\big] \\[3pt]
			 			&=  \E\big[\big\langle y_n\m, \E \big[\m \widehat{y}_{n+1} \vert \mathcal{F}_n \big] \big\rangle\big] +  \E\big[\langle\m\widetilde{y}_{n+1} - y_n\m, \widehat{y}_{n+1} \rangle\big] \\[3pt]
			 			& \leq \big\|y_n\big\|_{\L_2}\big\|\m\E\big[\m\widehat{y}_{n+1} \vert \mathcal{F}_n \big]\big\|_{\L_2} + \big\|\m\widetilde{y}_{n+1} - y_n \big\|_{\L_2}\big\|\m\widehat{y}_{n+1}\big\|_{\L_2}\\[3pt]
						&\leq \widetilde{e}_n \|A\|_2 \big\|\E \big[\m\widehat{Q}_{n+1} - q_{t_{n+1}}  \vert \mathcal{F}_n \big]\big\|_{\L_2} \\[2pt]
						&\mmm + \|A\|_2^2\big\|\big(Q_{n+1} - \widehat{Q}_{n+1}\big) - (Q_n - q_{t_n}) \big\|_{\L_2}\big\|\widehat{Q}_{n+1} - q_{t_{n+1}}\big\|_{\L_2}.
\end{align*}
Using the local strong error bounds for QUICSORT, Theorem \ref{thm:QUICSORT_linear_growth_bound} and Corollary \ref{thm:conditional_local_error_bounds_app}, we have
\begin{align*}
			\E\big[\langle \m\widetilde{y}_{n+1}, \widehat{y}_{n+1} \rangle\big]
						&\leq \sigma_{\max} \m \widetilde{e}_n C_w \m d^q \m h^{\m p_w} + \sigma^2_{\max} \m C_q e_n h \cdot C_s \m d^q h^{\m p_s}\\[3pt]
						&\leq \sigma_{\max} \m \widetilde{e}_n \big( C_w + \kappa_A \m C_q \m C_s \big) d^{\m q} h^{\m p_w},
	\end{align*}
	where we also used the assumption that $h \leq 1$. Putting everything together and applying the inequality \eqref{eq:cauchy_schwarz_young_inequality} in Theorem \ref{thm:cauchy_schartz_young_inequality} with $c = \frac{1}{\alpha h}\m$, we obtain
	\begin{align*}
			\widetilde{e}_{n+1}^{\,2} &\leq \big(1 - 2\alpha h\big)\m \widetilde{e}_n^{\,2} + \sigma^2_{\max} \m C^2_s \m d^{2q} h^{2p_s} + 2 \sigma_{\max} \m \widetilde{e}_n \big(C_w + \kappa_A \m C_q \m C_s \big) d^{\m q} h^{\m p_w} \\[3pt]
					&\leq \big(1 - 2\alpha h\big)\m \widetilde{e}_n^{\,2} + \sigma^2_{\max} \m C^2_s \m d^{2q} h^{2p_s} + \alpha h \widetilde{e}_n^{\,2} + \frac{1}{\alpha} \sigma^2_{\max} \big(C_w + \kappa_A \m C_q \m C_s\big)^2 d^{\m 2q} h^{\m 2 p_w - 1} \\[3pt]
					&\leq \big(1 - \alpha h\big)\m \widetilde{e}_n^{\,2} + \sigma^2_{\max} \m \Big(C^2_s + \frac{1}{\alpha} \big( C_w + \kappa_A \m C_q \m C_s \big)^2 \Big) d^{\m 2q} h^{\m 2p_w - 1}.
	\end{align*}
	Applying the above inequality $n$ times to $\widetilde{e}_n$ gives
	\begin{align*}
			\widetilde{e}_n^{\,2} &\leq \big(1 - \alpha h\big)^n \m\widetilde{e}_0^{\,2} +  \sigma^2_{\max} \m \Big(C^2_s + \frac{1}{\alpha} \big( C_w + \kappa_A \m C_q \m C_s \big)^2 \Big) d^{\m 2q} h^{\m 2p_w - 1} \sum^{n-1}_{i=0} \big(1 - \alpha h\big)^i \\[3pt]
				&= \big(1 - \alpha h\big)^n \widetilde{e}_0^{\,2} +  \sigma^2_{\max} \m \Big( C^2_s + \frac{1}{\alpha} \big( C_w + \kappa_A \m C_q \m C_s \big)^2 \Big) d^{\m 2q} h^{\m 2p_w - 1} \bigg(\frac{1 - (1 - \alpha h)^n}{\alpha h} \bigg)\\[3pt]
				&\leq e^{-n \alpha h}\m \widetilde{e}_0^{\,2}  + \sigma^2_{\max} \m \bigg( \frac{1}{\alpha} C^2_s + \frac{1}{\alpha^2} \big( C_w + \kappa_A \m C_q \m C_s \big)^2 \bigg) d^{\m 2q} h^{\m 2p_w - 2}.
	\end{align*}
    Like earlier, we set $w = 0$ and used Remark \ref{thm:positive_contractivity_constant} so that $0 < \alpha h \leq 1$ if $h$ is sufficiently small. 
	Taking a square root on both sides of the inequality and using $\sqrt{a^2 + b^2} \leq a + b$ for $a,b \geq 0$,
	\begin{align*}
		\widetilde{e}_n \leq e^{-\frac{1}{2} n \alpha h} \widetilde{e}_0 + \sigma_{\max} \bigg( \frac{1}{\sqrt{\alpha}} \m C_s + \frac{1}{\alpha} \big( C_w + \kappa_A \m C_q \m C_s \big) \bigg) d^{\m q} h^{\m p_w - 1}.
	\end{align*}
	Using $e_n \leq \frac{1}{\sigma_{\min}}\m\widetilde{e}_n$ and $\widetilde{e}_0 \leq \sigma_{\max} \m e_0\m$, we obtain
	\begin{align*}
		e_n \leq \kappa_A e^{-\frac{1}{2} n \alpha h} e_0 + \kappa_A \bigg( \frac{1}{\sqrt{\alpha}} \m C_s + \frac{1}{\alpha} \big( C_w + \kappa_A \m C_q \m C_s \big) \bigg) d^{\m q} h^{\m p_w - 1},
	\end{align*}
	for all $n \geq 0$. The third result now follows since $e_0 = \| X_0 - x_0 \|_{\L_2}$ and $\| X_n - x_{t_n} \|_{\L_2} \leq e_n\m$. 
\end{proof}
\end{document}